\documentclass[twoside,11pt]{article}

\usepackage[preprint]{jmlr2e}   
\usepackage{float}
\usepackage{booktabs}
\usepackage{amsmath}
\usepackage{enumitem}
\usepackage{algorithm}
\usepackage{algorithmicx}
\usepackage{algpseudocode}

\usepackage{natbib}
\usepackage{subcaption}
\usepackage{xcolor}

\DeclareMathOperator*{\argmax}{arg\,max}
\DeclareMathOperator{\Var}{Var}

\newtheorem{assumption}[theorem]{Assumption}
\newtheorem{property}[theorem]{Property} 

\newcommand{\OCE}{\mathrm{OCE}}
\newcommand{\tr}{\mathrm{Tr}}
\newcommand{\dist}{\mathrm{dist}}
\newcommand{\diag}{\mathrm{diag}}
\newcommand{\clip}{\mathrm{clip}}
\newcommand{\SubOpt}{\mathrm{SubOpt}}
\newcommand{\SubOptD}{\mathrm{SubOpt}_{\mathrm{D}}}
\newcommand{\SubOptS}{\mathrm{SubOpt}_{\mathrm{S}}}

\usepackage{lastpage}

\jmlrheading{}{}{1-\pageref{LastPage}}{; Revised}{}{21-0000}{Wang, Zhang, Lyu, Dai, Gao, and Qiu}

\ShortHeadings{Provably Efficient Risk-Sensitive Offline RL via Optimized Certainty Equivalents}{Wang, Zhang, Lyu, Dai, Gao, and Qiu}
\firstpageno{1}

\begin{document}

\title{Provably Efficient Risk-Sensitive Offline Reinforcement Learning via Optimized Certainty Equivalents}

\author{\name Hanzhen Wang \email hanzhwang9-c@my.cityu.edu.hk \\
       \addr City University of Hong Kong\\
       Hong Kong, China
       \AND
       \name Boxiang Lyu \email blyu@alumni.chicagobooth.edu \\
       \addr University of Chicago\\
       Chicago, IL, USA
       \AND
       \name Zhongxiang Dai \email daizhongxiang@cuhk.edu.cn \\
       \addr The Chinese University of Hong Kong, Shenzhen\\
       Shenzhen, Guangdong, China
       \AND
       \name Siyang Gao \email siyangao@cityu.edu.hk \\
       \addr City University of Hong Kong\\
       Hong Kong, China
       \AND
       \name Dake Zhang\textsuperscript{*} \email dk.zhang@sjtu.edu.cn \\
       \addr Shanghai Jiao Tong University\\
       Shanghai, China
       \AND
       \name Shuang Qiu\thanks{Corresponding author} \email shuanqiu@cityu.edu.hk \\
       \addr City University of Hong Kong\\
       Hong Kong, China
       }

\editor{}

\maketitle

\begin{abstract}

Risk-sensitive reinforcement learning (RL) addresses decision-making under uncertainty and adverse outcomes. Learning risk-sensitive policies from fixed datasets is important when additional interaction with the environment is costly. Existing work on risk-sensitive RL focuses on online settings or specific risk measures, leaving a general theoretical understanding of its offline setting underdeveloped. In this work, we investigate efficient offline risk-sensitive RL using the optimized certainty equivalent (OCE), a general class of risk measures. We consider two formulations: dynamic-OCE RL, which applies OCE recursively at each step, and static-OCE RL, which applies a single OCE to the cumulative rewards. Under the tabular setting, we develop pessimistic value-iteration algorithms for both formulations and establish tight sublinear suboptimality bounds, explicitly capturing how risk sensitivity, encoded by utility functions, affects learning difficulty. Moreover, we establish minimax suboptimality lower bounds to characterize the problem’s fundamental limits. To handle large state spaces, we further develop pessimistic least-squares value-iteration algorithms that incorporate the general OCE risk criterion into linear function approximation. We establish utility-dependent sublinear guarantees with explicit dependence on the feature dimension $d$. To our knowledge, we provide the first offline algorithms along with the first upper and lower suboptimality bounds for both dynamic- and static-OCE RL.

\end{abstract}


\addtocontents{toc}{\protect\setcounter{tocdepth}{-1}} 
\section{Introduction}
Classical reinforcement learning (RL) aims at maximizing the expected cumulative reward by optimizing the policy in the unknown environment through learning \citep{Sutton1998}, typically modeled as a Markov Decision Process (MDP) \citep{Puterman2014}. 
However, in many real-world applications, agents are often required to manage the trade-off between risk and reward. This has motivated the development of risk-sensitive reinforcement learning (risk-sensitive RL) frameworks, which incorporate various risk measures into the learning objective.
Risk-sensitive RL is widely used by a variety of risk-sensitive industries, ranging from finance~\citep{Hambly2023}, self-driving~\citep{kamran2020risk}, to wireless networks~\citep{khalifa2019risk}. In risk-sensitive RL, the agent aims to optimize certain risk-sensitive reward metrics such as mean-variance risk measure~\citep{sood2023deep, huang2022achieving}, entropic risk~\citep{hau2023entropic}, and conditional value-at-risk (CVaR)~\citep{hakobyan2019risk}. 
Despite the progress in risk-sensitive RL, the majority of existing works are confined to the online setting, in which the agent learns through ongoing interactions with the environment, simultaneously balancing exploration and exploitation.
However, the risk-sensitive nature of real-world applications makes this process costly, motivating a line of work on offline risk-sensitive RL~\citep{ma2021conservative, Zhang2024}, in which the agent only has access to a pre-collected dataset and cannot further interact with the environment.

Different risk measures are tailored to different applications and scenarios.  
Therefore, it is crucial to develop a general framework that can flexibly incorporate a broad class of risk measures.
However, existing offline risk-sensitive RL studies often focus on a single risk measure, and there is no algorithm that is provably efficient for offline risk-sensitive RL with general risk measures. In recent years, the optimized certainty equivalent (OCE) framework, introduced by \citet{BenTal2007}, has emerged as a suitable candidate for risk-sensitive RL research due to its ability to unify commonly used risk measures such as CVaR, entropic risk measure, and mean-variance. 
Previous research on risk-sensitive RL can be broadly categorized into two paradigms based on how the risk measure is incorporated. The first paradigm introduces the risk measure at each step of the decision process, as in \citet{Chen2023} and \citet{Xu2023}. This approach integrates risk locally at every stage of learning, leading to a modified Bellman equation and resulting in a dynamic formulation of the objective.
And the second paradigm formulates risk-sensitive RL as a two-stage optimization problem, where the reinforcement learning component acts as a subroutine within a larger risk-aware objective, such as in \citet{Wang2023} and \citet{Wang2024}. This static perspective evaluates risk based on the total return, rather than stage-wise rewards, and facilitates the use of history-dependent policies, offering greater flexibility in algorithm design.
Although these works are sufficiently general in terms of risk metrics considered, they only consider the online setting, with little guidance on how to best utilize a pre-collected dataset. The gap in research highlights an intriguing question: 
\begin{center}
\textit{Can we design offline risk-sensitive RL algorithms that are provably efficient for the general OCE risk measure?}    
\end{center}
Answering this question requires addressing four challenges. 
First, while pessimism is well understood in risk-neutral offline RL~\citep{Levine2020,Jin2020,  Xie2021, Rashidinejad2021, NguyenTang2023}, it is unclear how pessimistic estimators can be constructed in offline risk-sensitive RL with general OCE risk measures, as earlier research relied on the mathematical properties of specific risk measures~\citep{Zhang2024}. 
Second, dynamic and static risk formulations pose distinct challenges for algorithm design and theoretical analysis. Dynamic OCE entails stagewise propagation of nonlinear risk evaluations, whereas static OCE evaluates the total return through a single risk functional and requires additional machinery to recover a recursive formulation. Developing algorithms and establishing theoretical guarantees for each formulation therefore require distinct technical treatments, yet a systematic investigation of both in offline RL remains lacking.
Third, to the best of our knowledge, minimax suboptimality lower bounds for offline risk-sensitive RL under general OCE risk measures have not been established.
Finally, extending the offline OCE-RL framework to settings with linear function approximation introduces further technical challenges. In the dynamic formulation, the nonlinear utility transformation and auxiliary OCE optimization require uniform estimation guarantees over a richer value-function class, while the static formulation requires augmenting the state with a budget variable and controlling the resulting discretization error. Although linear function approximation has been studied for specific risk measures, theoretical guarantees for offline RL under general OCE risk measures remain lacking for both dynamic and static formulations.

\noindent\textbf{Our Contributions.} We make the following four main contributions as we derive a unifying framework for offline risk-sensitive RL with general risk measures. 
\begin{itemize}
  \item We develop provably efficient offline algorithms for both dynamic-OCE and static-OCE formulations. These algorithms provide a common algorithmic foundation for both two major paradigms of risk-sensitive RL and clarify how pessimism can be designed in each case.
  \item We prove sublinear suboptimality bounds for the proposed algorithms, which can recover to risk-neutral RL and various risk-sensitive RL as special cases. These results are the first suboptimality upper bounds for offline risk-sensitive RL with general risk measures, showing that our algorithms are provably efficient.
  \item We establish the first information-theoretic lower bound for offline risk-sensitive RL with general risk measures, which holds for multiple types of offline risk-sensitive RL with the property of OCE. The bound identifies the intrinsic statistical difficulty of the problem and shows how risk sensitivity fundamentally influences the sample complexity.
  \item We generalize our results to the linear function approximation setting, which is the first provably efficient risk-sensitive RL algorithm with linear function approximation for OCE. This extension broadens the scope of our theory beyond tabular MDPs and improves its relevance to large-scale applications.
\end{itemize}
Taken together, these contributions form a coherent theoretical framework for offline risk-sensitive RL with OCE. They directly address the central question raised above by establishing provably efficient algorithms, matching upper and lower bound analyses, and extending to linear function approximation.
\subsection{Related Works} 
This work builds upon a growing body of research on offline risk-neutral RL, where the central goal is to identify optimal policies using only pre-collected datasets, without additional interaction with the environment~\citep{Levine2020}.  
In such a setting, the agent is required to infer the optimal policy exclusively from the dataset with no direct access to the underlying transition dynamics.  
A series of recent studies have investigated this challenge from multiple aspects, leading to a rich line of results \citep{Chen2019, Jin2020,Rashidinejad2021, Xie2021, Cheng2022, NguyenTang2023, Madhow2025, Xiong2022, Liu2026, qiao2026model, qiao2025cross, hao2026fastratesofflineqlearning,jiang2025offline,ye2023corruption,xu2025tackling}.

Our work is closely related to a long line of research on risk-sensitive RL. For the entropic risk measure, \citet{Fei2020} proposed an online algorithm in the tabular MDP setting, which was later extended to the function approximation regime in \citet{Fei2021a}. For iterated CVaR, \citet{Du2022} introduced a tabular algorithm, while \citet{Chen2023} extended this framework to incorporate function approximation. \citet{Xu2023} developed a dynamic-OCE-based algorithm for online tabular MDPs. In the offline setting, \citet{Zhang2024} proposed a linear function approximation method with entropic risk.
In terms of static-OCE risk formulations, \citet{Wang2023} studied the online tabular CVaR-RL problem and further extended their framework to the more general OCE measure in \citet{Wang2024}.  Beyond these, a number of works have laid theoretical foundations for risk-sensitive RL \citep{Osogami2012, Shen2013, Baeuerle2014, Prashanth2014, Shen2014, Baeuerle2024, Liang2024a, Ma2025, Ghosh2025, Deng2025, Lee2026, Mortensen2026, Essakine2026}.

There are also a number of works that focus on linear function approximation, which are closely related to our work. 
\citet{Zhang2024} introduced a linear function approximation method for offline RL under the entropic risk metrics, which, however, did not address the more general and challenging OCE setting. Our algorithmic design is further motivated by a broader set of advances in function approximation~\citep{Cai2020, Jin2019, Jin2020, Wang2020, Agarwal2020, Zanette2021a, Qiu2022, Liu2023, Modi2024, Di2024, Golowich2024, Qiu2024a, Du2025, Lee2025}.

\section{Problem Setup} \label{sec:setting}

In this section, we formalize the offline RL setting, introduce the optimized certainty equivalent (OCE) as a unifying risk measure together with its key properties, and formulate both dynamic and static OCE-based RL problems.

\subsection{Offline Reinforcement Learning}
We define an episodic Markov decision process (MDP) $\mathcal{M}$ using the tuple $(\mathcal{S}, \mathcal{A}, \mathbb{P}, r, H)$, where $\mathcal{S}$ denotes a (possibly infinite) state space, $\mathcal{A}$ denotes a finite action space, and $H$ denotes the horizon. We let $\mathbb{P} = \{\mathbb{P}_h\}_{h=1}^{H}$ denote the transition kernel, where $\mathbb{P}_h(s'  |  s, a)$ is the probability of transitioning to state $s' \in \mathcal{S}$ from state $s \in \mathcal{S}$ upon taking action $a \in \mathcal{A}$ at step $h$. We assume a deterministic reward function $r = \{r_h\}_{h=1}^{H}$, where $r_h: \mathcal{S} \times \mathcal{A} \to [0, 1]$. We assume that both $\mathbb{P}$ and $r$ are unknown a priori, and, without loss of generality, that the initial state is fixed at $s_1$.
We further assume that a pre-collected dataset has been generated via a behavior policy, which is formalized below.
\begin{assumption}\label{assumption: Offline Dataset} Let $\mathcal{D} = \{ ( s_h^k, a_h^k, r_h(s_h^k, a_h^k) ) \}_{h=1,k=1}^{H,K}$ denote a pre-collected dataset comprising $K$ trajectories, where $s_h^k$ and $a_h^k$ are the state and action at step $h$ in the $k$-th trajectory. We assume that this dataset is produced by an unknown behavior policy $\mu$ through the interaction with the environment.
\end{assumption}
Under Assumption~\ref{assumption: Offline Dataset}, the learner has access only to the fixed dataset $\mathcal{D}$ and cannot collect additional data by interacting with the environment. This setting is standard in offline RL and reflects applications where data are obtained from historical logs or human demonstrations. The quality of the learned policy will depend on how well the state-action distribution induced by the behavior policy $\mu$ covers the support of the optimal policy, which is captured by the concentrability coefficient defined in Definition~\ref{assumption: Markov property of dataset}.

For any policy $\pi$, we denote its state-action occupancy as
$d_h^{\pi}(s, a) := \Pr(s_h = s,\, a_h = a  |  \pi, s_1)$ and its state occupancy as $d_h^{\pi}(s) := \Pr(s_h = s  |  \pi, s_1)$. In line with existing risk-neutral offline RL research \citep{Xie2021, Rashidinejad2021}, we define the single concentrability coefficient as follows.
\begin{definition}[Single Policy Concentrability]\label{assumption: Markov property of dataset} For an optimal policy $\pi^*$, we define $C^*$ to be the smallest value such that
        $\max_{h \in [H],\, (s, a) \in \mathcal{S} \times \mathcal{A}} \frac{d_h^{\pi^*}(s, a)}{d_h^{\mu}(s, a)} \leq C^*$. 
\end{definition}
We note that $C^*$ in Definition~\ref{assumption: Markov property of dataset} measures the worst-case ratio between the state-action occupancy under the optimal policy $\pi^*$ and that under the behavior policy $\mu$. The single policy concentrability coefficient $C^*$ quantifies the difficulty of offline learning under a fixed dataset and is crucial to the sample complexity guarantees in this paper.

\subsection{Optimized Certainty Equivalent}
This work focuses on risk-sensitive offline RL incorporating a general risk measure named optimized certainty equivalent (OCE) \citep{BenTal2007}.
\begin{definition}[Optimized Certainty Equivalent] \label{def_definition_of_OCE}
Let $u: \mathbb{R} \to [-\infty, +\infty)$ be a closed, non-decreasing, and concave utility function with a non-empty effective domain. The OCE of a random variable $ X $ is defined as
\begin{equation}\label{eq: definition of OCE}
	\OCE^u(X)= \textstyle \sup_{b \in \mathbb{R}}\big\{b + \mathbb{E}\big[u(X - b)\big]\big\}.
\end{equation}	
\end{definition}
Generally, OCE is a sufficiently general risk measure. Depending on the choice of the utility function $u$, which typically satisfies $u(0) = 0$ and $1 \in \partial u(0)$, OCE recovers commonly used risk metrics such as conditional value-at-risk (CVaR), entropic risk, and mean variance. Table~\ref{tab:oce_reduction} presents several specific instantiations of OCE \citep{Xu2023, Wang2024}. 

\begin{table}[t]
\caption{Special cases of OCE for various $u$.}
    \centering
    \begin{tabular}{lll}
         \toprule
         Name &  $\OCE^u(X)$ & Utility function $u$\\
         \midrule
         Mean & $\mathbb{E}[X]$ & $u(t)=t$\\
         Entropic risk & $\frac{1}{\alpha}\log\mathbb{E}[e^{\alpha X}]$ & $u(t)=\frac{1}{\alpha}e^{\alpha t}-\frac{1}{\alpha}$\\
         Mean-Variance & $\mathbb{E}[X]-c\Var(X)$ & $u(t)=(t-ct^2)\mathbb{I}\{t\leq \frac{1}{2c}\}+\frac{1}{4c}\mathbb{I}\{t>\frac{1}{2c}\}$ \\
         CVaR & $\mathbb{E}[x|X\leq \min\{x|F_X(x)\geq \alpha\}]$ & $u(t)=-\frac{1}{\alpha}[-t]_+$\\
         \bottomrule
    \end{tabular}
    
    \label{tab:oce_reduction}
\end{table}
Table~\ref{tab:oce_reduction} shows that an appropriate choice of the utility function $u$ enables OCE to recover a broad spectrum of risk-sensitive criteria in the prior RL literature. Throughout the paper, both the dynamic and static-OCE RL formulations can be specialized to particular risk measures by choosing the corresponding $u$. In addition to its generality, OCE has several key properties, as summarized below~\citep{BenTal2007, Xu2023, Wang2024}.
\begin{property}\label{prop:utility} For any utility function $u$ satisfying the properties in Definition~\ref{def_definition_of_OCE}, any constant $c \in \mathbb{R}$, and any bounded random variable $ X$, the OCE satisfies the following properties:
\begin{enumerate}[leftmargin=*,itemsep=2pt]
    \item $\OCE^u(X+c)=\OCE^u(X)+c$.
    \item $\OCE^u(c)=c$.
    \item If $X_1(\omega)\leq X_2(\omega)$  for any event $\omega \in \Omega$, $\OCE^u(X_1)\leq\OCE^u(X_2)$.
    \item For any $\mu\in(0,1)$, $\OCE^u(\mu X_1+(1-\mu)X_2)\geq\mu\OCE^u(X_1)+(1-\mu)\OCE^u(X_2)$.
\end{enumerate}
    
\end{property}
These properties show that OCE behaves as a well-structured general risk measure. In this paper, we consider two complementary formulations of risk-sensitive RL based on OCE, namely dynamic-OCE RL and static-OCE RL, and exploit the aforementioned properties in our theoretical analysis. To control the propagation of estimation errors over the finite horizon, we further impose a regularity condition on the utility function.

\begin{assumption}\label{assump: Lipschitz utility}
The utility function satisfies $u(0)=0$ and is $L_u$-Lipschitz on $[-H,H]$: there exists $L_u\geq1$ such that
\[
    |u(x)-u(y)|\leq L_u|x-y|,
    \qquad \forall x,y\in[-H,H].
\]
\end{assumption}
All theoretical results below are stated under Assumption~\ref{assump: Lipschitz utility}. In particular, $|u(x)|\leq HL_u$ for every $x\in[-H,H]$, and the range of $u$ over this interval is at most $2HL_u$. The Lipschitz condition is an additional regularity requirement beyond the general OCE definition, and it holds on $[-H,H]$ for the commonly used utility functions in Table~\ref{tab:oce_reduction}. Specifically, one may take $L_u=1$ for the risk-neutral utility $u(t)=t$, $L_u=1/\alpha, \alpha\in(0,1]$ for the CVaR utility $u(t)=-\alpha^{-1}[-t]_+$, $L_u=e^{|\alpha|H}$ for the entropic utility $u(t)=(e^{\alpha t}-1)/\alpha$ with $\alpha<0$, and $L_u=1+2cH$ for the truncated mean--variance utility in Table~\ref{tab:oce_reduction}. Hence, the assumption accommodates the principal risk measures considered in this paper while explicitly characterizing the utility-dependent regularity required by the finite-horizon analysis.

\subsection{Dynamic-OCE RL}
Under the dynamic-OCE RL setting, we consider a Markovian policy $\pi = \{\pi_h\}_{h=1}^{H}$, where $\pi_h(a | s)$ denotes the probability of taking action $a$ at state $s$ and step $h$, and let $\Pi$ denote the associated policy class. To incorporate OCE into sequential decision-making, prior works~\citep{Ruszczynski2010, Baeuerle2022, Xu2023} have proposed the dynamic-OCE RL framework, leading to the following Bellman equation, where the OCE risk measure is applied recursively from step $H$ to $1$:
\begin{equation*}
  \begin{aligned}
    &Q_h^\pi(s_h,a_h)=r_h(s_h,a_h)+\OCE_{s_{h+1}\sim \mathbb{P}_h(\cdot|s_h,a_h)}^u\big(V_{h+1}^\pi(s_{h+1})\big),\\
    &V_h^\pi(s_h)=\big\langle Q_h^\pi\big(s_h,\cdot\big),\pi_h(\cdot |  s_h)\big\rangle_\mathcal{A},
  \end{aligned}
\end{equation*}
where $V_h^\pi$ and $Q_h^\pi$ denote the value function and Q-function for dynamic-OCE RL, respectively, at step $h$ under policy $\pi$. With a slight abuse of notation, we write $\OCE_{s'\sim \mathbb{P}_h(\cdot|s,a)}^u\big(\allowbreak V_{h+1}^\pi(s')\big) := \sup_{b \in \mathbb{R}}\{b + \mathbb{E}_{s'\sim \mathbb{P}_h(\cdot|s,a)}[u(V_{h+1}^\pi(s') - b)]\}$. We note that for dynamic-OCE RL, the Q-function is 
\begin{align*}
Q_h^\pi(s_h,a_h) 
:= &
r_h(s_h,a_h)+\OCE_{\begin{subarray}{l} s_{h+1}\sim\mathbb{P}_{h}\\
a_{h+1}\sim\pi_{h+1}\\
\end{subarray}}^u\{ r_{h+1}(s_{h+1},a_{h+1})\\
+&
\OCE_{\begin{subarray}{l} s_{h+2}\sim\mathbb{P}_{h+1}\\
a_{h+2}\sim\pi_{h+2}\\
\end{subarray}}^u\{ r_{h+2}(s_{h+2},a_{h+2})
+
\{ \cdots \OCE_{\begin{subarray}{l} s_H\sim 
\mathbb{P}_{H-1}\\
a_H\sim\pi_H\\
\end{subarray}}^u\{r_H(s_H,a_H)\} \} \}\},
\end{align*}
where \(s_{h+1}\sim\mathbb{P}_h, a_{h+1}\sim\pi_{h+1}\) is shorthand for \(s_{h+1}\sim\mathbb{P}_h(\cdot|s_h,a_h), a_{h+1}\sim\pi_{h+1}(\cdot|s_{h+1})\) in this equation.
By the definition of OCE in \eqref{eq: definition of OCE}, there exists an optimal policy $\pi^*=\{\pi^*_h\}_{h=1}^H$ such that $\pi^*= \arg\max_{\pi\in\Pi} V_1^\pi(s_1)$ \citep{Baeuerle2022}.
We measure the learning performance of a policy $\pi$ in dynamic-OCE RL by its suboptimality, defined as
\begin{equation*}
    \SubOptD(\pi) = V_1^{\pi^*}(s_1)-V_1^\pi(s_1),
\end{equation*}
which quantifies the gap between the value of the optimal policy $\pi^*$ and that of policy $\pi$. A policy $\pi$ is $\varepsilon$-approximately optimal if $\SubOptD(\pi) \leq \varepsilon$.

\subsection{Static-OCE RL}
The static-OCE RL framework considers an alternative risk-sensitive formulation based on the total return rather than an iterative recursion. Since the optimal policy may be history-dependent in this setting, we adopt the following specialized definition of the learning objective:
\begin{align}
	\OCE^u_{\pi, \mathbb{P}}\big(\textstyle\sum_{h=1}^H r_h(s_h,a_h)\big) = \textstyle\sup_{b\in[0,H]}\big\{b+\mathbb{E}_{\pi, \mathbb{P}} [u(\textstyle\sum_{h=1}^H r_h(s_h,a_h)-b)]\big\}, \label{eq:static-oce}
\end{align}
where $\mathbb{E}_{\pi, \mathbb{P}}$ denotes the expectation over trajectories generated according to $a_h\sim \pi_h, s_{h+1}\sim\mathbb{P}_h$ for all $h\in[H]$. Note that by Proposition~\ref{prom: finite b} in the appendix, the value of $b$ in \eqref{eq:static-oce} can be restricted to $[0, H]$. 

A critical challenge for this setting is that the optimal policy for the above problem is \emph{history-dependent} \citep{Wang2024}. To address this, we employ the augmented MDP \citep{Baeuerle2011, Wang2024, Baeuerle2021} with an expanded state space $ (s_h, b_h) \in \mathcal{S_{\text{aug}}}:=\mathcal{S} \times [0, H]$ for each step $h$, comprising the state $s_h$ and a residual budget variable $b_h$ that transitions via $b_{h+1} = b_h - r_h$ with $b_1 \in [0, H]$ chosen by the learning algorithm. Thus, $b_h$ encodes the cumulative rewards collected up to step $h-1$ through the remaining budget. Under such construction, we define a \emph{Markovian} policy of the form $\pi_h(a_h|s_h,b_h)$ on the expanded state space $\mathcal{S_{\text{aug}}}$. We define the augmented value functions as $V_h^\pi(s_h,b_h):=\mathbb{E}_{\pi, \mathbb{P}}[u(\sum_{h'=h}^{H} r_{h'}(s_{h'},a_{h'})-b_h)| s_h,b_h]$. Then, a Bellman-like equation is given by
\begin{equation*}
  \begin{aligned}
    &Q_h^\pi(s_h,b_h,a_h)=\mathbb{E}_{s_{h+1}\sim \mathbb{P}_h(\cdot|s_h,a_h)} \big[V_{h+1}^\pi(s_{h+1},b_{h+1})\big],\\
    &V_h^\pi(s_h,b_h)=\langle Q_h^\pi(s_h,b_h,\cdot), \pi_h(\cdot|s_h,b_h)\rangle_\mathcal{A}.
  \end{aligned}
\end{equation*}
where we have $b_{h+1} = b_h - r_h$. By the definition of $V_h^\pi(s_h,b_h)$, we have $V_{H+1}^\pi(s, b) = u(-b), \forall (s, b)$. Further by \eqref{eq:static-oce}, static-OCE RL equivalently solves $\max_\pi\sup_{b_1\in[0,H]}\{b_1+V_1^{\pi}(s_1, b_1)\}$, where $\pi$ is the Markovian policy defined on $\mathcal{S_{\text{aug}}}$. There exists an optimal initial budget $b_1^*$ and an optimal policy $\pi^* := \{\pi_h^*\}_{h=1}^H$ such that $\pi^*$ with $b_1^*$ can maximize $\sup_{b_1\in[0,H]}\{b_1+V_1^{\pi}(s_1, b_1)\}$ \citep{Wang2024}. 
Ideally, $b_h$ at step $h$ should be treated as a continuous variable over the interval $[0, H]$. However, for practicability and computational efficiency, we discretize $b_h$ using a $\varsigma$-net over $[0,H]$, defined as $\mathcal{N}_\varsigma^b:=\{n\varsigma:n=0,\ldots,\lfloor H/\varsigma \rfloor\}\cup \{H\}$. In the following sections, we focus on bounding the sample complexity and defer the discussion of discretization to Section~\ref{subsec_discussion_on_discretization}, as it affects computation more than the theoretical analysis. Accordingly, the suboptimality of any policy $\pi$ in static-OCE RL can be defined as
\begin{equation*}
    \SubOptS(\pi):=\textstyle\sup_{b_1\in[0,H]}\big\{b_1+V_1^{\pi^*}(s_1, b_1)\big\}-\textstyle\sup_{b_1\in[0,H]}\big\{b_1+V_1^{\pi}(s_1, b_1)\big\}.
\end{equation*}
For a policy implemented on the budget grid $\mathcal{N}_\varsigma^b$, we define the discrete-budget suboptimality by
\begin{equation*}
    \SubOptS^{\mathtt{dsc}}(\pi):=\textstyle\sup_{b_1\in[0,H]}\big\{b_1+V_1^{\pi^*}(s_1,b_1)\big\}-\textstyle\sup_{b_1\in\mathcal{N}_\varsigma^b}\big\{b_1+V_1^{\pi}(s_1,b_1)\big\}.
\end{equation*}

For dynamic-OCE and static-OCE RL in the offline setting, our goal is to find policies $\widehat{\pi}$ in their corresponding policy classes such that $\SubOptD(\widehat{\pi})$, $\SubOptS(\widehat{\pi})$, or $\SubOptS^{\mathtt{dsc}}(\widehat{\pi})$ is sufficiently small, as appropriate.

\section{Risk-Sensitive Offline RL with OCE}\label{sec: Tabular Offline RL with OCE}

In this section, we develop and analyze the learning algorithms for risk-sensitive offline RL in the tabular setting, considering both dynamic-OCE RL and static-OCE RL formulations.

\subsection{Dynamic-OCE Pessimistic Value Iteration}

\begin{algorithm}[!t]
  \caption{\texttt{DOCE-PVI}: Dynamic-OCE Pessimistic Value Iteration}\label{alg: recursiv tabular OCE-PVI}
  \begin{algorithmic}[1]
    \State \textbf{Input:} Offline data $\mathcal{D}=\{ ( s_h^k, a_h^k, r_h(s_h^k, a_h^k)) \}_{h=1,k=1}^{H,K}$
    \State \textbf{Initialize:} $\widehat{V}_{H+1}(s)= 0$ for all $s$
    \For{$h = H, H-1\dots,1$ ~~\textbf{and}~~ all $(s,a)\in\mathcal{S}\times\mathcal{A}$}
    \State Compute reward and transition estimates, $\widehat{r}_h(s,a)$ and $\widehat{\mathbb{P}}_h(\cdot|s,a)$, via \eqref{eq:estimate} using $\mathcal{D}$
    \State $\Gamma_h(s,a)=\sqrt{\frac{1}{\max\{1,N_h(s,a)\}}}+[u(H-h)-u(h-H)]\sqrt{\frac{2\log(|\mathcal{S}||\mathcal{A}|HK/\delta)}{\max\{1,N_h(s,a)\}}}$\label{line: dynamic-OCE meta rl bonus}
    \State \label{line: dynamic-OCE meta rl key}$\overline{Q}_h(s,a)= \widehat{r}_h(s,a)+\OCE_{s'\sim\widehat{\mathbb{P}}_h(\cdot |  s,a)}^u\Big[\widehat{V}_{h+1}(s')\Big]-\Gamma_h(s,a)$
    \State  $\widehat{Q}_h(s,a)=\clip\big\{\overline{Q}_h(s,a),[0,H-h+1]\big\}$\label{line: dynamic-OCE meta rl truncate}
    \State $\widehat{\pi}_h(\cdot|\cdot)=\argmax_{\pi_h}\big\langle\widehat{Q}_h(\cdot,\cdot),\pi_h(\cdot|\cdot)\big\rangle_\mathcal{A}$\label{line: dynamic-OCE meta rl get pi}
    \State $\widehat{V}_h(\cdot)=\big\langle\widehat{Q}_h(\cdot,\cdot),\widehat{\pi}_h(\cdot|\cdot)\big\rangle_\mathcal{A}$\label{line: dynamic-OCE meta rl get V}
    \EndFor
    \State \textbf{Return:} $\widehat{\pi}=\big\{\widehat{\pi}_h\big\}_{h=1}^H$.
  \end{algorithmic}
\end{algorithm}

\textbf{Algorithm.} For the dynamic-OCE RL setting, we propose a pessimistic value iteration algorithm termed \textbf{D}ynamic-\textbf{OCE} \textbf{P}essimistic \textbf{V}alue \textbf{I}teration (\texttt{DOCE-PVI}) as summarized in Algorithm~\ref{alg: recursiv tabular OCE-PVI}. This algorithm first constructs the estimated transition $\widehat{\mathbb{P}}_h$ and estimated reward $\widehat{r}_h$ via  
\begin{equation}
\begin{aligned}
&\widehat{\mathbb{P}}_h(s'\mid s,a)=\frac{N_h(s,a,s')}{\max\{1,N_h(s,a)\}},\qquad \widehat{r}_h(s,a)=\frac{\sum_{k=1}^K\mathbb{I}\{(s_h^k,a_h^k)=(s,a)\}r_h(s_h^k,a_h^k)}{\max\{1,N_h(s,a)\}}, \label{eq:estimate}
\end{aligned}
\end{equation}
where $\mathbb{I}\{\cdot\}$ is an indicator function and $N_h(s,a,s')$ and $N_h(s,a)$ are the state-action visitation counters for the pre-collected data $\mathcal{D}$, defined as $N_h(s,a,s')=\sum_{k=1}^{K}\mathbb{I}\{(s_h^k,a_h^k,s_{h+1}^k)=(s,a,s')\}$ and $N_h(s,a)=\sum_{k=1}^{K}\mathbb{I}\{(s_h^k,a_h^k)=(s,a)\}$. The bonus term $\Gamma_h$ is constructed on Line~\ref{line: dynamic-OCE meta rl bonus}, which measures the uncertainty arising from model estimation. This term explicitly incorporates the OCE risk measure via the factor $u(H-h)-u(h-H)$, whose specific form is determined by the chosen utility function $u$. Lines~\ref{line: dynamic-OCE meta rl key} and \ref{line: dynamic-OCE meta rl truncate} compute a pessimistic estimate of the Q-function $\widehat{Q}_h$ by applying the Bellman equation that incorporates the OCE risk measure iteratively in the dynamic-OCE RL setup. The $\clip\{x,[x_1,x_2]\}$ operator on Line \ref{line: dynamic-OCE meta rl truncate} projects $x$ into the interval $[x_1,x_2]$ to ensure boundedness. The estimated optimal policy at step \(h\), denoted by \(\widehat{\pi}_h\), is a greedy policy based on the pessimistic Q-function estimate $\widehat{Q}_h$, and the estimated value function \(\widehat{V}_h\) is then obtained using the learned policy. We note that the algorithm reduces to the risk-neutral pessimistic value iteration when $u(t) = t$. Line \ref{line: dynamic-OCE meta rl key} involves an optimization problem of the form $\OCE_{s'\sim\widehat{\mathbb{P}}_h(\cdot |  s, a)}^u[\widehat{V}_{h+1}(s')] = \sup_{b\in[0, H-h]}\sum_{s'\in\mathcal{S}}\widehat{\mathbb{P}}_h(s'| s, a)[b+u(\widehat{V}_{h+1}(s')-b)]$ (by Proposition \ref{prom: finite b} in Appendix \ref{sec: properties of OCE}) depending on the choice of \(u\). Since \(u\) is concave as shown in Property \ref{prop:utility}, this becomes a one-dimensional concave maximization problem which can be efficiently solved.

\vspace{5pt}
\noindent\textbf{Theoretical Result.} We establish a suboptimality bound for Algorithm \ref{alg: recursiv tabular OCE-PVI} below.
\begin{theorem}\label{the: dynamic-OCE tabular result}
    Under the offline tabular setting of dynamic-OCE RL, with probability at least $1-\delta$ for $\delta \in (0,1)$, the learned policy $\widehat{\pi}$ via Algorithm~\ref{alg: recursiv tabular OCE-PVI} admits the following suboptimality bound
    \begin{align}\label{the: subopt dynamic-OCE meta}
      \SubOptD(\widehat{\pi})
      \leq& 
      \widetilde{\mathcal{O}}\Big(
      \textstyle\sum_{h=1}^H \big[1+u(H-h)-u(h-H)\big]\sqrt{C^* |\mathcal{S}| /K}\Big)\\\nonumber
      \leq&
      \widetilde{\mathcal{O}}\Big(
      \textstyle\sum_{h=1}^H \big[1+2(H-h)L_u\big]\sqrt{C^* |\mathcal{S}| /K}\Big),
    \end{align}
    where $\widetilde{\mathcal{O}}$ hides logarithmic dependence on $H, |\mathcal{S}|, K$, and $1/\delta$.
\end{theorem}

This theorem demonstrates that \texttt{DOCE-PVI} achieves a suboptimality upper bound with $\widetilde{\mathcal{O}}(\sqrt{C^*|\mathcal{S}|/K})$ dependence on coverage, state-space size, and sample size. This matches the dependence with risk-neutral offline value-iteration methods under comparable coverage and sampling assumptions \citep{Rashidinejad2021}. The factor $u(H-h)-u(h-H)$ is the stage-dependent utility range, while the additive $1$ accounts for reward estimation. Under Assumption~\ref{assump: Lipschitz utility}, this range is at most $2(H-h)L_u$, giving the displayed $L_u$-based form and thereby aligning the dynamic result with the remaining theoretical guarantees. Hence, the theorem quantifies precisely how incorporating risk sensitivity through OCE affects the sample complexity while retaining the parametric $K^{-1/2}$ rate. Here, $L_u$ is the Lipschitz modulus of $u$ introduced in Assumption~\ref{assump: Lipschitz utility}. It quantifies the sensitivity of the OCE objective to return perturbations, therefore a larger $L_u$ leads to a larger utility-dependent statistical error term.

We compare our result with risk-neutral offline value iteration algorithms to highlight the similarities and differences in theoretical guarantees. Compared with the result of \citet{Xie2021}, our algorithm has the same $\widetilde{\mathcal{O}}(\sqrt{C^* |\mathcal S|/K})$ dependence under comparable coverage and data assumptions. In terms of the horizon dependence, our result includes the stage-dependent multiplicative factor $\sum_{h=1}^H[1+u(H-h)-u(h-H)]$, where \(u(H-h)-u(h-H)\) captures the additional effect of risk sensitivity induced by the OCE formulation. This reveals that the suboptimality is affected by the risk preferences encoded in the utility function $u$. Moreover, when $u(t) = t$, the overall error scales as $\widetilde{\mathcal{O}}(H^2)$ under the present convention, which is comparable to Hoeffding-style risk-neutral offline-RL bounds after accounting for differences in the data model and coverage assumptions \citep{Levine2020, Rajaraman2020}. Overall, our algorithm preserves the same parametric \(K^{-1/2}\) dependence as these risk-neutral offline-RL results under comparable assumptions, while introducing an additional utility-dependent term that quantifies the effect of OCE-based risk sensitivity.


\subsection{Static-OCE Pessimistic Value Iteration}

\begin{algorithm}[!t]
  \caption{\texttt{SOCE-PVI}: Static-OCE Pessimistic Value Iteration}\label{alg: static-OCE tabular OCE-PVI}
  \begin{algorithmic}[1]
    \State \textbf{Input:} Offline data $\mathcal{D}=\{ ( s_h^k, a_h^k, r_h(s_h^k, a_h^k)
    ) \}_{h=1,k=1}^{H,K}$
    \State \textbf{Initialize:} $\widehat{V}_{H+1}(s,b)= u(-b)$ for all $s\in\mathcal{S}$ and $b\in\mathcal{N}_\varsigma^b$
    \For{$h = H, H-1\dots,1$ ~~\textbf{and}~~ all $(s,a,b)\in\mathcal{S}\times\mathcal{A}\times\mathcal{N}_\varsigma^b$}
    \State Estimate $\widehat{\mathbb{P}}_h(\cdot|s,a)$ and $\widehat{r}_h(s,a)$ using $\mathcal{D}$ via \eqref{eq:estimate}
    \State $\Gamma_h(s, a)=2HL_u\sqrt{\frac{2\log(|\mathcal{S}||\mathcal{A}|HK/\delta)}{\max\{1,N_h(s,a)\}}}$\label{line: static-OCE meta rl bonus}
    \State Let $b':=b-\widehat{r}_h(s,a)$ \label{line:iterated_r}
    \State\label{line: static-OCE meta rl key} $\overline{Q}_h(s,a,b)= \mathbb{E}_{s'\sim\widehat{\mathbb{P}}_h(\cdot|s,a)}[\widehat{V}_{h+1}(s',b')]-\Gamma_h(s,a)$

    \State  $\widehat{Q}_h(s,a,b)=\clip\big\{\overline{Q}_h(s,a,b),[u(-b),u(H-h+1-b)]\big\}$\label{line: static-OCE meta rl truncate}
    \State $\widehat{\pi}_h(\cdot|s,b)=\argmax_{\pi_h}\big\langle\widehat{Q}_h(s,\cdot,b),\pi_h(\cdot|s,b)\big\rangle_\mathcal{A}$\label{line: static-OCE meta rl get pi}
    \State $\widehat{V}_h(s,b)=\big\langle\widehat{Q}_h(s,\cdot,b),\widehat{\pi}_h(\cdot|s,b)\big\rangle_\mathcal{A}$\label{line: static-OCE meta rl get V}
    \EndFor
    \State  $\widehat{b}_1=\argmax_{b\in\mathcal{N}_\varsigma^b}\big\{b+\widehat{V}_1(s_1,b)\big\}$\label{line: static-OCE b}
    \State \textbf{Return:} $\widehat{\pi}=\{\widehat{\pi}_h\}_{h=1}^H$, $\widehat{b}_1$.
  \end{algorithmic}
\end{algorithm}

\textbf{Algorithm.} The static-OCE RL formulation applies the OCE risk measure to the total return only once, whereas dynamic-OCE RL applies it recursively at each step. Consequently, the optimal policy for static-OCE RL may be history-dependent. A detailed discussion can be found in Appendix~\ref{sbusec: Comparison between Recursive OCE-RL and static-OCE RL}. Based on the augmented MDP formulation of static-OCE RL, we propose the \textbf{S}tatic-\textbf{OCE} \textbf{P}essimistic \textbf{V}alue \textbf{I}teration (\texttt{SOCE-PVI}) algorithm in Algorithm~\ref{alg: static-OCE tabular OCE-PVI}.

Based on the augmented MDP formulation, \texttt{SOCE-PVI} encodes the relevant history through the auxiliary budget state $b_h$, which is initialized at the algorithmically chosen budget $b_1$ and evolves according to $b_{h+1}=b_h-r_h(s_h,a_h)$. Thus, a policy of the form $\pi_h(\cdot|\cdot,b_h)$ is Markovian on the augmented state space $\mathcal{S}_\text{aug}$ defined in Section~\ref{sec:setting}, and the optimal static-OCE policy can be found within this augmented Markovian policy class.

In Algorithm~\ref{alg: static-OCE tabular OCE-PVI}, we first obtain the estimated transition $\widehat{\mathbb{P}}_h$ and the estimated reward $\widehat{r}_h$ via \eqref{eq:estimate}. The bonus term $\Gamma_h$, computed in Line~\ref{line: static-OCE meta rl bonus}, measures the model estimation uncertainty for each state-action pair $(s, a)$. Its utility-dependence is controlled by $HL_u$, which uniformly bounds the augmented value function under Assumption~\ref{assump: Lipschitz utility}. Notably, the bonus term used in Algorithm~\ref{alg: static-OCE tabular OCE-PVI} differs from the one in Algorithm~\ref{alg: recursiv tabular OCE-PVI}, reflecting the different structures of dynamic-OCE and static-OCE RL. Line~\ref{line:iterated_r} updates the budget from $b$ to $b'$ using the estimated reward $\widehat{r}_h(s,a)$. Lines~\ref{line: static-OCE meta rl key} and \ref{line: static-OCE meta rl truncate} construct the pessimistic Q-function estimate $\widehat{Q}_h$ through the static-OCE Bellman equation together with the truncation operator $\clip$. Line~\ref{line: static-OCE meta rl get pi} computes the learned policy $\widehat{\pi}_h$ via a greedy optimization of the estimated Q-function. In Line~\ref{line: static-OCE meta rl get V}, the estimated value function $\widehat{V}_h$ is obtained using the learned policy $\widehat{\pi}_h$. Finally, the estimated initial budget $\widehat{b}_1$ is computed in Line~\ref{line: static-OCE b}. 

Algorithm~\ref{alg: static-OCE tabular OCE-PVI} outputs a history-dependent policy represented through the recursively updated budget variable $\widehat{b}_h$, initialized at $\widehat{b}_1$ and updated during execution by $\widehat{b}_{h}=\widehat{b}_{h-1}-r_{h-1}(s_{h-1},a_{h-1})$, where $(s_{h-1},a_{h-1})$ is the observed state-action pair at step $h-1$. In the augmented state space, the learned policy is Markovian and selects actions according to $a_h\sim\widehat{\pi}_h(\cdot|s_h,\widehat{b}_h)$. Owing to the structure of static-OCE RL, Algorithm~\ref{alg: static-OCE tabular OCE-PVI} evaluates the OCE once for the total return, whereas Algorithm~\ref{alg: recursiv tabular OCE-PVI} applies it recursively at every step. This yields a different computational profile because the auxiliary budget state $b$ must be updated and discretized iteratively. For different choices of $u$, the algorithm specializes to the corresponding OCE objective, including the risk-neutral objective when $u(t)=t$. 

Algorithm~\ref{alg: static-OCE tabular OCE-PVI} is implemented over the budget grid $\mathcal{N}_\varsigma^b$. For the theoretical analysis, we first analyze the corresponding Bellman recursion with a continuous budget $b\in[0,H]$. Lemma~\ref{lemma_subopt_static-OCE_discrete} then transfers the continuous-budget guarantee to the grid implementation, incurring an additional discretization error $(1+L_u)\varsigma$ in the suboptimality bound.

\vspace{5pt}
\noindent\textbf{Theoretical Result.} We establish a suboptimality bound for Algorithm \ref{alg: static-OCE tabular OCE-PVI} below.
\begin{theorem}\label{the: static-OCE tabular result}    
    Under the offline tabular setting of static-OCE RL with deterministic rewards $r_h=r_h(s,a)\in[0,1]$, with probability at least $1-\delta$ for $\delta\in(0,1)$, the policy $\widehat\pi$ returned by Algorithm~\ref{alg: static-OCE tabular OCE-PVI} satisfies
    \begin{equation*}
    \SubOptS^{\mathsf{dsc}}(\widehat{\pi})
      \leq 
      \widetilde{\mathcal{O}}\Big(\textstyle\sum_{h=1}^H HL_u\sqrt{C^*|\mathcal{S}|/K}\Big)+(1+L_u)\varsigma,
    \end{equation*}
      where $\widetilde{\mathcal{O}}$ hides logarithmic dependence on $H$, $|\mathcal{S}|$, $K$, and $1/\delta$, and $\SubOptS^{\mathtt{dsc}}$ is defined in Appendix~\ref{subsec_discussion_on_discretization}.
\end{theorem}

The first term in Theorem~\ref{the: static-OCE tabular result} is the continuous-budget statistical error established by the Bellman analysis. The additive term $(1+L_u)\varsigma$ is the discretization error incurred by the grid implementation in Algorithm~\ref{alg: static-OCE tabular OCE-PVI}, as established in Lemma~\ref{lemma_subopt_static-OCE_discrete}. Thus, the theorem provides a guarantee for the implementable discrete-budget algorithm, while $HL_u$ captures the utility-dependent scale of the statistical error. For the idealized continuous-budget case, the same analysis yields $\SubOptS(\widehat{\pi})\leq\widetilde{\mathcal{O}}(\textstyle\sum_{h=1}^H HL_u\sqrt{C^*|\mathcal{S}|/K})$. In this case, no discretization error is incurred.

Similar to Theorem~\ref{the: dynamic-OCE tabular result}, the result exhibits $\widetilde{\mathcal{O}}(\sqrt{C^* |\mathcal{S}|/K})$ dependence in its displayed statistical term, up to logarithmic and utility-dependent factors. This comparison is made under comparable coverage and sampling assumptions \citep{Xie2021}. When $u(t) = t$, the horizon dependence factor scales as $\widetilde{\mathcal{O}}(H^2)$ under the present convention, which is comparable to prior risk-neutral offline-RL results after aligning their sampling and coverage assumptions \citep{Levine2020,Rajaraman2020}. 

The main result above concerns the deterministic-reward implementation in Algorithm~\ref{alg: static-OCE tabular OCE-PVI}. A separate stochastic-reward extension is given in the appendix under a joint transition--reward model. It requires a Bellman update based on the joint law of $(s',r)$, whereas the practical update uses $\widehat r_h(s,a)$.

\section{Information-Theoretic Lower Bounds}\label{sec: Information-Theoretic Lower Bounds of Tabular Offline RL with OCE}

In this section, we establish a minimax lower bound for the suboptimality of offline risk-sensitive RL under the OCE risk measure, which holds for both the dynamic-OCE and the static-OCE formulations. We first state the main lower bound result via the following theorem.

\begin{theorem}
\label{the: lower bound}
    For an MDP $\mathcal{M} = (\mathcal{S}, \mathcal{A}, H, \mathbb{P}, r)$ with $|\mathcal{S}| \geq 3$, $H \geq 2$, $|\mathcal{A}| \geq C^* \geq 2$, and $K > \frac{1}{4}C^*|\mathcal{S}|H$, let $\rho\in(0,1)$ satisfy $\rho H\in\mathbb{N}$. Let $\mathcal{D}$ be a dataset collected from the underlying MDP $\mathcal{M}$ and $\mathrm{Alg}$ be any offline learning algorithm that maps $\mathcal{D}$ to a learned policy. Further let $b_1^*\in\argmax_{b \in [0,\rho H]}\{b + \frac{1}{2H}u(\rho H-b)+(1-\frac{1}{2H})u(-b)\}$. 
    \begin{itemize}
    \item For dynamic-OCE RL, with $\widehat{\pi}$ returned by $\mathrm{Alg}(\mathcal{D})$, the following minimax lower bound holds:
    \begin{align*}
    	\inf_{\mathrm{Alg}}\max_{\mathcal{M}}~~\SubOptD\big(\mathcal{M}, \mathrm{Alg}(\mathcal{D})\big)
    	\geq \Omega \big(\big[u(\rho H-b_1^*)-u(-b_1^*)\big]\sqrt{C^*|\mathcal{S}|H/K} \big).
    \end{align*}

    \item For static-OCE RL, suppose that $\mathrm{Alg}(\mathcal{D})$ returns $(\widehat{\pi},\widehat{b}_1)$ and induces the Markovian policy $\widehat{\pi}_h(\cdot|s,\widehat{b}_h)$ at step $h$, where the budget is updated as $\widehat{b}_{h+1}=\widehat{b}_h-r_h$ starting from $\widehat{b}_1$. Then the following minimax lower bound holds:
    \begin{align*}
    	\inf_{\mathrm{Alg}}\max_{\mathcal{M}}~~\SubOptS\big(\mathcal{M}, \mathrm{Alg}(\mathcal{D})\big)
    	\geq \Omega \big(\big[u(\rho H-b_1^*)-u(-b_1^*)\big]\sqrt{C^*|\mathcal{S}|H/K} \big).
    \end{align*}
    \end{itemize}

\end{theorem}

Theorem~\ref{the: lower bound} provides a minimax lower bound showing that, for any offline algorithm under either dynamic-OCE or static-OCE RL, there exists an instance such that the suboptimality is at least $\Omega\big([u(\rho H-b_1^*)-u(-b_1^*)]\sqrt{C^*|\mathcal{S}|H/K}\big)$. This demonstrates that, in the worst-case scenario, any algorithm operating under these OCE RL formulations must incur a learning suboptimality that jointly depends on the dataset coverage $C^*$, the state space size $|\mathcal{S}|$, the horizon $H$, and the risk preferences captured by the utility function $u$. Moreover, this suboptimality lower bound decays at the rate $K^{-1/2}$. This theorem highlights the intrinsic difficulty of offline risk-sensitive RL under general OCE risk measures, and, to the best of our knowledge, establishes the first lower bounds for offline OCE-RL problems. Notably, accounting for the discretization of $b$ changes the lower bound of static-OCE RL in Theorem~\ref{the: lower bound} by at most an additive $(1+L_u)\varsigma$, as shown in Lemma~\ref{lemma_subopt_static-OCE_discrete}. Hence, budget discretization preserves the order of the lower bound.

Theorem~\ref{the: lower bound} establishes a general lower bound for offline risk-sensitive RL with OCE, rather than a bound being specific to a particular risk measure. Together with the utility function $u$, the choice of $\rho$, which is determined by the intrinsic parameter of the risk measure (e.g., $\alpha$ in CVaR), specializes this bound to concrete risk measures. To the best of our knowledge, our results provide the first lower bounds for offline RL with the general OCE risk measures. 

Our results can be compared with existing risk-sensitive lower bounds, either for the offline setting with specific risk measures or for the online setting. In particular, although a direct comparison with online risk-sensitive RL is not entirely rigorous due to the inherent differences between online and offline settings, such a comparison still provides useful insights into the dependence on key problem parameters. For offline RL, the comparisons are as follows:
\begin{itemize}
  \item Compared with lower bounds for offline risk-neutral RL algorithms \citep{Xie2021, Rashidinejad2021}, our lower bound shares $\Omega(\sqrt{C^*|\mathcal{S}|})$ dependence in the coverage and state-space factor. In addition, it explicitly includes the risk-sensitive factor induced by the utility function $u$, thereby capturing the additional utility-dependent difficulty associated with risk sensitivity.
  \item Compared with offline entropic RL with linear function approximation \citep{Zhang2024}, choosing $u(t)=\tfrac{1}{|\alpha|}(e^{|\alpha|t}-1)$ and an admissible $\rho$ as specified in Theorem~\ref{the: lower bound} yields an exponential utility factor of order $(e^{|\alpha|\rho H}-1)/|\alpha|$. This exhibits an exponential dependence qualitatively comparable to the entropic factor appearing in their upper bound, but the results use different function-approximation and sampling assumptions.
\end{itemize}
For online risk-sensitive RL, the comparisons are as follows:
\begin{itemize}
  \item Compared with online dynamic-OCE RL \citep{Xu2023}, \citet{Xu2023} established a lower bound of the form 
  \(
      \Omega([u((1-2/c_2)H-b_1^*)-u(-b_1^*)]\sqrt{C^* |\mathcal{S}| H K}),\) for \( c_2>2,
  \)
  for online dynamic-OCE RL. Setting $\rho=1-2/c_2$ makes the utility arguments formally comparable, while the offline and online lower bounds are not directly comparable in their statistical rates since the offline and online settings involve different information structures and sample models.
  \item Compared with online dynamic-CVaR RL (iterated-CVaR RL) \citep{Chen2023}, choosing $u(t)=-\tfrac{1}{\alpha}[-t]_+$ with $\alpha \in (0,1]$ enables a comparison of the CVaR-induced utility factor when the chosen $\rho$ lies in $(0,1)$. This comparison concerns the utility factor, while the two lower bounds use different sampling settings.
  \item Compared with online static-CVaR RL \citep{Wang2023}, choosing $u(t)=-\tfrac{1}{\alpha}[-t]_+$ with $\alpha \in (0,1]$ makes the dependence of the utility factor on the CVaR level explicit. However, the online and offline lower bounds are derived under different information and sampling structures.
  \item Compared with online entropic RL \citep{Fei2020}, choosing $u(t)=\tfrac{1}{|\alpha|}(e^{|\alpha|t}-1)$ yields an exponential utility factor whose horizon dependence can be compared with the entropic risk-sensitive lower bound established by \citet{Fei2020}. The comparison addresses the risk-sensitive factor rather than complete equality of rates.
\end{itemize}
Except for the offline entropic-RL result of \citet{Zhang2024}, most existing lower bounds for risk-sensitive RL are established in the online setting. Hence, the above comparisons serve as indicative references rather than direct comparisons. Nevertheless, the risk-sensitive factors identified in these works, such as the term $u((1-2/c_2)H-b_1^*)-u(-b_1^*)$ in \citet{Xu2023}, provide useful benchmarks for understanding the corresponding offline setting in our work.

Moreover, the comparison of Theorem~\ref{the: lower bound} with the upper bounds in Theorems~\ref{the: subopt dynamic-OCE meta} and~\ref{the: static-OCE tabular result} is restricted to the common dependence on coverage, state-space size, and sample size, namely $\sqrt{C^*|\mathcal{S}|/K}$. In addition, Theorem~\ref{the: lower bound} shows that the constructed hard instances yield the same lower-bound form for both dynamic-OCE and static-OCE RL. The detailed construction and justification are provided in Appendix~\ref{sec: lower bounds}.

\section{Linear Function Approximation for Offline RL with OCE}\label{sec: OCE-RL with Linear Function Approximation}

In large state spaces, the proposed algorithms under the tabular setting would incur a large suboptimality bound, according to Theorems \ref{the: dynamic-OCE tabular result} and \ref{the: static-OCE tabular result}. A standard approach to mitigating this dependence is to use function approximation. While function approximation has been extensively studied in RL, developing a theoretically justified algorithm for RL under the OCE risk measure (let alone the offline setting) remains limited. In this section, we take an initial step in this direction by studying offline RL with OCE under linear function approximation, a practical and widely used implementation of function approximation, and we develop learning algorithms for both dynamic-OCE RL and static-OCE RL.

\vspace{5pt}
\noindent\textbf{Linear MDP.} We consider a commonly adopted linear MDP model, in which both the reward function and the transition model admit a linear structure. Specifically, we have
\begin{align}\label{eq:linearMDP}
	r_h(s,a)=\big\langle \theta_h,\phi(s,a) \big\rangle,\qquad \mathbb{P}_h(\cdot|s,a)=\big\langle \mu_h(\cdot),\phi(s,a) \big\rangle, 
\end{align}
where $\int_{\mathcal{S}} \|\mu_h(s)\|_2 \, ds \leq \sqrt{d}$ and $\|\theta_h\|_2 \leq \sqrt{d}$. We define $\phi: \mathcal{S} \times \mathcal{A} \to \mathbb{R}^d $ to be a feature map satisfying $ \|\phi(s, a)\|_2 \leq 1$ for all $(s, a) \in \mathcal{S} \times \mathcal{A}$. The linear MDP model includes the tabular MDP setting as a special case by choosing $d = |\mathcal{S}| \cdot |\mathcal{A}|$ and setting the feature map to the canonical basis vector $\phi(s, a) = \mathbf{e}_{(s, a)}$, when assuming discrete state and action spaces.

\subsection{Dynamic-OCE Pessimistic Least-Squares Value Iteration}

\begin{algorithm}[!t]
  \caption{\texttt{DOCE-PLSVI}: Dynamic-OCE Pessimistic Least-Squares Value Iteration}\label{alg:recursiv OCE-PVI}
  \begin{algorithmic}[1]
  \State \textbf{Input:} Offline data $\mathcal{D}=\{( s_h^k, a_h^k, r_h(s_h^k, a_h^k) ) \}_{h=1,k=1}^{H,K}$
        \State \textbf{Initialize:} $\widehat{V}_{H+1}(s)= 0$ for all $s$
  \For{$h = H, H-1\dots,1$}
            \State $\Lambda_h=\sum_{k=1}^K\phi(s_h^k,a_h^k)\phi(s_h^k,a_h^k)^\top +\lambda\mathbf{I}$\label{line:dynamic_lambda}
            \State $\widehat{w}_h(b)=\Lambda_h^{-1}\sum_{k=1}^K\phi(s_h^k,a_h^k)u\big(\widehat{V}_{h+1}(s_{h+1}^k)-b\big)$\label{line:dynamic_w}
            \State $\widehat{\theta}_h=\Lambda_h^{-1}\sum_{k=1}^K\phi(s_h^k,a_h^k)r_h(s_h^k, a_h^k)$\label{line:dynamic_theta}
            \State $\Gamma_h(\cdot,\cdot)=\beta\sqrt{\phi(\cdot,\cdot)^\top \Lambda_h^{-1}\phi(\cdot,\cdot)}$\label{line: bouns}
            \State $\overline{Q}_h(\cdot,\cdot)=\phi(\cdot,\cdot)^\top \widehat{\theta}_h+\sup_{b\in[0, H-h]}\big\{b+\phi(\cdot,\cdot)^\top \widehat{w}_h(b)\big\}-\Gamma_h(\cdot,\cdot)$\label{line: dynamic-OCE linear key}
            \State $\widehat{Q}_h(\cdot,\cdot)=\clip\big\{\overline{Q}_h(\cdot,\cdot),[0,H-h+1]\big\}$\label{line: bound Q dynamic-OCE}
            \State $\widehat{\pi}_h(\cdot|\cdot)=\argmax_{\pi_h}\big\langle\widehat{Q}_h(\cdot,\cdot),\pi_h(\cdot|\cdot)\big\rangle_\mathcal{A}$\label{line:dynamic_pi}
            \State $\widehat{V}_h(\cdot)=\big\langle\widehat{Q}_h(\cdot,\cdot),\widehat{\pi}_h(\cdot|\cdot)\big\rangle_\mathcal{A}$\label{line:dynamic_v}
    \EndFor
    \State \textbf{Return:} $\widehat{\pi}=\{\widehat{\pi}_h\}_{h=1}^H$.
  \end{algorithmic}
\end{algorithm}

\textbf{Algorithm.} We propose the pessimistic value iteration for dynamic-OCE RL with linear function approximation, termed \textbf{D}ynamic-\textbf{OCE} \textbf{P}essimistic \textbf{L}east-\textbf{S}quares \textbf{V}alue \textbf{I}teration (\texttt{DOCE-PLSVI}), as summarized in Algorithm~\ref{alg:recursiv OCE-PVI}.
Due to the special structure of the Bellman equation for dynamic-OCE RL, we apply linear function approximation separately to the reward and transition-related components of the Bellman update. We first estimate the reward function $r_h$ by ridge regression,
\begin{align}
	\textstyle\min_{\theta\in\mathbb{R}^d} &\textstyle\sum_{k=1}^{K}\big[ r_h(s_h^k,a_h^k)-\phi(s_h^k,a_h^k)^\top \theta \big]^2+\lambda\|\theta\|_2^2 \label{eq: estimate r}.
\end{align}
Let \(\widehat{\theta}_h\) denote the solution to \eqref{eq: estimate r}. We then define the estimated reward function as $\widehat{r}_h(\cdot, \cdot)=\phi(\cdot,\cdot)^\top \widehat{\theta}_h$. In addition, by exploiting the linear structure of the transition model, we have $\mathbb{E}_{s'\sim\mathbb{P}_h(\cdot|s, a)}[u(\widehat{V}_{h+1}(s') - b)] = \int_\mathcal{S} [u(\widehat{V}_{h+1}(s') - b)] \langle \mu_h(s'),\phi(s, a) \big\rangle ~\mathrm{d}s' \allowbreak = \langle w(b),\phi(s, a) \rangle$ where $w(b):=\int_\mathcal{S} [u\big(\widehat{V}_{h+1}(s') - b\big)] \mu_h(s') ~\mathrm{d}s'$. Accordingly, for each \(b\), we estimate \(w_h(b)\) via ridge regression:
\begin{align}
    \textstyle\min_{w(b)\in\mathbb{R}^d} &\textstyle\sum_{k=1}^{K}\big[u\big(\widehat{V}_{h+1}(s_{h+1}^k)-b\big)-\phi(s_h^k,a_h^k)^\top w(b) \big]^2+\lambda\|w(b)\|_2^2.\label{eq: estimate w}
\end{align}
This yields the approximation $\phi(s,a)^\top \widehat{w}_h(b)\approx \mathbb{E}_{s'\sim\mathbb{P}_h(\cdot|s,a)}[u(\widehat{V}_{h+1}(s')-b)]$ without explicitly estimating $\mathbb{P}_h$. Therefore, the transition kernel \(\mathbb P_h\) need not be estimated explicitly, and the resulting procedure is model-free. Accordingly, $\OCE_{s'\sim \mathbb{P}_h(\cdot|s,a)}^u\{\widehat{V}_{h+1}(s')\}$ is estimated by $\sup_{b\in[0,H-h]}\{b+\phi(s,a)^\top \widehat{w}_h(b)\}$, where the budget $b$ is restricted to $[0, H-h]$ by Proposition~\ref{prom: finite b}.

At each step $h$, Line~\ref{line:dynamic_lambda}, Line~\ref{line:dynamic_w}, and Line~\ref{line:dynamic_theta} construct the design matrix \(\Lambda_h\) and compute the least-squares estimates $\widehat{w}_h(b)$ and $\widehat{\theta}_h$. Line~\ref{line: bouns} constructs the elliptical pessimistic bonus $\Gamma_h(s,a)=\beta\sqrt{\phi(s,a)^\top\Lambda_h^{-1}\phi(s,a)}$, with confidence radius $\beta=\mathcal{O}(d[1+(H-h)L_u]\allowbreak\sqrt{\log(2dHKL_u(1+HL_u)\delta^{-1})})$. Line~\ref{line: dynamic-OCE linear key} combines the estimated reward, the estimated OCE continuation term $\sup_{b\in[0,H-h]}\{b+\phi(s,a)^\top\widehat{w}_h(b)\}$, and the bonus to form the pessimistic Q-estimate $\overline{Q}_h$, which is then clipped to the valid range in Line~\ref{line: bound Q dynamic-OCE} to obtain $\widehat{Q}_h$. Finally, Line~\ref{line:dynamic_pi} computes the greedy policy with respect to $\widehat{Q}_h$, and Line~\ref{line:dynamic_v} updates the value estimate $\widehat{V}_h$.

\vspace{5pt}
\noindent\textbf{Theoretical Result.} We establish the following suboptimality bound for Algorithm~\ref{alg:recursiv OCE-PVI}.
\begin{theorem}\label{the: subopt dynamic-OCE}
Under the offline setting of dynamic-OCE RL with linear function approximation, 
    with probability at least $1-\delta$ for $\delta\in (0,1)$,
	the learned policy $\widehat{\pi}$ obtained by Algorithm~\ref{alg:recursiv OCE-PVI} admits the suboptimality bound
    \begin{equation*}
		\SubOptD(\widehat{\pi})\leq 
		\widetilde{\mathcal{O}}\Big(d \textstyle\sum_{h=1}^H[1+(H-h)L_u]\mathbb{E}_{\pi^*}\Big[\sqrt{\phi(s_h,a_h)^\top \Lambda_h^{-1}\phi(s_h,a_h)}\Big|s_1\Big]\Big),
	\end{equation*}     
	where $\Lambda_h=\sum_{k=1}^K\phi(s_h^k,a_h^k)\phi(s_h^k,a_h^k)^\top +\lambda\mathbf{I}$ with $\lambda=1$ and $\beta=cd[1+\allowbreak(H-h)\allowbreak L_u]\allowbreak\cdot\allowbreak\sqrt{\log(2dHKL_u(1+HL_u)\delta^{-1})}$, where $c$ is a constant satisfying $c>0$ and $12\log(64c^2)+46\leq c^2/4$. Here $\widetilde{\mathcal{O}}$ hides logarithmic dependence on $H, d, K$, $L_u$, and $1/\delta$.
\end{theorem}

Theorem~\ref{the: subopt dynamic-OCE} shows that the factor \(1+(H-h)L_u\) captures the stage-dependent statistical difficulty, with the \(L_u\)-dependent component reflecting the effect of the utility function. When $u(t)=t$, substituting $L_u=1$ gives the risk-neutral specialization $\widetilde{\mathcal{O}}(dH\sum_{h=1}^H\mathbb{E}_{\pi^*}[\allowbreak\sqrt{\phi(s_h,a_h)^\top \Lambda_h^{-1}\phi(s_h,a_h)}|s_1])$, which is comparable to the corresponding risk-neutral linear-MDP guarantee in \citet{Jin2020} under matched coverage and regularization assumptions. For the entropic utility $u(t)=\frac{1}{\alpha}(e^{\alpha t}-1)$ with $\alpha<0$, the factor $L_u=e^{|\alpha|H}$ yields an exponential dependence on horizon $H$ induced by the uniform Lipschitz bound. Thus, the theorem recovers the risk-neutral case and extends to general OCE utilities through an explicit utility-dependent factor.

Compared with Algorithm~\ref{alg: recursiv tabular OCE-PVI}, Algorithm~\ref{alg:recursiv OCE-PVI} replaces tabular estimates with ridge-regression estimates under the linear-MDP model and is applicable when this model is appropriate for the state--action representation. When $\phi(s, a) = \mathbf{e}_{(s, a)}$ and $d = |\mathcal{S}| \cdot |\mathcal{A}|$, we have
\(
    \Lambda_h = \diag(\{N_h(s,a) + \lambda\}_{(s,a) \in \mathcal{S} \times \mathcal{A}}),
\)
by the definition of $\Lambda_h$. Consequently, we have
\(
    \mathbb{E}_{\pi^*}[\sqrt{\phi(s_h,a_h)^\top \Lambda_h^{-1}\phi(s_h,a_h)} \,|\, s_1 ]
    =\mathbb{E}_{\pi^*}[(N_h(s_h,a_h) + \lambda)^{-1/2}\mid s_1].
\)
Substituting this expression into our suboptimality bound in Theorem~\ref{the: subopt dynamic-OCE} and following the proof in Appendix~\ref{subsec: Proof of subopt dynamic-OCE meta} yields an upper bound of
\(
    \SubOptD(\widehat{\pi})\leq\widetilde{\mathcal{O}}(|\mathcal{S}||\mathcal{A}|\sum_{h=1}^H[1+(H-h)L_u]\sqrt{\frac{C^*|\mathcal{S}|}{K}}).
\)
Thus the result in Theorem~\ref{the: subopt dynamic-OCE meta} can be compared with the specialization of Theorem~\ref{the: subopt dynamic-OCE} at the level of their coverage structure, but the bound under linear function approximation setting retains an additional dimension dependence. 

\subsection{Static-OCE Pessimistic Least-Squares Value Iteration}\label{subsec: static-OCE RL with Linear Function Approximation}

\begin{algorithm}[!t]
  \caption{\texttt{SOCE-PLSVI}: Static-OCE Pessimistic Least-Squares Value Iteration}\label{alg: static-OCE-PVI}
  \begin{algorithmic}[1]
  \State \textbf{Input:} Offline data $\mathcal{D}=\big\{ \big( s_h^k, a_h^k, r_h(s_h^k, a_h^k) \big) \big\}_{h=1,k=1}^{H,K}$
        \State \textbf{Initialize:} $\widehat{V}_{H+1}(s,b)= u(-b)$ for all $s\in\mathcal{S}$ and $b\in\mathcal{N}_\varsigma^b$
  \For{$h = H, H-1\dots,1$ and all $b\in\mathcal{N}_\varsigma^b$}
            \State $\Lambda_h=\sum_{k=1}^K\phi(s_h^k,a_h^k)\phi(s_h^k,a_h^k)^\top +\lambda\mathbf{I}$\label{line:static_lambda}
            \State $\widehat{w}_h(b)=\Lambda_h^{-1}\sum_{k=1}^K\phi(s_h^k,a_h^k)\widehat{V}_{h+1}(s_{h+1}^k,b-r_h(s_h^k, a_h^k))$\label{line: static-OCE linear distribution}
            \State $\Gamma_h(\cdot,\cdot)=\beta\sqrt{\phi(\cdot,\cdot)^\top \Lambda_h^{-1}\phi(\cdot,\cdot)}$\label{line: static-OCE linear bouns}
            \State $\overline{Q}_h(\cdot,\cdot, b)=\phi(\cdot,\cdot)^\top \widehat{w}_h(b)-\Gamma_h(\cdot,\cdot)$\label{line: static-OCE linear key}
            \State $\widehat{Q}_h(\cdot,\cdot, b)=\clip\big\{\overline{Q}_h(\cdot,\cdot, b),[u(-b),u(H-h+1-b)]\big\}$\label{line: bound Q static-OCE}
            \State $\widehat{\pi}_h(\cdot|\cdot, b)=\argmax_{\pi_h}\big\langle\widehat{Q}_h(\cdot,\cdot, b),\pi_h(\cdot|\cdot,b)\big\rangle_\mathcal{A}$\label{line: static-OCE linear pi}
            \State $\widehat{V}_h(\cdot,b)=\big\langle\widehat{Q}_h(\cdot,\cdot, b),\widehat{\pi}_h(\cdot|\cdot,b)\big\rangle_\mathcal{A}$\label{line: static-OCE linear V}
    \EndFor
            \State $\widehat{b}_1=\argmax_{b\in\mathcal{N}_\varsigma^b}\big\{b+\widehat{V}_1(s_1,b)\big\}$
            \label{line: b}
    \State \textbf{Return:} $\widehat{\pi}=\{\widehat{\pi}_h\}_{h=1}^H$, $\widehat{b}_1$.
  \end{algorithmic}
\end{algorithm}

\textbf{Algorithm.}
For static-OCE RL with linear function approximation, we propose the \textbf{S}tatic-\textbf{OCE} \textbf{P}essimistic \textbf{L}east-\textbf{S}quares \textbf{V}alue \textbf{I}teration (\texttt{SOCE-PLSVI}), as summarized in Algorithm~\ref{alg: static-OCE-PVI}.

In particular, based on the linear structure of the transition model, letting $b'=b-r_h(s, a)$, we have $\mathbb{E}_{s'\sim\mathbb{P}_h(\cdot|s, a)}[\widehat{V}_{h+1}\big(s',b')]=\int_\mathcal{S}[\widehat{V}_{h+1}\big(s',b')]\langle\mu_h(s'), \phi(s, a)\rangle ds'=\allowbreak\langle w_h(b), \allowbreak\phi(s, a)\rangle$, where $w_h(b):=\int_\mathcal{S}[\widehat{V}_{h+1}\big(s',b')]\mu_h(s') ds'$. This identity motivates the following ridge regression estimator $\widehat{w}_h(b)$:
\begin{equation*}
        \textstyle\min_{w(b)\in\mathbb{R}^d}
        \textstyle\sum_{k=1}^K\Big[\widehat{V}_{h+1}\big(s_{h+1}^k,b-r_h(s_h^k,a_h^k)\big) - \phi(s_h^k, a_h^k)^\top w(b) \Big]^2
        +\lambda\|w(b)\|_2^2.
\end{equation*}
Therefore, we have $\phi(s, a)^\top \widehat{w}_h(b)\approx\mathbb{E}_{s'\sim\mathbb{P}_h(\cdot|s, a)}[\widehat{V}_{h+1}\big(s',b-r_h(s, a))]$.
We note that, in contrast to the dynamic-OCE RL algorithm in Algorithm~\ref{alg:recursiv OCE-PVI}, there is no need to treat the transition model $\mathbb{P}_h$ and the reward model $r_h$ separately as in \eqref{eq: estimate r} and \eqref{eq: estimate w}.

Lines~\ref{line:static_lambda} and~\ref{line: static-OCE linear distribution} first construct the design matrix \(\Lambda_h\) and compute the estimate $\widehat{w}_h(b)$. The pessimistic bonus $\Gamma_h(\cdot,\cdot)$ is then constructed in Line~\ref{line: static-OCE linear bouns}, where the confidence radius $\beta=\mathcal{O}(d\,HL_u\sqrt{\log(2dHK\delta^{-1})})$. Lines~\ref{line: static-OCE linear key} forms $\overline{Q}_h$ by subtracting the bonus from $\phi(\cdot,\cdot)^\top\widehat{w}_h(b)$ and then Line~\ref{line: bound Q static-OCE} clips it to obtain the pessimistic Q-function estimate $\widehat{Q}_h(\cdot,\cdot,b)$. The algorithm subsequently computes the greedy policy $\widehat{\pi}_h$ and the corresponding value estimate $\widehat{V}_h$ in Lines~\ref{line: static-OCE linear pi} and~\ref{line: static-OCE linear V}. Line~\ref{line: b} obtains the estimated initial budget $\widehat{b}_1$. During execution, the budget evolves from $\widehat{b}_1$ according to $\widehat{b}_{h+1}=\widehat{b}_{h}-r_h$, thereby encoding relevant reward history through the augmented state. Consequently, this construction calculates OCE only once rather than applying it at every step, at the cost of augmenting the state space with the budget variable.

Algorithm~\ref{alg: static-OCE-PVI} is implemented on the budget grid $\mathcal{N}_\varsigma^b$. As in the tabular static-OCE analysis, we first establish the statistical guarantee for the associated continuous-budget Bellman recursion with $b\in[0,H]$. Lemma~\ref{lemma_subopt_static-OCE_discrete} then transfers this guarantee to the grid implementation and incurs an additional discretization error of $(1+L_u)\varsigma$.

\vspace{5pt}
\noindent\textbf{Theoretical Result.} We establish the following suboptimality bound for Algorithm~\ref{alg: static-OCE-PVI}.
\begin{theorem}\label{the: static-OCE oce result}
	Under the offline setting of static-OCE RL with linear function approximation, for any $\delta\in(0,1)$, with probability at least $1-\delta$ the grid-based policy $\widehat{\pi}$ returned by Algorithm~\ref{alg: static-OCE-PVI} satisfies
	\begin{align*}
		\SubOptS^{\mathtt{dsc}}(\widehat{\pi})
		\leq
		\widetilde{\mathcal{O}}\Big(d \textstyle\sum_{h=1}^H HL_u\mathbb{E}_{\pi^*}\Big[\sqrt{\phi(s_h,a_h)^\top \Lambda_h^{-1}\phi(s_h,a_h)}\Big|s_1,b_1^*\Big]\Big)
		+(1+L_u)\varsigma,
	\end{align*}
	where $b_1^*\in\argmax_{b_1\in[0,H]}\big\{b_1+V_1^*(s_1,b_1)\big\}$, $\Lambda_h=\sum_{k=1}^K\phi(s_h^k,a_h^k)\phi(s_h^k,a_h^k)^\top +\lambda\mathbf{I}_d$ with $\lambda=1$, and $\beta=2cd\,HL_u\sqrt{\log(2dHK\delta^{-1})}$, where $c$ is a constant satisfying $c>0$ and $8\log(64c^2)+34\leq c^2/4$. Here $\widetilde{\mathcal{O}}$ hides logarithmic dependence on $H$, $d$, $K$, and $1/\delta$, and $\SubOptS^{\mathtt{dsc}}$ is defined in Appendix~\ref{subsec_discussion_on_discretization}.
\end{theorem}
The first term in Theorem~\ref{the: static-OCE oce result} is the continuous-budget statistical error established by the Bellman analysis. The additive $(1+L_u)\varsigma$ term is the error incurred by the discrete budget grid in Algorithm~\ref{alg: static-OCE-PVI}, as established in Lemma~\ref{lemma_subopt_static-OCE_discrete}. Thus, the theorem provides a finite-sample guarantee for the implementable grid-based algorithm. The continuous-budget statistical term scales linearly with the feature dimension \(d\), up to logarithmic factors. The bound depends on the utility through \(L_u\) and on the optimal initial budget \(b_1^*\), reflecting the effect of risk sensitivity and the augmented-state formulation. The continuous-budget component of Theorem~\ref{the: static-OCE oce result} is $\widetilde{\mathcal{O}}(d\,HL_u\allowbreak\sum_{h=1}^H\mathbb{E}_{\pi^*}[\allowbreak\sqrt{\phi(s_h,a_h)^\top \Lambda_h^{-1}\phi(s_h,a_h)}\allowbreak|s_1,b_1^*])$. For the entropic utility, this uniform argument yields the valid (though generally non-sharp) scale $HL_u=He^{|\alpha|H}$. Furthermore, by setting $u(t)=t$, we have $L_u=1, HL_u=H$, and the bound simplifies to $\widetilde{\mathcal{O}}(dH)\sum_{h=1}^H\mathbb{E}_{\pi^*}\allowbreak[\sqrt{\phi(s_h,a_h)^\top \Lambda_h^{-1}\phi(s_h,a_h)}|s_1,b_1^*]$, which is comparable to the result for offline risk-neutral RL \citep{Jin2020} under aligned model, coverage, and regularization assumptions.

Moreover, compared with Algorithm~\ref{alg: static-OCE tabular OCE-PVI}, Algorithm~\ref{alg: static-OCE-PVI} replaces the tabular expectation estimators with ridge-regression estimators under the linear-MDP assumption. This formulation accommodates feature-based representations when the linear-MDP assumption is appropriate. When choosing $\phi(s, a) = \mathbf{e}_{(s, a)}$ and $d = |\mathcal{S}| \cdot |\mathcal{A}|$, the matrix $\Lambda_h$ takes the form
\(
    \Lambda_h = \diag(\{N_h(s,a) + \lambda\}_{(s,a) \in \mathcal{S} \times \mathcal{A}}).
\)
Accordingly, we obtain
\(
    \mathbb{E}_{\pi^*}[\sqrt{\phi(s_h,a_h)^\top \Lambda_h^{-1}\phi(s_h,a_h)} \,\allowbreak|\, s_1, b_1^* ]
    =\mathbb{E}_{\pi^*}[(N_h(s_h,a_h) + \lambda)^{-1/2}\mid s_1,b_1^*].
\)
Substituting this identity into the continuous-budget component of Theorem~\ref{the: static-OCE oce result} and following the proof in Appendix~\ref{subsec: proof of static-OCE tabular result} yield
\(
  \SubOptS^{\mathtt{dsc}}(\widehat{\pi})\allowbreak\leq \widetilde{\mathcal{O}}(|\mathcal{S}||\mathcal{A}|\sum_{h=1}^H HL_u\sqrt{\frac{C^*|\mathcal{S}|}{K}})+(1+L_u)\varsigma.
\)
The displayed expression shows the structural specialization of the linear result under canonical features. It retains the explicit factor $d=|\mathcal{S}||\mathcal{A}|$ and the dimension and regularization dependence inherited from the least-squares analysis, whereas Theorem~\ref{the: static-OCE tabular result} is derived directly from tabular concentration. The two results therefore share the utility-dependent and coverage-dependent components but are tailored to different representation assumptions. This specialization is a structural correspondence rather than a claim of rate equivalence.

\section{Theoretical Analysis}\label{sec_theorical_analysis}

This section provides theoretical analysis of the main results in Sections~\ref{sec: Tabular Offline RL with OCE},~\ref{sec: Information-Theoretic Lower Bounds of Tabular Offline RL with OCE}, and~\ref{sec: OCE-RL with Linear Function Approximation}.

\subsection{Analysis of the Tabular Setting}\label{subsec: proof sketch tabular}

In this subsection, we sketch the key steps for proving the suboptimality upper bounds of \texttt{DOCE-PVI} and \texttt{SOCE-PVI}, which are formally stated in Theorems~\ref{the: dynamic-OCE tabular result} and~\ref{the: static-OCE tabular result}, respectively.

\paragraph{Proof Sketch of Theorem~\ref{the: dynamic-OCE tabular result}.}
The proof of Theorem~\ref{the: dynamic-OCE tabular result} begins by decomposing the suboptimality as
\[
\SubOptD(\widehat{\pi}) = \underbrace{V_1^*(s_1) - \widehat{V}_1(s_1)}_{\text{(I)}} + \underbrace{\widehat{V}_1(s_1) - V_1^{\widehat{\pi}}(s_1)}_{\text{(II)}}.
\]
We first bound term (II). Inspired by \citet{Xu2023}, we introduce the OCE-reweighted transition kernel $\mathbb{C}_h(s'|s,a):=\mathbb{P}_h(s'|s,a)B_{h+1}(s')$  where $B_{h+1}(s')\in\partial u(V_{h+1}^{\widehat{\pi}}(s')-b_{h+1})$ and $\mathbb{E}_{s'\sim\mathbb{P}_h(s'|s,a)}[B_{h+1}(s')]=1$ (Definition~\ref{def: dynamic-OCE B}) such that
$\OCE_{s'\sim\mathbb{P}_h(\cdot|s,a)}^u[\widehat{V}_{h+1}(s')] - \OCE_{s'\sim\mathbb{P}_h(\cdot|s,a)}^u \allowbreak [V_{h+1}^{\widehat{\pi}}(s')] \leq \mathbb{E}_{s'\sim\mathbb{C}_h(\cdot|s,a)}\allowbreak[\widehat{V}_{h+1}(s') - V_{h+1}^{\widehat{\pi}}(s')]$.
Letting $\omega_h(s,a)$ be the state-action distribution induced by $\mathbb{C}_h$ and recursively applying the Bellman equation for dynamic OCE, on the high-probability event of Lemma~\ref{def: dynamic-OCE event}, we have 
\begin{align*}
\text{(II)}=\widehat{V}_1(s_1) - V_1^{\widehat{\pi}}(s_1) \leq \sum_{h=1}^H \mathbb{E}_{\omega_h}[ -\iota_h(s_h,a_h) \mid s_1 ] \leq 0,
\end{align*}
where we define the one-step Bellman error for the dynamic-OCE setting as  $\iota_h(s, a) = r_h(s, a) +\allowbreak \OCE_{s'\sim\mathbb{P}_h(\cdot|s, a)}^u[\allowbreak\widehat{V}_{h+1}(s')] - \widehat{Q}_h(s, a)$. Here $\iota_h(s, a) \geq 0$ holds with high probability by its construction (Lemma~\ref{lemma: dynamic-OCE bound iota}). The above result demonstrates the effectiveness of using the pessimistic Q-function estimate \(\widehat{Q}_h\) together with the carefully designed bonus term \(\Gamma_h\) for the dynamic-OCE setting.

To bound term (I), using the Bellman equation for dynamic OCE together with the greedy policy update in \texttt{DOCE-PVI}, term (I) can be bounded as 
$$\text{(I)}=V_1^*(s_1) - \widehat{V}_1(s_1) \leq \sum_{h=1}^H \mathbb{E}_{\pi^*}[ \iota_h(s_h,a_h) \mid s_1 ].$$ 
It remains to bound $\iota_h$ for term (I). This is done by controlling the Bellman estimation error $\mathbb{B}_h \widehat{V}_{h+1} - \allowbreak\widehat{\mathbb{B}}_h \widehat{V}_{h+1}$, which can be further decomposed into the reward estimation error $r_h(s_h,a_h)-\widehat{r}_h(s_h,a_h)$ and the OCE estimation error $\OCE_{\mathbb{P}_h}^u[\widehat{V}_{h+1}(s_{h+1})] - \OCE_{\widehat{\mathbb{P}}_h}^u[\widehat{V}_{h+1}(s_{h+1})]$, i.e., the discrepancy induced by the OCE operator with the true transition
\(\mathbb{P}_h\) versus with its estimate
\(\widehat{\mathbb{P}}_h\). The main technical challenge lies in bounding the OCE estimation error. To this end, let $\widehat{b}_{h+1}^* := \argmax_{b}\{ b + \mathbb{E}_{\mathbb{P}_h}[u(\widehat{V}_{h+1}(s_{h+1})-b)] \}$. Then the OCE estimation error is bounded by $\mathbb{E}_{\mathbb{P}_h}[u(\widehat{V}_{h+1}(s_{h+1})-\widehat{b}_{h+1}^*)] - \mathbb{E}_{\widehat{\mathbb{P}}_h}[u(\widehat{V}_{h+1}(s_{h+1})-\widehat{b}_{h+1}^*)]$ with $u(\widehat{V}_{h+1}(s_{h+1})-\widehat{b}_{h+1}^*) \in [u(-\widehat{b}_{h+1}^*), u(H-h-\widehat{b}_{h+1}^*)] \subset [u(h-H), u(H-h)]$ as $\widehat{b}_{h+1}^*\in[0,H-h]$. 
Then, the reward error and the OCE estimation error are controlled using Hoeffding's inequality for empirical means, yielding a per-step bound of order $[1+u(H-h)-u(h-H)]\sqrt{\log(|\mathcal{S}||\mathcal{A}|HK/\delta)/\max\{1, N_h(s, a)\}}$. By Assumption~\ref{assump: Lipschitz utility}, this factor is at most $1+2(H-h)L_u$.
Applying Cauchy-Schwarz, using the concentrability coefficient $C^*$ to replace the distribution under $\pi^*$ with that under the behavior policy, and further summing over all $h,s,a$ give the final suboptimality bound stated in Theorem~\ref{the: dynamic-OCE tabular result}.

\paragraph{Proof Sketch of Theorem~\ref{the: static-OCE tabular result}.}
The proof first analyzes the idealized static-OCE Bellman recursion with a continuous budget $b\in[0,H]$. It then invokes Lemma~\ref{lemma_subopt_static-OCE_discrete} to transfer the resulting statistical guarantee to the grid-based algorithm. For the continuous-budget analysis, fix an optimal initial budget $b_1^*\in\argmax_{b\in[0,H]}\{b+V_1^*(s_1,b)\}$. Since the learned policy can be evaluated at \(b_1^*\), the continuous-budget suboptimality is upper bounded by, the continuous-budget static-OCE suboptimality is bounded by $V_1^*(s_1,b_1^*)-V_1^{\widehat{\pi}}(s_1,b_1^*)$. We then decompose this quantity in the augmented state space $(s,b)$ as
\[
V_1^*(s_1,b_1^*) - V_1^{\widehat{\pi}}(s_1,b_1^*) = \underbrace{V_1^*(s_1,b_1^*) - \widehat{V}_1(s_1,b_1^*)}_{\text{(I)}} + \underbrace{\widehat{V}_1(s_1,b_1^*) - V_1^{\widehat{\pi}}(s_1,b_1^*)}_{\text{(II)}}.
\]
We first bound term (II). Define the one-step Bellman error in the augmented MDP by $\iota_h(s_h,b_h,a_h):=\mathbb{E}_{s',r}[\widehat{V}_{h+1}(s',b_h-r)]-\widehat{Q}_h(s_h,b_h,a_h)$. Since the Bellman operator is an ordinary expectation in the augmented MDP, on the high-probability event of Lemma~\ref{lemma: static-OCE bound iota}, a standard telescoping argument gives
\[
\text{(II)}=\widehat{V}_1(s_1,b_1^*) - V_1^{\widehat{\pi}}(s_1,b_1^*)\leq\sum_{h=1}^H\mathbb{E}_{\widehat{\pi}}[-\iota_h(s_h,b_h,a_h)\mid s_1,b_1^*]\leq0.
\]
Unlike the dynamic-OCE case, no OCE-reweighted transition kernel is required: the OCE is applied only once to the total return, while the augmented budget variable summarizes the relevant reward history. This reduces the analysis of term (II) to a standard pessimistic value-iteration argument in the augmented MDP.

For term (I), the Bellman equation in the augmented MDP and the greedy policy update imply
\[
\text{(I)}=V_1^*(s_1,b_1^*) - \widehat{V}_1(s_1,b_1^*)\leq\sum_{h=1}^H\mathbb{E}_{\pi^*}[\iota_h(s_h,b_h,a_h)\mid s_1,b_1^*].
\]
Moreover, Lemma~\ref{lemma: static-OCE bound iota} ensures that $0\leq\iota_h\leq2\Gamma_h$ on the same event. Combining the two bounds yields the continuous-budget guarantee $\SubOptS(\widehat{\pi})\leq\sum_{h=1}^H\mathbb{E}_{\pi^*}[\iota_h(s_h,b_h,a_h)\mid s_1,b_1^*]$ for the continuous-budget recursion. It remains to bound $\iota_h$. The relevant estimation error is $\mathbb{U}_h\widehat{V}_{h+1}(s,b,a)-\widehat{\mathbb{U}}_h\widehat{V}_{h+1}(s,b,a)$, where, for any function $f$, the augmented-MDP Bellman operator is defined by
\(
\mathbb{U}_h f(s,b,a)\allowbreak:=\mathbb{E}_{s'\sim\mathbb{P}_h(\cdot\mid s,a),\,r\sim\mathcal{R}_h(\cdot\mid s,a)}\big[f(s',b-r)\big],
\widehat{\mathbb{U}}_h f(s,b,a):=\mathbb{E}_{s'\sim\widehat{\mathbb{P}}_h(\cdot\mid s,a),\,r\sim\widehat{\mathcal{R}}_h(\cdot\mid s,a)}\big[f(s',b-r)\big].
\)
This quantity is a single expectation difference, rather than the nonlinear OCE difference arising in the dynamic case. Since the augmented value function is uniformly bounded by $HL_u$, Lemma~\ref{lemma: static-OCE xu lemma 4} controls the estimation error by a term of order $HL_u\sqrt{\frac{\log(|\mathcal{S}||\mathcal{A}|HK/\delta)}{\max\{1,N_h(s,a)\}}}$. Summing over $h$ and $(s,a)$, applying the Cauchy--Schwarz inequality, and using the concentrability coefficient $C^*$ to transfer the expectation under $\pi^*$ to the behavior-policy distribution yield the continuous-budget statistical term. Finally, Lemma~\ref{lemma_subopt_static-OCE_discrete} gives the grid-based guarantee $\SubOptS^{\mathtt{dsc}}(\widehat{\pi})$ with the additional $(1+L_u)\varsigma$ term stated in Theorem~\ref{the: static-OCE tabular result}. The factor \(HL_u\) reflects the horizon-scaled utility sensitivity arising from the uniform bound on the augmented value function.

\subsection{Analysis of the Lower Bound}\label{subsec: proof sketch lower bounds}
In this subsection, we sketch the key steps for proving the minimax lower bound result, as formally stated in Theorem~\ref{the: lower bound}.

\paragraph{Proof Sketch of Theorem~\ref{the: lower bound}.}
The proof of Theorem~\ref{the: lower bound} begins by constructing a family of hard MDPs $\mathcal{M}^\dag$ that generalizes the regular-state and absorbing-state constructions in prior work~\citep{Jin2020, Xie2021, Xu2023}. The state space consists of $|\mathcal{S}|$ regular states $s_1,\ldots,s_{|\mathcal{S}|}$ and two absorbing states, $s^\sharp$ and $s^\diamond$. For each instance $\mathcal{M}_{a^*}$ in the family $\mathcal{M}^\dag$, an optimal action $a_{h,i}^*$ is assigned to every step--state pair $(h,s_i)$. Choosing $a_{h,i}^*$ increases the probability of transitioning to $s^\sharp$ by $\tau$. The offline dataset is generated by a behavior policy that is uniform over $C=\min\{\lfloor C^*\rfloor,|\mathcal{A}|\}$ actions, thereby controlling the concentrability coefficient through $C$. We further introduce a null MDP $\mathcal M_0$ that shares the same state space and behavior policy but has action-independent transition dynamics. The full construction is given in Appendix~\ref{sec: lower bounds}.

Under this construction, Lemma~\ref{lemma: same lower bound} shows that the dynamic-OCE and static-OCE formulations induce the same suboptimality expression on the hard instance. The reason is that, since, under the absorbing-state construction with rewards occurring only after step $\overline H$, the nontrivial OCE evaluation effectively occurs only once. Hence, the same information-theoretic argument yields a minimax lower bound for both formulations. For any learned policy, the hard instance $\mathcal{M}_{a^*}$ yields $\SubOpt(\mathcal{M}_{a^*},\mathrm{Alg}(\mathcal{D}))\geq\sum_{h=1}^H\sum_{i=1}^{|\mathcal{S}|}d_h^{\widehat{\pi}}(s_i)\tau[u(H-\overline{H}-b_1^*)-u(-b_1^*)]\mathbb{I}\{\widehat{a}_{h,i}\neq a_{h,i}^*\}$, where $d_h^{\widehat{\pi}}(s_i)=|\mathcal{S}|^{-1}(1-2p)^{h-1}$. Averaging over the uniformly selected optimal actions reduces the problem to controlling the expected visitation count $\mathbb{E}[N_h(s_i,a_{h,i}^*)]$. A change-of-measure argument between $\mathcal M_{a^*}$ and $\mathcal M_0$, followed by Pinsker's inequality, gives $(\mathbb{E}_{\mathcal{M}_{a^*}}[N_h(s_i,a_{h,i}^*)]-\mathbb{E}_{\mathcal{M}_0}[N_h(s_i,a_{h,i}^*)])/K\leq\sqrt{\frac{p}{2}\log\!\left(\frac{p^2}{p^2-\tau^2}\right)\mathbb{E}_{\mathcal{M}_0}[N_h(s_i,a_{h,i}^*)]}$. Summing over $(h,i)$ and applying Cauchy--Schwarz inequality then controls the average action-identification error. Finally, since $|\mathcal{A}|\geq C^*$, we choose $p=1/(2H)$, $\overline{H}=(1-\rho)H$, $C=\lfloor C^*\rfloor$, and $\tau=\sqrt{C|\mathcal{S}|/(16HK)}$. These choices yield $\tau<p$ and give the lower bound $\Omega([u(\rho H-b_1^*)-u(-b_1^*)]\sqrt{C^*|\mathcal{S}|H/K})$ in Theorem~\ref{the: lower bound}. Details appear in Appendix~\ref{subsec: proof of static-OCE tabular lower bound}.

\subsection{Analysis of the Function Approximation Setting}\label{subsec: proof sketch linear}

In this subsection, we sketch the key steps for proving the suboptimality upper bounds of \texttt{DOCE-PLSVI} and \texttt{SOCE-PLSVI}, as stated in Theorems~\ref{the: subopt dynamic-OCE} and~\ref{the: static-OCE oce result}, respectively.

\paragraph{Proof Sketch of Theorem~\ref{the: subopt dynamic-OCE}.}
The proof of Theorem~\ref{the: subopt dynamic-OCE} begins with the same decomposition as in the tabular setting:
\[
\SubOptD(\widehat{\pi}) = \underbrace{V_1^*(s_1) - \widehat{V}_1(s_1)}_{\text{(I)}} + \underbrace{\widehat{V}_1(s_1) - V_1^{\widehat{\pi}}(s_1)}_{\text{(II)}}.
\]
Since Lemma~\ref{lemma: dynamic-OCE iota} applies to both the tabular and linear-function-approximation settings, term (II) is non-positive and term (I) is bounded by $\sum_{h=1}^H\mathbb{E}_{\pi^*}[\iota_h(s_h,a_h)\mid s_1]$. It therefore remains to bound $\iota_h(s_h,a_h)$, where the proof diverges from the tabular setting. In the tabular case, the Bellman estimation error is decomposed into the reward error and the OCE error, so that the resulting terms can be bounded using visitation counts $N_h(s,a)$ and the concentrability coefficient $C^*$. Under linear function approximation, the reward and the utility-transformed continuation expectation are parametrized by \(\theta_h\) and \(w_h(b)\), respectively. The true Bellman update thus takes the form $\phi(s,a)^\top \theta_h + \sup_{b}\{b + \phi(s,a)^\top w_h(b)\}$, and the estimates $\widehat{\theta}_h$ and $\widehat{w}_h(b)$ are obtained by two separate ridge regressions. The Bellman estimation error is therefore related to the estimation errors of these parameters.
Let $b_h\in\argmax_{b\in[0,H-h]}\{b+\phi(s,a)^\top\widehat{w}_h(b)\}$. The Bellman error is bounded by a regularization term and a self-normalized stochastic term:
\begin{align*}
  &\Big|\mathbb{B}_h\widehat{V}_{h+1}(s,a)-\widehat{\mathbb{B}}_h\widehat{V}_{h+1}(s,a)\Big|\\
  \leq&
  \underbrace{\lambda\Big|\phi(s,a)^\top \Lambda_h^{-1}(\theta_h+w_h(b_h))\Big|}_{\text{(III)}}
  +
  \underbrace{\Big\|\sum_{k=1}^K\phi(s_h^k,a_h^k)\epsilon_h^k(\widehat{V}_{h+1})\Big\|_{\Lambda_h^{-1}}\cdot\sqrt{\phi(s,a)^\top \Lambda_h^{-1}\phi(s,a)}}_{\text{(IV)}},
\end{align*}
where, for the fixed optimizer \(b_h\) defined above,
\(
  \epsilon_h^k(f):=r_h^k+u(f(s_{h+1}^k)-b_h)-r_h(s_h^k,a_h^k)-\mathbb{E}_{s'\sim\mathbb{P}_h(\cdot|s_h^k,a_h^k)}[u(f(s')-b_h)],
\)
and $f:\mathcal{S}\to[0,H-h+1]$ is an arbitrary function.

Term (III) is bounded by noting that the combined vector $\theta_h + w_h(b_h)$ has norm at most $[1+(H-h)L_u]\sqrt{d}$, which yields $\text{(III)} \leq \lambda\cdot[1+(H-h)L_u]\sqrt{d}\cdot\sqrt{\phi(s,a)^\top \Lambda_h^{-1}\phi(s,a)}$. In contrast to the tabular argument, the bound is expressed through the feature covariance matrix rather than through $N_h(s,a)$ or a sum over $(s,a)$. Self-normalized concentration with the matrix $\Lambda_h=\sum_{k=1}^K\phi(s_h^k,a_h^k)\phi(s_h^k,a_h^k)^\top\allowbreak+\lambda\mathbf{I}$ then yields a per-step bound proportional to $[1+(H-h)L_u]\sqrt{\phi^\top \Lambda_h^{-1}\phi}$. To bound the second term (IV), the proof relies on an $\varepsilon$-covering number analysis for the value function class under linear approximation. This class (Definition~\ref{def: dynamic-OCE function class}) differs from its risk-neutral counterpart~\citep{Jin2020} in its parametrization by $\theta$, $w(b)$, $\beta$, $\Lambda$, and for this reason the proof additionally requires a dedicated $\varepsilon$-covering number bound for it (Lemma~\ref{lemma: dynamic-OCE e covering number}). With the bonus $\Gamma_h=\beta\sqrt{\phi^\top \Lambda_h^{-1}\phi}$, where $\beta$ is chosen as $\mathcal{O}(d[1+(H-h)L_u]\sqrt{\log(2dHK(1+HL_u)L_u\delta^{-1})})$, summing over $h$ and taking expectation under $\pi^*$ yields the final suboptimality bound stated in Theorem~\ref{the: subopt dynamic-OCE}. The detailed proof is provided in Appendix~\ref{subsec: proof of subopt dynamic-OCE}.

\paragraph{Proof Sketch of Theorem~\ref{the: static-OCE oce result}.}
The proof of Theorem~\ref{the: static-OCE oce result} follows the same continuous-budget-then-discretization structure as in the tabular static-OCE case. The analysis is presented in the more general stochastic-reward setting, which includes the deterministic-reward case as a special case. In Appendix~\ref{subsec: pf of static-OCE oce result} we first establish the continuous-budget statistical term for the stochastic-reward extension by treating the next state and reward jointly and using the linear representation $\mathbb{U}_hV_{h+1}(s,b,a)=\phi(s,a)^\top w_h(b)$. Lemma~\ref{lemma_subopt_static-OCE_discrete} then transfers the resulting continuous-budget guarantee to the grid-based implementation. Fixing $b_1^*\in\argmax_{b\in[0,H]}\{b+V_1^*(s_1,b)\}$, wwe decompose the continuous-budget value gap at $b_1^*$ in the augmented state space as
\[
V_1^*(s_1,b_1^*) - V_1^{\widehat{\pi}}(s_1,b_1^*) = \underbrace{V_1^*(s_1,b_1^*) - \widehat{V}_1(s_1,b_1^*)}_{\text{(I)}} + \underbrace{\widehat{V}_1(s_1,b_1^*) - V_1^{\widehat{\pi}}(s_1,b_1^*)}_{\text{(II)}}.
\]
By Lemma~\ref{lemma: static-OCE iota}, term (II) is non-positive, whereas term (I) is at most $\sum_{h=1}^H\mathbb{E}_{\pi^*}[\iota_h(s_h,b_h,a_h)\mid s_1,b_1^*]$, where $\iota_h(s,b,a)=\mathbb{E}_{s',r}[\widehat{V}_{h+1}(s',b-r)]-\widehat{Q}_h(s,b,a)$. Thus, it remains to control this Bellman error.

It therefore remains to bound $\iota_h$, where the analysis departs from the tabular setting. In the tabular case, the Bellman estimation error is controlled directly through empirical visitation counts \(N_h(s,a)\) and the concentrability coefficient \(C^*\). Under linear function approximation, the same conditional expectation is represented through a feature-dependent parameter \(w_h(b)\), so the analysis instead reduces to controlling its ridge-regression estimation error. Under static-OCE with linear function approximation, the Bellman estimation error here is $\mathbb{U}_h\widehat{V}_{h+1}(s,b, a) - \widehat{\mathbb{U}}_h\widehat{V}_{h+1}(s,b, a)$, rather than a sum of reward and OCE estimation errors as in the dynamic-OCE case. Thus, unlike the dynamic-OCE analysis, no auxiliary OCE optimizer or OCE linearization argument is required. The analysis extends to stochastic rewards via the identity $\mathbb{E}_{s',r}[V_{h+1}(s',b-r)] = \phi^\top w(b)$ under the joint linear transition-reward model used in the stochastic-reward extension, and the estimator $\widehat{w}_h(b)$ is obtained by a single ridge regression, in contrast to the two ridge regressions for $\widehat{\theta}_h$ and $\widehat{w}_h(b)$ in the dynamic-OCE case. The Bellman estimation error is therefore related to the estimation error of this parameter. After some algebra, the absolute value of the Bellman error can be decomposed and bounded as
\begin{align*}
  &\Big|\mathbb{U}_h\widehat{V}_{h+1}(s,b,a) - \widehat{\mathbb{U}}_h\widehat{V}_{h+1}(s,b,a)\Big|
  =
  \Big|\phi(s,a)^\top w_h(b) - \phi(s,a)^\top \widehat{w}_h(b)\Big|\\
  \leq&
  \underbrace{\lambda\Big|\phi(s,a)^\top \Lambda_h^{-1} w_h(b)\Big|}_{\text{(III)}}
  +
  \underbrace{\Big\|\sum_{k=1}^K\phi(s_h^k,a_h^k)\epsilon_h^k(\widehat{V}_{h+1})\Big\|_{\Lambda_h^{-1}}\cdot\sqrt{\phi(s,a)^\top \Lambda_h^{-1}\phi(s,a)}}_{\text{(IV)}},
\end{align*}
where $\epsilon_h^k(V):=V(s_{h+1}^k,b-r_h^k)-\mathbb{E}_{s',r}[V(s',b-r)\mid s_h^k,a_h^k]$, and $V:\mathcal{S}\times[0,H]\to[-HL_u,HL_u]$ is arbitrary. This range contains the clipped static-OCE value functions under Assumption~\ref{assump: Lipschitz utility}. For term (III), we use $\|w_h(b)\|\leq HL_u\sqrt{d}$ in Appendix~\ref{subsec: pf of static-OCE oce result}, which gives a bound of order $HL_u\sqrt{\phi(s,a)^\top\Lambda_h^{-1}\phi(s,a)}$. Term (IV) is controlled uniformly over the static-OCE value-function class using self-normalized concentration and an $\varepsilon$-covering argument. The augmented budget variable enlarges the value-function class and enters the covering-number analysis in Lemma~\ref{lemma: static-OCE e covering number}. With $\Gamma_h=\beta\sqrt{\phi^\top\Lambda_h^{-1}\phi}$ and $\beta=\mathcal{O}(d\,HL_u\sqrt{\log(2dHK\delta^{-1})})$, terms (III) and (IV) are dominated by the bonus term on the resulting high-probability event. Summing over $h$ and taking expectation along trajectories induced by $\pi^*$ yields the continuous-budget statistical term. Lemma~\ref{lemma_subopt_static-OCE_discrete} then supplies the additional discretization term $(1+L_u)\varsigma$ for $\SubOptS^{\mathtt{dsc}}(\widehat{\pi})$. This yields the grid-based bound in Theorem~\ref{the: static-OCE oce result}.

\section{Conclusion}\label{sec_conclusion}

Since the majority of existing research on risk-sensitive RL focuses on online settings or specific risk measures, we study the offline risk-sensitive RL with OCE as a general risk measure. We develop provably efficient offline RL algorithms for both dynamic-OCE and static-OCE formulations and establish corresponding suboptimality upper bounds. Additionally, we establish the first minimax suboptimality lower bound for offline RL under general OCE risk measures, applicable to both formulations. Finally, we extend our framework to linear function approximation and develop provably efficient algorithms for both dynamic-OCE and static-OCE RL, together with corresponding suboptimality upper bounds.

\appendix
\newpage
{\centering
{\Large  \textbf{Appendix}}
\par }
\vspace{0.6cm}
{
\hypersetup{linkcolor=black}
\tableofcontents
}


\addtocontents{toc}{\protect\setcounter{tocdepth}{2}}




\newpage
\section{Discussions of OCE RL}

\subsection{An Auxiliary Property of OCE}\label{sec: properties of OCE}
In the proof, we use an auxiliary property of OCE: for any utility function $u$ satisfying Definition~\ref{def_definition_of_OCE} and any positive bounded random variable $X$,
it is sufficient to optimize over a finite set of $b$ values rather than the entire space of $b$. 
\begin{proposition}\label{prom: finite b}
    For any bounded positive random variable $X\in[0,M]$ for some $M>0$, we have
    \begin{equation*}
        \OCE^u(X) = \sup_{b\in\mathbb{R}}\big\{b+\mathbb{E}\big[u(X-b)\big]\big\} = \sup_{b\in[0,M]}\big\{b+\mathbb{E}\big[u(X-b)\big]\big\}.
    \end{equation*}
\end{proposition}
\begin{proof}
    Define the function $F(b)$ by
    \(
        F(b) = b + \mathbb{E}\big[u(X-b)\big]
    \).
    For $b>M$, let $\Delta=b-M>0$. Since $X-M\leq0$, concavity and $1\in\partial u(0)$ imply that every supergradient of $u$ at a negative argument is at least $1$. Hence,
    \[
        u(X-b)=u(X-M-\Delta)\leq u(X-M)-\Delta.
    \]
    Consequently,
    \[
        F(b)\leq b+\mathbb{E}[u(X-M)]-\Delta
        =M+\mathbb{E}[u(X-M)]=F(M).
    \]
    Similarly, for $b<0$, let $\Delta=-b>0$. Since $X\geq0$, every supergradient of $u$ at a nonnegative argument is at most $1$, and thus
    \[
        u(X-b)=u(X+\Delta)\leq u(X)+\Delta.
    \]
    Therefore,
    \[
        F(b)\leq-\Delta+\mathbb{E}[u(X)+\Delta]=F(0).
    \]
    It follows that
    \begin{equation*}
        \OCE^u(X) = \sup_{b \in \mathbb{R}} \big\{ F(b) \big\}
        = \sup_{b \in [0, M]} \big\{ b + \mathbb{E}\big[u(X - b)\big] \big\}.
    \end{equation*}
    This completes the proof.
\end{proof}

\subsection{Discussion of Static-OCE and augmented MDP}\label{sbusec: Comparison between Recursive OCE-RL and static-OCE RL}
For the static-OCE setting, our objective is to maximize $\OCE\{\sum_{i=1}^Hr_i\}$. With the definition of OCE in \eqref{eq: definition of OCE}, we have
\begin{align*}
    \OCE^*\Big\{\sum_{i=1}^Hr_i\Big\}
    =
    \max_{\pi\in \Pi_2}\max_{b\in[0,H]}\Big\{b+\mathbb{E}\Big[u\Big(\sum_{i=1}^Hr_i-b\Big)\Big]\Big\}
    =
    \max_{b\in[0,H]}\Big\{b+\max_{\pi\in \Pi_2}\mathbb{E}\Big[u\Big(\sum_{i=1}^Hr_i-b\Big)\Big]\Big\}.
\end{align*}
Several methods have been proposed to optimize over $b$. It therefore remains to solve $\max_{\pi\in \Pi_2}\mathbb{E}[u(\sum_{i=1}^Hr_i-b)]$. Following the standard RL formulation, define $V_1^\pi(s_1,b)=\mathbb{E}[u(\sum_{i=1}^Hr_i-b)]$. Dynamic programming then yields $V_1^*(s_1,b)$. Following the augmented-MDP construction of \citet{Baeuerle2011, Baeuerle2021}, we obtain
\begin{align*}
    \OCE^*\Big\{\sum_{i=1}^Hr_i\Big\}
    =
    \max_{b\in[0,H]}\Big\{b+V_1^{\pi^*}(s_1,b)\Big\}
    =
    \max_{b\in[0,H]}\Big\{b+V_1^*(s_1,b)\Big\}
\end{align*}
In the augmented MDP, $\pi(\cdot\mid s,b)$ is Markovian because $b$ is included in the state. In the original MDP, however, it is history-dependent. The following CVaR example illustrates that a history-dependent policy can strictly outperform every history-independent policy:
\begin{figure}[H]
\centering
\includegraphics[width=0.6\textwidth]{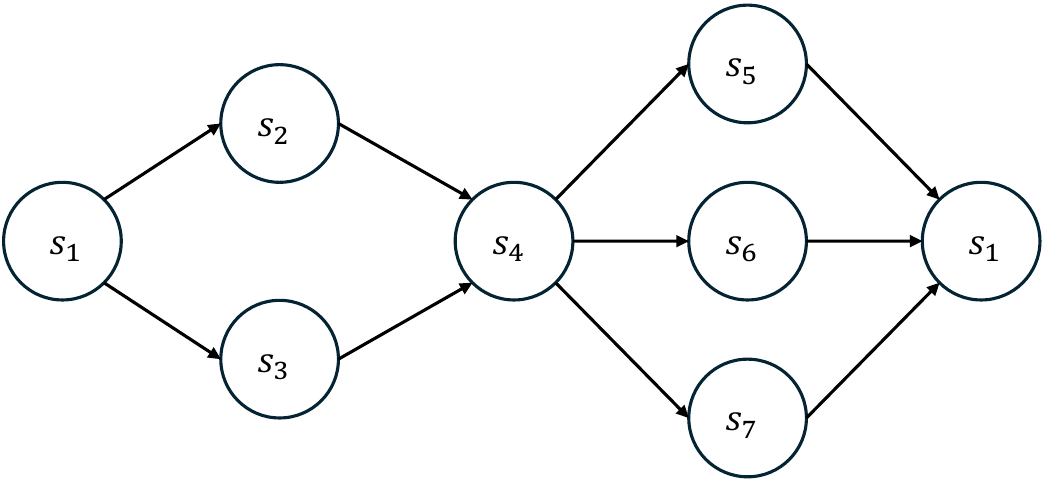}
\caption{An MDP with the history-dependent optimal CVaR policy.}
\label{fig:Exp_HD_MDP}
\end{figure}
This MDP has seven states $(s_1,s_2,\dots,s_7)$ in the state space $\mathcal{S}$ and two actions $(a_1,a_2)$ in the action space $\mathcal{A}$, with
\begin{align*}
    &\mathbb{P}(s_1|s_5,a)=\mathbb{P}(s_1|s_6,a)=\mathbb{P}(s_1|s_7,a)=1, \quad \forall a\in\mathcal{A}\\
    &\mathbb{P}(s_2|s_1,a)=\mathbb{P}(s_3|s_1,a)=0.5, \quad \forall a\in\mathcal{A}\\
    &\mathbb{P}(s_4|s_2,a)=\mathbb{P}(s_4|s_3,a)=1, \quad \forall a\in\mathcal{A}\\
    &\mathbb{P}(s_5|s_4,a_1)=0.75, \quad \mathbb{P}(s_7|s_4,a_1)=0.25, \quad \mathbb{P}(s_6|s_4,a_2)=1\\
    &r(s_1,a)=r(s_3,a)=r(s_4,a)=r(s_7,a)=0, 
    r(s_2,a)=1, r(s_5,a)=1.5, r(s_6,a)=0.5, \quad \forall a\in\mathcal{A}.\\
\end{align*}
Set $H=4$. Only the action at step $h=3$ affects $\mathrm{CVaR}(\sum_{h=1}^H r_h)$. A direct calculation gives $\mathrm{CVaR}(\sum_{h=1}^H r_h)=0.5$ for every history-independent policy $\pi_3(\cdot\mid s_{h=3})$. Consider instead the following history-dependent policy $\pi_3(\cdot\mid s_{h=2},s_{h=3})$:
\begin{align*}
    &\pi_3(a_1|s_{h=2}=s_3,s_{h=3}=s_4)=1, \quad
    \pi_3(a_2|s_{h=2}=s_2,s_{h=3}=s_4)=1.\\
\end{align*}
Under this history-dependent policy, the cumulative reward over the four steps is
\begin{align*}
    \sum_{h=1}^{H}r_h=\begin{cases}
        0,& \text{~w.p.~} \frac{1}{8}\\
        1.5& \text{~w.p.~} \frac{7}{8}.\\
    \end{cases}
\end{align*}
Consequently, $\mathrm{CVaR}(\sum_{h=1}^H r_h)=0.75$. This MDP therefore admits a history-dependent policy that strictly outperforms every history-independent policy, showing that an optimal static-OCE policy can be history-dependent.

\subsection{Comparison between Dynamic-OCE RL and Static-OCE RL}
The dynamic-OCE formulation corresponds to the risk-sensitive RL objective commonly referred to as ``dynamic risk'' (also known as iterated risk). In this setting, for $a_h \sim \pi_h(\cdot|s_h)$, $\pi=\{\pi_h\}_{h=1}^H$, the optimization objective is
\begin{align*}
    J_D(\pi)
    &=
    \OCE^u_{s_1\sim\rho,a_1\sim\pi_1(\cdot|s_1)}\Bigg\{r_1(s_1,a_1)+\OCE_{s_2\sim\mathbb{P}_1(\cdot|s_1,a_1),a_2\sim\pi_2(\cdot|s_2)}^u\bigg\{ r_2(s_2,a_2)\\
    +&
    \OCE_{s_3\sim\mathbb{P}_2(\cdot|s_2,a_2),a_3\sim\pi_3(\cdot|s_3)}^u\Big\{ r_3(s_3,a_3)\\
    +&
    \big\{ \cdots \OCE_{s_H\sim \mathbb{P}_{H-1}(\cdot|s_{H-1},a_{H-1}),a_H\sim\pi_H(\cdot|s_H)}^u\{r_H(s_H,a_H)\} \big\} \Big\} \bigg\}\Bigg\}.
\end{align*}
A key advantage of the dynamic (iterated) risk is the existence of Bellman equations and optimal Markovian policies, which allow direct adaptations of standard RL algorithms.  
To illustrate, consider the widely-used mean-variance risk measure. By choosing
\(
u(t) = (t-ct^2)\,\mathbb{I}\{t\leq \tfrac{1}{2c}\}+\tfrac{1}{4c}\,\mathbb{I}\{t>\tfrac{1}{2c}\},
\)
the dynamic-OCE objective reduces to
\[
    J_D(\pi) = \sum_{i=1}^{H}\mathbb{E}[r_i] - \sum_{i=1}^{H}\Var[r_i].
\]

In contrast, under the static-OCE setting, for $a_h \sim \pi_h(\cdot|s_h)$, $\pi=\{\pi_h\}_{h=1}^H$ is a history-independent policy, the optimization objective is
\begin{align*}
    J_S(\pi)
    =&
    \OCE^u\Big\{\sum_{h=1}^H r_h\Big\}.
\end{align*}
When reduced to the mean-variance measure, the objective becomes
\[
    J_S(\pi) = \mathbb{E}\Big[\sum_{i=1}^{H}r_i\Big] - \Var\Big[\sum_{i=1}^{H}r_i\Big].
\]
In practice, decision-making often involves balancing the expected cumulative reward against its overall variance. Due to the properties of variance, the dynamic-OCE formulation effectively behaves as a step-wise greedy strategy: it separately accounts for the variance of each stage reward. 
Therefore, unlike the dynamic-OCE formulation, the static-OCE formulation implicitly accounts for covariance terms across different time steps through the variance of the cumulative return.
This makes dynamic-OCE potentially less stable, being overly aggressive in some cases and overly conservative in others, compared to the static-OCE formulation.  
Moreover, when OCE reduces to CVaR, the dynamic-OCE formulation becomes particularly difficult to interpret, as CVaR lacks favorable linearity properties. This further highlights that dynamic-OCE risk, while algorithmically convenient, is generally less interpretable than its static-OCE counterpart.

\subsection{Discussion on the discretization of $b$}\label{subsec_discussion_on_discretization}

In the static-OCE setting, the optimal policy can be history-dependent. To obtain a computationally tractable implementation, we discretize the budget variable $b$. We define the suboptimality gap for the discrete-budget setting by
\begin{align*}
    \SubOptS^{\varsigma}(\pi)
    :=
    \textstyle\sup_{b_1\in[0,H]}\big\{b_1+V_1^{\pi^*}(s_1, b_1)\big\}-\textstyle\sup_{b_1\in\mathcal{N}_\varsigma^b}\big\{b_1+V_1^{\pi}(s_1, b_1)\big\}.
\end{align*}
To characterize the effect of discretization on the optimal initial budget, we impose the following stability assumption.
\begin{assumption}
\label{assump: budget stability}
    For any policy $\pi$ and any sufficiently small $\varsigma>0$, define
    \[
        G^\pi(b):=b+V_1^\pi(s_1,b),\qquad
        b^*\in\argmax_{b\in[0,H]}G^\pi(b),\qquad
        b_\varsigma^*\in\argmax_{b\in\mathcal{N}_\varsigma^b}G^\pi(b).
    \]
    There exist selections of $b^*$ and $b_\varsigma^*$ such that
    \[
        |b^*-b_\varsigma^*|\leq\varsigma.
    \]
\end{assumption}
The following lemma quantifies the resulting discretization error.
\begin{lemma}\label{lemma_subopt_static-OCE_discrete}
    In the static-OCE setting, for any $\varsigma\in(0,H)$, if the budget variable $b_h$ is discretized using a $\varsigma$-net over $[0,H]$, defined as $\mathcal{N}_\varsigma^b:=\{n\varsigma:n=0,\ldots,\lfloor H/\varsigma \rfloor\}$, then the suboptimality gap induced by discretization is bounded by
    \begin{align*}
        \big|\SubOptS^{\varsigma}(\widehat{\pi})-\SubOptS(\widehat{\pi})\big| \leq (1+L_u)\varsigma,
    \end{align*}
    where $\SubOptS$ is the suboptimality gap without discretization.
\end{lemma}
\begin{proof}
    Decompose the discretized suboptimality gap $\SubOptS^{\varsigma}$ into the continuous-budget suboptimality gap $\SubOptS$ and the additional discretization error:
\begin{align*}
    &\SubOptS^{\varsigma}(\widehat{\pi})\\
    =& 
    \textstyle\sup_{b_1\in[0,H]}\big\{b_1+V_1^{\pi^*}(s_1, b_1)\big\}
    -\textstyle\sup_{b_1\in[0,H]}\big\{b_1+V_1^{\pi}(s_1, b_1)\big\}\\
    &+\textstyle\sup_{b_1\in[0,H]}\big\{b_1+V_1^{\pi}(s_1, b_1)\big\}
    -\textstyle\sup_{b_1\in\mathcal{N}_\varsigma^b}\big\{b_1+V_1^{\pi}(s_1, b_1)\big\}\\
    =&
    \SubOptS(\widehat{\pi})
    +\textstyle\sup_{b_1\in[0,H]}\big\{b_1+V_1^{\pi}(s_1, b_1)\big\}
    -\textstyle\sup_{b_1\in\mathcal{N}_\varsigma^b}\big\{b_1+V_1^{\pi}(s_1, b_1)\big\}.
\end{align*}
Therefore,
\begin{align*}
    &\big|\SubOptS^{\varsigma}(\widehat{\pi})-\SubOptS(\widehat{\pi})\big|\\
    =&
    \Big|\textstyle\sup_{b_1\in[0,H]}\big\{b_1+V_1^{\pi}(s_1, b_1)\big\}
    -\textstyle\sup_{b_1\in\mathcal{N}_\varsigma^b}\big\{b_1+V_1^{\pi}(s_1, b_1)\big\}\Big|\\
    =&
    \Big|\big\{b^*+V_1^{\pi}(s_1, b^*)\big\}
    -\big\{b_\varsigma^*+V_1^{\pi}(s_1, b_\varsigma^*)\big\}\Big|\\
    \leq&
    |b^*-b_\varsigma^*|
    +\bigg|\mathbb{E}_{\widehat{\pi},\mathbb{P}}\bigg[u\Big(\sum_{h\in[H]}r_h-b^*\Big)\bigg]-\mathbb{E}_{\widehat{\pi},\mathbb{P}}\bigg[u\Big(\sum_{h\in[H]}r_h-b_\varsigma^*\Big)\bigg]\bigg|\\
    \leq&
    \varsigma
    +L_u\big|b_\varsigma^*-b^*\big|
    \leq 
    (1+L_u)\varsigma.\\
\end{align*}
Here, $b^*$ and $b_\varsigma^*$ denote maximizers of $b_1+V_1^\pi(s_1,b_1)$ over $[0,H]$ and $\mathcal{N}_\varsigma^b$, respectively. Assumption~\ref{assump: budget stability} gives $|b_\varsigma^*-b^*|\leq\varsigma$. Since the total reward and both budgets belong to $[0,H]$, the two utility arguments lie in $[-H,H]$. Assumption~\ref{assump: Lipschitz utility} therefore bounds the utility difference by $L_u|b_\varsigma^*-b^*|\leq L_u\varsigma$.
\end{proof}
Thus, $\SubOptS(\widehat{\pi})-(1+L_u)\varsigma\leq\SubOptS^{\varsigma}(\widehat{\pi})\leq \SubOptS(\widehat{\pi})+(1+L_u)\varsigma$. The additional suboptimality caused by discretization is therefore at most $(1+L_u)\varsigma$ and vanishes as the grid is refined.

\section{Proofs for Section \ref{sec: Tabular Offline RL with OCE}} \label{sec: proof of tabular offline RL with OCE}
\subsection{Lemmas for Theorem \ref{the: subopt dynamic-OCE meta}} \label{subsec: Lemmas for subopt dynamic-OCE meta}
The suboptimality is governed by the model-evaluation error. Define $\iota_h$ as the error in the estimated Bellman equation at step $h$:
\begin{equation}\label{eq: def of iota}
    \iota_h(s,a)
    :=
    r_h(s,a)
    +\OCE_{s'\sim\mathbb{P}_h(\cdot\mid s,a)}^u
    \big\{\widehat{V}_{h+1}(s')\big\}
    -\widehat{Q}_h(s,a).
\end{equation}
Under the dynamic-OCE RL formulation, define the Bellman operator with OCE by
\begin{equation*}
    \begin{split}
        &\mathbb{B}_hf(s,a)=r_h(s,a)+\OCE_{s'\sim\mathbb{P}_h(\cdot|s,a)}^u\big\{f(s')\big\}\\
        &\widehat{\mathbb{B}}_hf(s,a)=\widehat{r}_h(s,a)+\OCE_{s'\sim\widehat{\mathbb{P}}_h(\cdot|s,a)}^u\big\{f(s')\big\}.
    \end{split}
\end{equation*}
Define the event $\mathcal{E}_h$ as follows.
\begin{definition}\label{def: dynamic-OCE event}
    In the dynamic-OCE setting, let
    \begin{equation*}
        \mathcal{E}_h=\Big\{\Big|\mathbb{B}_{h}\widehat{V}_{h+1}(s,a)-\widehat{\mathbb{B}}_{h}\widehat{V}_{h+1}(s,a)\Big|\leq\Gamma_h(s,a)\Big\},
    \end{equation*}
    where $\{\Gamma_h\}_{h=1}^H$ denotes the bonus sequence and $\mathbb{P}(\bigcap_{h=1}^H\mathcal{E}_h)\geq1-\delta$.
\end{definition}
On the event $\mathcal{E}_h$, the suboptimality bound is controlled by the Bellman estimation error.
\begin{lemma}\label{lemma: dynamic-OCE ineq of u}
    Under the dynamic-OCE setting, we have
    \begin{equation*}
        \mathbb{B}_{h}V_{h+1}^{\pi^*}(s_h, a_h)-\mathbb{B}_h\widehat{V}_{h+1}(s_h, a_h)
        \leq \mathbb{E}_{s_{h+1}\sim\mathbb{P}_h(\cdot|s_h,a_h)}\Big[V_{h+1}^{\pi^*}(s_{h+1})-\widehat{V}_{h+1}(s_{h+1})\Big].
    \end{equation*}
\end{lemma}

\begin{proof}
    The left-hand side equals
    \begin{equation*}
        \begin{split}
            &\mathbb{B}_{h}V_{h+1}^{\pi^*}(s_h, a_h)-\mathbb{B}_h\widehat{V}_{h+1}(s_h, a_h)\\
            =&\Big(r_h(s_h,a_h)+\OCE_{s_{h+1}\sim\mathbb{P}_h(\cdot|s_h,a_h)}^u\big\{V_{h+1}^{\pi^*}(s_{h+1})\big\}\Big)\\
            &-\Big(r_h(s_h,a_h)+\OCE_{s_{h+1}\sim\mathbb{P}_h(\cdot|s_h,a_h)}^u\big\{\widehat{V}_{h+1}(s_{h+1})\big\}\Big)\\
            =&\sup_{b\in[0,H]}\Big\{b+\mathbb{E}_{s_{h+1}\sim\mathbb{P}_h(\cdot|s_h,a_h)}\big[u(V_{h+1}^{\pi^*}(s_{h+1})-b)\big]\Big\}\\
            &-
            \sup_{b\in[0,H]}\Big\{b+\mathbb{E}_{s_{h+1}\sim\mathbb{P}_h(\cdot|s_h,a_h)}\big[u(\widehat{V}_{h+1}(s_{h+1})-b)\big]\Big\}
        .\end{split}
    \end{equation*}
    Let $b^\dag\in\argmax_{b\in[0,H]}\{b+\mathbb{E}_{s_{h+1}\sim\mathbb{P}_h(\cdot|s_h,a_h)}[u(V_{h+1}^{\pi^*}(s_{h+1})-b)]\}$. Then
    \begin{equation*}
        \begin{split}
            &\mathbb{B}_{h}V_{h+1}^{\pi^*}(s_h, a_h)-\mathbb{B}_h\widehat{V}_{h+1}(s_h, a_h)\\
            \leq&
            \mathbb{E}_{s_{h+1}\sim\mathbb{P}_h(\cdot|s_h,a_h)}\Big[u\big(V_{h+1}^{\pi^*}(s_{h+1})-b^\dag\big)-u\big(\widehat{V}_{h+1}(s_{h+1})-b^\dag\big)\Big]
            \\
            \leq&
            \mathbb{E}_{s_{h+1}\sim\mathbb{P}_h(\cdot|s_h,a_h)}\Big[V_{h+1}^{\pi^*}(s_{h+1})-\widehat{V}_{h+1}(s_{h+1})\Big]
        .\end{split}
    \end{equation*}
The last inequality follows from $1\in\partial u(0)$, $0\leq\widehat{V}_{h}(s_{h})\leq V_{h}^{\pi^*}(s_{h})\leq H$, and the concavity and monotonicity of $u$. This completes the proof.
\end{proof}

\begin{lemma}\label{lemma: dynamic-OCE error vh}
    Under the dynamic-OCE setting,
    \begin{equation*}
        \begin{split}
            V_h^{\pi^*}(s_h)-\widehat{V}_h(s_h)
            \leq &
            \big\langle \mathbb{E}_{s_{h+1}\sim\mathbb{P}_h(\cdot|s_h,a_h)}\big[V_{h+1}^{\pi^*}(s_{h+1})-\widehat{V}_{h+1}(s_{h+1})\big] +\iota_h(s_h,\cdot), \pi^*_h(\cdot|s_h) \big\rangle_{\mathcal{A}} \\
            &-\big\langle \widehat{Q}_h(s_h, \cdot), \pi^*_h(\cdot|s_h)-\widehat{\pi}_h(\cdot|s_h) \big\rangle_{\mathcal{A}}
        .\end{split}
    \end{equation*}
\end{lemma}

\begin{proof}
    By the Bellman equation,
    \begin{equation*}
    {\small
        \begin{split}
            &V_h^{\pi^*}(s_h)-\widehat{V}_h(s_h)\\
            =&\big\langle Q_h^{\pi^*}(s_h, \cdot), \pi^*_h(\cdot|s_h) \big\rangle_{\mathcal{A}}
            -\big\langle \widehat{Q}_h(s_h, \cdot), \pi^*_h(\cdot|s_h) \big\rangle_{\mathcal{A}}\\
            &+\big\langle \widehat{Q}_h(s_h, \cdot), \pi^*_h(\cdot|s_h) \big\rangle_{\mathcal{A}}
            -\big\langle \widehat{Q}_h(s_h, \cdot), \widehat{\pi}_h(\cdot|s_h) \big\rangle_{\mathcal{A}}\\
            =&\big\langle Q_h^{\pi^*}(s_h, \cdot)-\widehat{Q}_h(s_h, \cdot), \pi^*_h(\cdot|s_h) \big\rangle_{\mathcal{A}}
            -\big\langle \widehat{Q}_h(s_h, \cdot), \pi^*_h(\cdot|s_h)-\widehat{\pi}_h(\cdot|s_h) \big\rangle_{\mathcal{A}}\\
            =&\big\langle \mathbb{B}_hV_{h+1}^{\pi^*}(s_h, \cdot)-\mathbb{B}_h\widehat{V}_{h+1}(s_h, \cdot)+\iota_h(s_h,\cdot), \pi^*_h(\cdot|s_h) \big\rangle_{\mathcal{A}}
            -\big\langle \widehat{Q}_h(s_h, \cdot), \pi^*_h(\cdot|s_h)-\widehat{\pi}_h(\cdot|s_h) \big\rangle_{\mathcal{A}}\\
            \leq&
            \big\langle \mathbb{E}_{s_{h+1}\sim\mathbb{P}_h(\cdot|s_h,a_h)}\big[V_{h+1}^{\pi^*}(s_{h+1})-\widehat{V}_{h+1}(s_{h+1})\big] +\iota_h(s_h,\cdot), \pi^*_h(\cdot|s_h) \big\rangle_{\mathcal{A}}\\
            &-\big\langle \widehat{Q}_h(s_h, \cdot), \pi^*_h(\cdot|s_h)-\widehat{\pi}_h(\cdot|s_h) \big\rangle_{\mathcal{A}},
        \end{split}
        }
    \end{equation*}
    where the last inequality holds due to Lemma \ref{lemma: dynamic-OCE ineq of u}. This completes the proof.
\end{proof}

\begin{lemma}\label{lemma: dynamic-OCE bound firsterm v-v}
    Under the dynamic-OCE setting, we have
    \begin{equation*}
        \begin{split}
            V_1^*(s_1)-\widehat{V}_1(s_1)\leq\sum_{h=1}^H\mathbb{E}_{\pi^*}\big[\iota_h(s_h,a_h)\mid s_1\big]
        .\end{split}
    \end{equation*}
\end{lemma}

\begin{proof}
    Define
    \begin{equation*}
        \begin{split}
            \big(\mathbb{J}_hf\big)(s)=\big\langle f(s,\cdot), \pi^*_h(\cdot|s)\big\rangle_{\mathcal{A}}, \quad
            \big(\mathbb{P}_hf\big)(s,a)=\mathbb{E}_{s'\sim\mathbb{P}_h(\cdot|s,a)}\big[f(s')\big]
        .\end{split}
    \end{equation*}
    Recursively applying Lemma \ref{lemma: dynamic-OCE error vh} yields
    \begin{equation*}
        \begin{split}
             &V_1^*(s_1)-\widehat{V}_1(s_1)\\
             \leq&
             \Big(\prod_{i=1}^H\mathbb{J}_i\mathbb{P}_i\Big)\big(V_{H+1}^*(s_{H+1})-\widehat{V}_{H+1}(s_{H+1})\big)
             +\sum_{h=1}^H\Big(\prod_{i=1}^{h-1}\mathbb{J}_i\mathbb{P}_i\Big)\big(\mathbb{J}_h\iota_h\big)\\
             &+\sum_{h=1}^H\Big(\prod_{i=1}^{h-1}\mathbb{J}_i\mathbb{P}_i\Big)\big\langle \widehat{Q}_h(\cdot, \cdot), \pi^*_h(\cdot|\cdot)-\widehat{\pi}_h(\cdot|\cdot) \big\rangle_{\mathcal{A}}\\
             =&
             \sum_{h=1}^H\Big(\prod_{i=1}^{h-1}\mathbb{J}_i\mathbb{P}_i\Big)\big(\mathbb{J}_h\iota_h\big)
             +\sum_{h=1}^H\Big(\prod_{i=1}^{h-1}\mathbb{J}_i\mathbb{P}_i\Big)\big\langle \widehat{Q}_h(\cdot, \cdot), \pi^*_h(\cdot|\cdot)-\widehat{\pi}_h(\cdot|\cdot) \big\rangle_{\mathcal{A}}\\
             =&\sum_{h=1}^H\mathbb{E}_{\pi^*}\big[\iota_h(s_h,a_h)\mid s_1\big]+\sum_{h=1}^H\mathbb{E}_{\pi^*}\Big[\big\langle \widehat{Q}_h(s_h, \cdot), \pi^*_h(\cdot|s_h)-\widehat{\pi}_h(\cdot|s_h) \big\rangle_{\mathcal{A}}\Big|s_1\Big]\\
             \leq&
             \sum_{h=1}^H\mathbb{E}_{\pi^*}\big[\iota_h(s_h,a_h)\mid s_1\big],
        \end{split}
    \end{equation*}
    The first equality uses $V_{H+1}^*(s_{H+1})=\widehat{V}_{H+1}(s_{H+1})=0$ for every $s_{H+1}\in\mathcal{S}$. The last inequality follows from $\widehat{\pi}_h(\cdot|s_h)\in\argmax_{\pi_h}\langle\widehat{Q}_h(s_h,\cdot),\pi_h(\cdot|s_h)\rangle_{\mathcal{A}}$, which implies $\langle \widehat{Q}_h(s_h,\cdot),\allowbreak \pi^*_h(\cdot|s_h)-\widehat{\pi}_h(\cdot|s_h) \rangle_{\mathcal{A}}\leq 0$. This completes the proof.
\end{proof}

\begin{definition}
\label{def: dynamic-OCE B}
    In the dynamic-OCE setting, define
    \begin{equation*}
        \mathbb{C}_h(s'|s,a):=\mathbb{P}_h(s'|s,a)B_{h+1}(s'),
    \end{equation*}
    where $B_{h+1}(s')\in\partial u(V_{h+1}^{\widehat{\pi}}(s')-b_{h+1})$ satisfies $\mathbb{E}_{s'\sim\mathbb{P}_h(\cdot|s,a)}[B_{h+1}(s')]=1$. Since $u$ is nondecreasing, $B_{h+1}(s')\geq0$ for every $s'\in\mathcal{S}$, and hence $\sum_{s'\in\mathcal{S}}\mathbb{C}_h(s'|s,a)=1$.
\end{definition}

\begin{lemma}\label{lemma: dynamic-OCE bound v-v}
    Under the dynamic-OCE setting, it always holds that
    \begin{equation*}
        \begin{split}
            &\OCE_{s'\sim\mathbb{P}_h(\cdot|s,a)}^u\Big\{\widehat{V}_{h+1}(s_{h+1})\Big\}-\OCE_{s'\sim\mathbb{P}_h(\cdot|s,a)}^u\Big\{V_{h+1}^{\widehat{\pi}}(s_{h+1})\Big\}\\
            \leq&\mathbb{E}_{s'\sim\mathbb{C}_h(\cdot|s,a)}\Big[\widehat{V}_{h+1}(s_{h+1})-V_{h+1}^{\widehat{\pi}}(s_{h+1})\Big],
        \end{split}
    \end{equation*}
    where $\mathbb{C}_h(\cdot|s,a)$ is a probability measure defined in Definition \ref{def: dynamic-OCE B}.
\end{lemma}

\begin{proof}
    Let $\widehat{b}_{h+1}\in\argmax_{b\in[0,H-h]}\{b+\mathbb{E}_{s'\sim\mathbb{P}_h(\cdot|s,a)}[\widehat{V}_{h+1}(s_{h+1})-b]\}$ and $b_{h+1}^{\widehat{\pi}}\in\argmax_{b\in[0,H-h]}\{b+\mathbb{E}_{s'\sim\mathbb{P}_h(\cdot|s,a)}[V_{h+1}^{\widehat{\pi}}(s_{h+1})-b]\}$. Then
    \begin{equation*}
    {\small
        \begin{split}
            &\OCE_{s'\sim\mathbb{P}_h(\cdot|s,a)}^u\Big\{\widehat{V}_{h+1}(s_{h+1})\Big\}-\OCE_{s'\sim\mathbb{P}_h(\cdot|s,a)}^u\Big\{V_{h+1}^{\widehat{\pi}}(s_{h+1})\Big\}\\
            =&\max_{b\in[0,H-h]}\Big\{b+\mathbb{E}_{s'\sim\mathbb{P}_h(\cdot|s,a)}\Big[u\Big(\widehat{V}_{h+1}(s_{h+1})-b\Big)\Big]\Big\}\\
            &-\max_{b\in[0,H-h]}\Big\{b+\mathbb{E}_{s'\sim\mathbb{P}_h(\cdot|s,a)}\Big[u\Big(V_{h+1}^{\widehat{\pi}}(s_{h+1})-b\Big)\Big]\Big\}\\
            =&\Big\{\widehat{b}_{h+1}+\mathbb{E}_{s'\sim\mathbb{P}_h(\cdot|s,a)}\Big[u\Big(\widehat{V}_{h+1}(s_{h+1})-\widehat{b}_{h+1}\Big)\Big]\Big\}
            -\Big\{b_{h+1}^{\widehat{\pi}}+\mathbb{E}_{s'\sim\mathbb{P}_h(\cdot|s,a)}\Big[u\Big(V_{h+1}^{\widehat{\pi}}(s_{h+1})-b_{h+1}^{\widehat{\pi}}\Big)\Big]\Big\}\\
            \leq&\Big(\widehat{b}_{h+1}-b_{h+1}^{\widehat{\pi}}\Big)
            +\mathbb{E}_{s'\sim\mathbb{P}_h(\cdot|s,a)}\Big[B_{h+1}(s_{h+1})\Big(\widehat{V}_{h+1}(s_{h+1})-V_{h+1}^{\widehat{\pi}}(s_{h+1})-\Big(\widehat{b}_{h+1}-b_{h+1}^{\widehat{\pi}}\Big)\Big)\Big]
        .\end{split}
        }
    \end{equation*}
    The preceding inequality follows from concavity of $u$, namely $u(y)\leq u(x)+z(y-x)$ for $z\in\partial u(x)$. Consequently,
    \begin{equation*}
        \begin{split}
            &\OCE_{s'\sim\mathbb{P}_h(\cdot|s,a)}^u\Big\{\widehat{V}_{h+1}(s_{h+1})\Big\}-\OCE_{s'\sim\mathbb{P}_h(\cdot|s,a)}^u\Big\{V_{h+1}^{\widehat{\pi}}(s_{h+1})\Big\}\\
            =&\Big(1-\mathbb{E}_{s'\sim\mathbb{P}_h(\cdot|s,a)}\Big[B_{h+1}(s_{h+1})\Big]\Big)\Big(\widehat{b}_{h+1}-b_{h+1}^{\widehat{\pi}}\Big)\\
            & +\mathbb{E}_{s'\sim\mathbb{P}_h(\cdot|s,a)}\Big[B_{h+1}(s_{h+1})\Big(\widehat{V}_{h+1}(s_{h+1})-V_{h+1}^{\widehat{\pi}}(s_{h+1})\Big)\Big]\\
            =&\mathbb{E}_{s'\sim\mathbb{C}_h(\cdot|s,a)}\Big[\widehat{V}_{h+1}(s_{h+1})-V_{h+1}^{\widehat{\pi}}(s_{h+1})\Big]
        .\end{split}
    \end{equation*}
    The last equality follows from Definition~\ref{def: dynamic-OCE B} and $1-\mathbb{E}_{s'\sim\mathbb{P}_h(\cdot|s,a)}[B_{h+1}(s_{h+1})]=0$. This completes the proof.
\end{proof}

\begin{definition}\label{def: dynamic-OCE omega}
    Under the dynamic-OCE formulation and Definition~\ref{def: dynamic-OCE B}, define the state-action distribution
    \begin{equation*}
        \omega_h(s_h,a_h)=
            \begin{cases}
                1, &h=1\\
                \mathbb{C}_1(s_2|s_1,a_1), &h=2\\
                \displaystyle\sum_{s_2\in\mathcal{S}}\sum_{s_3\in\mathcal{S}}\dots\sum_{s_{h-1}\in\mathcal{S}}\mathbb{C}_1(s_2|s_1,a_1)\mathbb{C}_2(s_3|s_2,a_2)\dots\mathbb{C}_{h-1}(s_h|s_{h-1},a_{h-1}), &h\geq3
            ,\end{cases}
    \end{equation*}
        where $\mathbb{C}_h(\cdot|s,a)$ is the probability measure defined in Definition~\ref{def: dynamic-OCE B}.
\end{definition}

\begin{lemma}\label{lemma: dynamic-OCE bound second term v-v}
    Under the dynamic-OCE setting, we have 
    \begin{equation*}
        \widehat{V}_1(s_1)-V_1^{\widehat{\pi}}(s_1)\leq \sum_{h=1}^H\mathbb{E}_{\omega_h}\big[-\iota_h(s_h,a_h)\big|s_1\big],
    \end{equation*}
    where $\mathbb{E}_{\omega_h}$ represents $\mathbb{E}_{(s_h,a_h)\sim\omega_h(\cdot,\cdot)}$.
\end{lemma}

\begin{proof}
    By the definition of $\widehat{V}_1(s_1)$ and $\widehat{Q}_1(s_1,a_1)$, we have
    \begin{equation*}
        \begin{split}
            &\widehat{V}_1(s_1)-V_1^{\widehat{\pi}}(s_1)
            \leq
            \widehat{Q}_1(s_1,a_1)-Q_1^{\widehat{\pi}}(s_1,a_1)
            =
            \Big(\mathbb{\widehat{B}}_1\widehat{V}_2\Big)(s_1,a_1)-\Gamma_1(s_1,a_1)-\Big(\mathbb{B}_1V_2^{\widehat{\pi}}\Big)(s_1,a_1)\\
            =&
            \Big(\mathbb{\widehat{B}}_1\widehat{V}_2\Big)(s_1,a_1)
            -\Big(\mathbb{B}_1\widehat{V}_2\Big)(s_1,a_1)
            +\Big(\mathbb{B}_1\widehat{V}_2\Big)(s_1,a_1)
            -\Big(\mathbb{B}_1V_2^{\widehat{\pi}}\Big)(s_1,a_1)
            -\Gamma_1(s_1,a_1)\\
            =&
            \Big(\mathbb{\widehat{B}}_1\widehat{V}_2\Big)(s_1,a_1)
            -\Big(\mathbb{B}_1\widehat{V}_2\Big)(s_1,a_1)
            +\mathbb{B}_1\Big(\widehat{V}_2-V_2^{\widehat{\pi}}\Big)(s_1,a_1)
            -\Gamma_1(s_1,a_1),
        \end{split}
    \end{equation*}
    where the first inequality holds because of $a_1=\argmax_{a\in\mathcal{A}}\widehat{Q}_1(s_1,a)$ such that $V_1^{\widehat{\pi}}(s_1)=\max_{a\in\mathcal{A}}Q_1^{\widehat{\pi}}(s_1,a)\geq Q_1^{\widehat{\pi}}(s_1,a_1)$. Then by plugging in the definition of $\mathbb{B}_h$,
    \begin{equation*}
        \begin{split}
            &\widehat{V}_1(s_1)-V_1^{\widehat{\pi}}(s_1)\\
            \leq&
            \Big(\mathbb{\widehat{B}}_1\widehat{V}_2\Big)(s_1,a_1)
            -\Big(\mathbb{B}_1\widehat{V}_2\Big)(s_1,a_1)
            +\OCE_{s_{2}\sim\mathbb{P}_1(\cdot|s_1,a_1)}^u\Big\{\widehat{V}_{2}(s_{2})\Big\} \\
            & -\OCE_{s_{2}\sim\mathbb{P}_1(\cdot|s_1,a_1)}^u\Big\{V_{2}^{\widehat{\pi}}(s_{2})\Big\}
            -\Gamma_1(s_1,a_1)\\
            \leq&
            \Big(\mathbb{\widehat{B}}_1\widehat{V}_2\Big)(s_1,a_1)
            -\Big(\mathbb{B}_1\widehat{V}_2\Big)(s_1,a_1)
            +\mathbb{E}_{s_{2}\sim\mathbb{C}_1(\cdot|s_1,a_1)}\Big[\widehat{V}_{2}(s_{2})-V_{2}^{\widehat{\pi}}(s_{2})\Big]
            -\Gamma_1(s_1,a_1)
        .\end{split}
    \end{equation*}
    The last inequality follows from Lemma~\ref{lemma: dynamic-OCE bound v-v}. The same argument yields
    \begin{align}\label{eq: dynamic-OCE bound v-v}
        &\widehat{V}_h(s_h)-V_h^{\widehat{\pi}}(s_h)\\\nonumber
        \leq&
        \Big(\mathbb{\widehat{B}}_h\widehat{V}_{h+1}\Big)(s_h,a_h)
        -\Big(\mathbb{B}_h\widehat{V}_{h+1}\Big)(s_h,a_h)\\
        &+\mathbb{E}_{s_{h+1}\sim\mathbb{C}_h(\cdot|s_h,a_h)}\Big[\widehat{V}_{h+1}(s_{h+1})-V_{h+1}^{\widehat{\pi}}(s_{h+1})\Big]
        -\Gamma_h(s_h,a_h).\nonumber
    \end{align}
    
    Recursively applying \eqref{eq: dynamic-OCE bound v-v}, together with Definition~\ref{def: dynamic-OCE omega}, \eqref{eq: def of iota}, and $\widehat{V}_{H+1}(s)=V_{H+1}^{\widehat{\pi}}(s)=0$ for every $s\in\mathcal{S}$, proves the claim.
\end{proof}

\begin{lemma}\label{lemma: dynamic-OCE bound iota}
    Under event $\mathcal{E}_h$, for all $s\in\mathcal{S}$, $a\in\mathcal{A}$, and $h\in[H]$, we have
    \begin{equation*}
        \iota_h(s,a)\geq 0.
    \end{equation*}
\end{lemma}

\begin{proof}
    If $\overline{Q}_h(s,a)<0$, by the definition of $\widehat{Q}_h(s,a)$ in Algorithm \ref{alg:recursiv OCE-PVI}, we have
    \begin{equation*}
        \widehat{Q}_h(s,a)=\max\Big\{\min\big\{\overline{Q}_h(s,a),H-h+1\big\},0\Big\}=0.
    \end{equation*}
    Therefore,
    \begin{equation*}
        \iota_h(s,a)=\mathbb{B}_{h}\widehat{V}_{h+1}(s,a)\geq0.
    \end{equation*}
    If $\overline{Q}_h(s,a)\geq0$, we have
    \begin{equation*}
        \widehat{Q}_h(s,a)=
        \max\Big\{\min\big\{\overline{Q}_h(s,a),H-h+1\big\},0\Big\}
        \leq \overline{Q}_h(s,a).
    \end{equation*}
    Consequently,
    \begin{equation*}
        \begin{split}
            \iota_h(s,a)
            \geq \mathbb{B}_{h}\widehat{V}_{h+1}(s,a)-\Big(\widehat{\mathbb{B}}_{h}\widehat{V}_{h+1}(s,a)-\Gamma_h(s,a)\Big)
            \geq 0,
        \end{split}
    \end{equation*}
    where the second inequality follows from the definition of $\mathcal{E}_h$. This completes the proof.
\end{proof}


By Lemma~\ref{lemma: dynamic-OCE iota}, the suboptimality bound depends on the bonus $\Gamma_h$ and the Bellman estimation error $\mathbb{B}_{h}\widehat{V}_{h+1}(s,a)-\widehat{\mathbb{B}}_{h}\widehat{V}_{h+1}(s,a)$.

\begin{lemma}\label{lemma: dynamic-OCE iota}
In the dynamic-OCE setting, with probability at least $1-\delta$,
\begin{equation*}
        \Big|\mathbb{B}_{h}\widehat{V}_{h+1}(s,a)-\widehat{\mathbb{B}}_{h}\widehat{V}_{h+1}(s,a)\Big|\leq\Gamma_h(s,a), \forall h\in [H], s\in\mathcal{S}, a\in\mathcal{A},
    \end{equation*}
    where $\{\Gamma_h\}_{h=1}^H$ is the bonus term. 
    the suboptimality of Algorithm~\ref{alg: recursiv tabular OCE-PVI} and Algorithm~\ref{alg:recursiv OCE-PVI} is bounded by
    \begin{equation*}
        \SubOptD(\widehat{\pi})\leq \sum_{h=1}^H\mathbb{E}_{\pi^*}\big[\iota_h(s_h,a_h)\big|s_1\big],
    \end{equation*}
where $\mathbb{E}_{\pi^*}$ is taken with respect to trajectories generated by $\pi^*$.
\end{lemma}

Lemma~\ref{lemma: dynamic-OCE iota} applies to both the tabular and linear-function-approximation settings.
\begin{proof}
    By the definition of suboptimality,
\begin{equation*}
    \begin{split}
        \SubOptD(\widehat{\pi})=&V_1^*(s_1)-V_1^{\widehat{\pi}}(s_1)
        =
        V_1^*(s_1)-\widehat{V}_1(s_1)+\widehat{V}_1(s_1)-V_1^{\widehat{\pi}}(s_1)\\
        \leq&
        \sum_{h=1}^H\mathbb{E}_{\pi^*}\big[\iota_h(s_h,a_h)\big|s_1\big]
        +\sum_{h=1}^H\mathbb{E}_{\omega_h}\big[-\iota_h(s_h,a_h)\big|s_1\big]
        \leq
        \sum_{h=1}^H\mathbb{E}_{\pi^*}\big[\iota_h(s_h,a_h)\big|s_1\big],
    \end{split}
\end{equation*}
where the first inequality follows from Lemmas~\ref{lemma: dynamic-OCE bound firsterm v-v} and~\ref{lemma: dynamic-OCE bound second term v-v}, while the last follows from Lemma~\ref{lemma: dynamic-OCE bound iota}, which gives $\sum_{h=1}^H\mathbb{E}_{\omega_h}\big[-\iota_h(s_h,a_h)\big|s_1\big]\leq 0$. This completes the proof.
\end{proof}

\begin{lemma}\label{lemma: dynamic-OCE xu lemma 4}
    For any $\delta\in(0,1)$, any $(s,a)\in\mathcal{S}\times\mathcal{A}$, $h\in[H]$, and $b^*\in[0,H-h]$, with probability at least $1-\delta$, the following inequality holds:
    \begin{equation*}
        \begin{split}
        &\mathbb{E}_{s'\sim\mathbb{P}_h(\cdot|s,a)}\Big[u(\widehat{V}_{h+1}(s_{h+1})-b^*)\Big]
        -\mathbb{E}_{s'\sim\widehat{\mathbb{P}}_h(\cdot|s,a)}\Big[u(\widehat{V}_{h+1}(s_{h+1})-b^*)\Big]\\
        \leq&
        (H-h)L_u\sqrt{\frac{2\log\frac{|\mathcal{S}||\mathcal{A}|HK}{\delta}}{\max\{1,N_h(s,a)\}}}.
        \end{split}
    \end{equation*}
   
\end{lemma}

\begin{proof}
    When $N_h(s,a)=0$, 
    we have 
    \begin{equation*}
        \begin{split}
        &\mathbb{E}_{s'\sim\mathbb{P}_h(\cdot|s,a)}\Big[u\Big(\widehat{V}_{h+1}(s_{h+1})-b^*\Big)\Big]-\mathbb{E}_{s'\sim\widehat{\mathbb{P}}_h(\cdot|s,a)}\Big[u\Big(\widehat{V}_{h+1}(s_{h+1})-b^*\Big)\Big]\\
        \leq&
        u(H-h-b^*)-u(-b^*)
        \leq
        \big[u(H-h-b^*)-u(-b^*)\big]\sqrt{\frac{2\log\frac{|\mathcal{S}||\mathcal{A}|HK}{\delta}}{\max\{1,N_h(s,a)\}}},
        \end{split}
    \end{equation*}
    where the first inequality holds since the utility function $u$ is nondecreasing and $\widehat{V}_{h+1}(s)\in[0,H-h]$. The last inequality holds due to the fact that $\log\frac{|\mathcal{S}||\mathcal{A}|HK}{\delta}> 1$.
    When $N_h(s,a)\geq 1$, we have 
    \begin{equation*}
        \begin{split}
            \mathbb{E}_{s'\sim\widehat{\mathbb{P}}_h(\cdot|s,a)}\Big[u\Big(\widehat{V}_{h+1}(s_{h+1})-b^*\Big)\Big]
            =
            \frac{1}{N_h(s,a)}\sum_{k=1}^K\mathbb{I}\Big((s_h^k,a_h^k)=(s,a)\Big)u\Big(\widehat{V}_{h+1}(s_{h+1}^k)-b^*\Big)
        .\end{split}
    \end{equation*}
    Define
    \begin{equation*}
        \begin{split}
            &X_i\\
            =&
            \mathbb{E}\Big[\mathbb{I}\Big((s_h^k,a_h^k)=(s,a)\Big)u\Big(\widehat{V}_{h+1}(s_{h+1}^k)-b^*\Big)\Big]
            -\mathbb{I}\Big((s_h^k,a_h^k)=(s,a)\Big)u\Big(\widehat{V}_{h+1}(s_{h+1}^k)-b^*\Big),
        \end{split}
    \end{equation*}
    we have
    \begin{equation*}
        |X_i|\leq u(H-h-b^*)-u(-b^*),
    \end{equation*}
    since $\widehat{V}_{h+1}(s)\in[0,H-h]$.
    For any $i\neq j$, $X_i$ and $X_j$ are independent. Hence, Hoeffding's inequality gives, with probability at least $1-\frac{\delta}{|\mathcal{S}||\mathcal{A}|HK}$,
    \begin{equation*}
        \begin{split}
            &\sum_{i=1}^KX_i\\
            =&
            N_h(s,a)\mathbb{E}_{s'\sim\mathbb{P}_h(\cdot|s,a)}\Big[u\Big(\widehat{V}_{h+1}(s_{h+1})-b^*\Big)\Big]-\sum_{k=1}^K\mathbb{I}\big((s_h^k,a_h^k)=(s,a)\big)u\Big(\widehat{V}_{h+1}(s_{h+1}^k)-b^*\Big)\\
            =&
            N_h(s,a)\mathbb{E}_{s'\sim\mathbb{P}_h(\cdot|s,a)}\Big[u\Big(\widehat{V}_{h+1}(s_{h+1})-b^*\Big)\Big]-N_h(s,a)\mathbb{E}_{s'\sim\widehat{\mathbb{P}}_h(\cdot|s,a)}\Big[u\Big(\widehat{V}_{h+1}(s_{h+1})-b^*\Big)\Big]\\
            \leq&
            \big[u(H-h-b^*)-u(-b^*)\big]\sqrt{2N_h(s,a)\log\frac{|\mathcal{S}||\mathcal{A}|HK}{\delta}}.
        \end{split}
    \end{equation*}
    Consequently, with probability at least $1-\frac{\delta}{|\mathcal{S}||\mathcal{A}|HK}$,
    \begin{equation*}
        \begin{split}
            &\mathbb{E}_{s'\sim\mathbb{P}_h(\cdot|s,a)}\Big[u\Big(\widehat{V}_{h+1}(s_{h+1})-b^*\Big)\Big]-\mathbb{E}_{s'\sim\widehat{\mathbb{P}}_h(\cdot|s,a)}\Big[u\Big(\widehat{V}_{h+1}(s_{h+1})-b^*\Big)\Big]\\
            \leq &
            \big[u(H-h-b^*)-u(-b^*)\big]\sqrt{\frac{2\log\frac{|\mathcal{S}||\mathcal{A}|HK}{\delta}}{\max\{1,N_h(s,a)\}}}
        .\end{split}
    \end{equation*}
    This completes the proof of Lemma \ref{lemma: dynamic-OCE xu lemma 4}.
\end{proof}

\begin{lemma}\label{lemma: dynamic-OCE xu lemma 4 r}
    For any $\delta\in(0,1)$, $(s,a)\in\mathcal{S}\times\mathcal{A}$, and $h\in[H]$, with probability at least $1-\delta$,
    \begin{equation*}
        r_h(s,a)-\widehat{r}_h(s,a)\leq \sqrt{\frac{1}{\max\{1,N_h(s,a)\}}}.
    \end{equation*}
\end{lemma}

\begin{proof}
    When $N_h(s,a)=0$, $\widehat{r}_h(s,a)=0$, we have 
    \begin{equation*}
        r_h(s,a)-\widehat{r}_h(s,a)=
        r_h(s,a)
        \leq 1 =
        \sqrt{\frac{1}{\max\{1,N_h(s,a)\}}},
    \end{equation*}
    where the first inequality holds since $r_h(s,a)\in [0,1]$. The last inequality holds due to the fact that $\log\frac{|\mathcal{S}||\mathcal{A}|HK}{\delta}> 1$.
    When $N_h(s,a)\geq 1$, we have
    \begin{align*}
        r_h(s,a)-\widehat{r}_h(s,a)=0\leq \sqrt{\frac{1}{\max\{1,N_h(s,a)\}}}.
    \end{align*}
    This completes the proof.
\end{proof}

\subsection{Proof of Theorem \ref{the: subopt dynamic-OCE meta}}\label{subsec: Proof of subopt dynamic-OCE meta}
By Lemma~\ref{lemma: dynamic-OCE iota}, it suffices to bound $\mathbb{B}_h\widehat{V}_{h+1}(s,a)-\widehat{\mathbb{B}}_h\widehat{V}_{h+1}(s,a)$.
Under the dynamic-OCE formulation,
\begin{equation*}
    \begin{split}
        &\mathbb{B}_h\widehat{V}_{h+1}(s,a)-\widehat{\mathbb{B}}_h\widehat{V}_{h+1}(s,a)\\
        =&
       r_h(s,a)-\widehat{r}_h(s,a)
        +\OCE_{s'\sim\mathbb{P}_h(\cdot|s,a)}^u\Big\{\widehat{V}_{h+1}(s_{h+1})\Big\}-\OCE_{s'\sim\widehat{\mathbb{P}}_h(\cdot|s,a)}^u\Big\{\widehat{V}_{h+1}(s_{h+1})\Big\}\\
        =&
        r_h(s,a)-\widehat{r}_h(s,a)
        +\max_{b\in[0,H-h]}\Big\{b+\mathbb{E}_{s'\sim\mathbb{P}_h(\cdot|s,a)}\Big[u\Big(\widehat{V}_{h+1}(s_{h+1})-b\Big)\Big]\Big\}\\
        &-\max_{b\in[0,H-h]}\Big\{b+\mathbb{E}_{s'\sim\widehat{\mathbb{P}}_h(\cdot|s,a)}\Big[u\Big(\widehat{V}_{h+1}(s_{h+1})-b\Big)\Big]\Big\}\\
        \leq&
        r_h(s,a)-\widehat{r}_h(s,a)\\
        &+\mathbb{E}_{s'\sim\mathbb{P}_h(\cdot|s,a)}\Big[u\Big(\widehat{V}_{h+1}(s_{h+1})-b^*\Big)\Big]
        -\mathbb{E}_{s'\sim\widehat{\mathbb{P}}_h(\cdot|s,a)}\Big[u\Big(\widehat{V}_{h+1}(s_{h+1})-b^*\Big)\Big],
    \end{split}
\end{equation*}
where the inequality uses $b^*\in \argmax_{b\in[0,H-h]}\{b+\mathbb{E}_{s'\sim\mathbb{P}_h(\cdot|s,a)}[u(\widehat{V}_{h+1}(s_{h+1})-b)]\}$. By Lemmas~\ref{lemma: dynamic-OCE xu lemma 4} and~\ref{lemma: dynamic-OCE xu lemma 4 r}, with probability at least $1-\delta$,
\begin{align*}
    &\mathbb{B}_h\widehat{V}_{h+1}(s,a)-\widehat{\mathbb{B}}_h\widehat{V}_{h+1}(s,a)\\
    \leq&
    \sqrt{\frac{1}{\max\{1,N_h(s,a)\}}}+\big[u(H-h-b^*)-u(-b^*)\big]\sqrt{\frac{2\log\frac{|\mathcal{S}||\mathcal{A}|HK}{\delta}}{\max\{1,N_h(s,a)\}}}\\
    \leq&
    \sqrt{\frac{1}{\max\{1,N_h(s,a)\}}}+\big[u(H-h)-u(h-H)\big]\sqrt{\frac{2\log\frac{|\mathcal{S}||\mathcal{A}|HK}{\delta}}{\max\{1,N_h(s,a)\}}}.
\end{align*}
This bounds $\mathbb{B}_h\widehat{V}_{h+1}(s,a)-\widehat{\mathbb{B}}_h\widehat{V}_{h+1}(s,a)$. Consequently,
\begin{equation*}
    \begin{split}
        &\SubOptD(\widehat{\pi})\\
        \leq &
        2\sum_{h=1}^H\big[u(H-h)-u(h-H)\big]\mathbb{E}_{\pi^*}\Bigg[\sqrt{\frac{1}{\max\{1,N_h(s,a)\}}}+\sqrt{\frac{2\log\frac{|\mathcal{S}||\mathcal{A}|HK}{\delta}}{\max\{1,N_h(s,a)\}}}\Bigg|s_1,b_1^*\Bigg]\\
        \leq&
        2\sum_{h=1}^H\big[u(H-h)-u(h-H)\big]\sum_{s,a}d_h^{\pi^*}(s,a)\Bigg(\sqrt{\frac{1}{\max\{1,N_h(s,a)\}}}+\sqrt{\frac{2\log\frac{|\mathcal{S}||\mathcal{A}|HK}{\delta}}{\max\{1,N_h(s,a)\}}}\Bigg)\\
        =&
        2\sum_{h=1}^H\big[u(H-h)-u(h-H)\big]\sum_{s,a}\sqrt{d_h^{\pi^*}(s,a)}\Bigg(\sqrt{\frac{d_h^{\pi^*}(s,a)}{Kd_h^{\mu}(s,a)}}+\sqrt{\frac{2d_h^{\pi^*}(s,a)\log\frac{|\mathcal{S}||\mathcal{A}|HK}{\delta}}{Kd_h^{\mu}(s,a)}}\Bigg)
    .\end{split}
\end{equation*}    
Since $\frac{d_h^{\pi^*}(s,a)}{d_h^{\mu}(s,a)}\leq C^*$,
\begin{equation*}
    \begin{split}
        &\SubOptD(\widehat{\pi})\\
        \leq &
        2\sum_{h=1}^H\sum_{s,a}\sqrt{d_h^{\pi^*}(s,a)}\Bigg(\sqrt{\frac{C^*}{K}}+\big[u(H-h)-u(h-H)\big]\sqrt{\frac{2C^*\log\frac{|\mathcal{S}||\mathcal{A}|HK}{\delta}}{K}}\Bigg)\\
        =&
        2\sum_{h=1}^H\sum_{s,a}\sqrt{d_h^{\pi^*}(s,a)\cdot\mathbb{I}\big(a=a^*_{h,s}\big)}\Bigg(\sqrt{\frac{C^*}{K}}+\big[u(H-h)-u(h-H)\big]\sqrt{\frac{2C^*\log\frac{|\mathcal{S}||\mathcal{A}|HK}{\delta}}{K}}\Bigg)\\
        \leq&
        2\sum_{h=1}^H\sqrt{\sum_{s,a}d_h^{\pi^*}(s,a)\cdot\sum_{s,a}\mathbb{I}\big(a=a^*_{h,s}\big)}\Bigg(\sqrt{\frac{C^*}{K}}+\big[u(H-h)-u(h-H)\big]\sqrt{\frac{2C^*\log\frac{|\mathcal{S}||\mathcal{A}|HK}{\delta}}{K}}\Bigg)\\
        =&
        2\sum_{h=1}^H\Bigg(\sqrt{\frac{C^*S}{K}}+\big[u(H-h)-u(h-H)\big]\sqrt{\frac{2C^*S\log\frac{|\mathcal{S}||\mathcal{A}|HK}{\delta}}{K}}\Bigg)\\
        \leq&
        \widetilde{\mathcal{O}}\Bigg(\sum_{h=1}^H\big[1+2(H-h)L_u\big]\sqrt{\frac{C^*|\mathcal{S}|}{K}}\Bigg),
    \end{split}
\end{equation*}   
where $a^*_{h,s}$ is sampled by $a^*_{h,s}\sim\pi^*_h(\cdot|s)$.
This completes the proof.

\subsection{Lemmas for Theorem \ref{the: static-OCE tabular result}}\label{subsec: Lemmas for static-OCE tabular result}
The deterministic-reward result in the main text is proved for the practical update based on $\widehat r_h(s,a)$. This subsection presents a separate stochastic-reward extension. It assumes access to a joint transition--reward model and its corresponding empirical counterpart for $(s',r)$, and it uses the Bellman update $\mathbb E_{s',r}[V(s',b-r)]$. This extension analyzes the joint-model implementation. When $\mathcal{R}_h(r_h\mid s,a)=1$, it specializes to the deterministic-reward setting.
For any $s$, $a$, and $b$, define the Bellman estimation error under stochastic rewards at step $h$ by
\begin{equation}\label{eq: static-OCE def iota}
    \iota_h(s,b,a)=\mathbb{E}_{s'\sim\mathbb{P}_h(\cdot|s,a),r\sim\mathcal{R}_h(\cdot|s,a)}\Big[\widehat{V}_{h+1}(s',b-r)\Big]-\widehat{Q}_h(s,b,a).
\end{equation}
For notational convenience, define the operator $\mathbb{U}_h$ by
\begin{equation*}
    \mathbb{U}_hf(s,b,a)=\mathbb{E}_{s'\sim\mathbb{P}_h(\cdot|s,a),r\sim\mathcal{R}_h(\cdot|s,a)}[f(s',b-r)].
\end{equation*}
Its empirical counterpart is
\begin{align*}
    \widehat{\mathbb{U}}_hf(s,b,a)=\mathbb{E}_{s'\sim\widehat{\mathbb{P}}_h(\cdot|s,a),r\sim\widehat{\mathcal{R}}_h(\cdot|s,a)}[f(s',b-r)].
\end{align*}
To bound the suboptimality, define the following event.
\begin{definition}\label{def: static-OCE event}
    Let
    \begin{equation*}
        \mathcal{E}'_h=\Big\{\Big|\mathbb{U}_h\widehat{V}_{h+1}(s,b,a)-\widehat{\mathbb{U}}_h\widehat{V}_{h+1}(s,b,a)\Big|\leq\Gamma_h(s,b,a)\Big\},
    \end{equation*}
    where $\{\Gamma_h\}_{h=1}^H$ denotes the bonus sequence and $\mathbb{P}(\bigcap_{h=1}^H\mathcal{E}'_h)\geq1-\delta$.
\end{definition}

\begin{lemma}\label{lemma: cai 2020 static-OCE}
    Under the static-OCE formulation and the definition of $\widehat{V}_h(s,b)$,
    \begin{equation*}
        \begin{split}
            &V^{\pi}_1(s_1,b_1^*)-\widehat{V}_1(s_1,b_1^*)\\
            =&
            \sum_{h=1}^H\mathbb{E}_{\pi}\Big[\iota_h(s_h,b_h^*,a_h)\Big|s_1,b_1^*\Big]+\sum_{h=1}^H\mathbb{E}_{\pi}\Big[\Big\langle  \widehat{Q}_h(s_h,b_h^*,\cdot), \pi_h(\cdot|s_h,b_h^*)-\widehat{\pi}_h(\cdot|s_h,b_h^*) \Big\rangle_{\mathcal{A}}\Big|s_1,b_1^*\Big],
        \end{split}
    \end{equation*}
    where $b_h^*$ is the stage-$h$ budget induced by the initial budget $b_1^*$ and the recursion $b_{h+1}^*=b_h^*-r_h$.
\end{lemma}

\begin{proof}
    Let $\Delta_h(s_h,b_h^*)=\langle  \widehat{Q}_h(s_h,b_h^*,\cdot), \pi_h(\cdot|s_h,b_h^*)-\widehat{\pi}_h(\cdot|s_h,b_h^*) \rangle_{\mathcal{A}}$. Then
    \begin{equation*}
        \begin{split}
            &V^{\pi}_h(s_h,b_h^*)-\widehat{V}_h(s_h,b_h^*) 
            =\Big\langle Q_h^{\pi}(s_h,b_h^*,\cdot), \pi_h(\cdot|s_h,b_h^*)\Big\rangle_{\mathcal{A}}
            -\Big\langle  \widehat{Q}_h(s_h,b_h^*,\cdot), \widehat{\pi}_h(\cdot|s_h,b_h^*) \Big\rangle_{\mathcal{A}}\\
            =&\Big\langle Q_h^{\pi}(s_h,b_h^*,\cdot), \pi_h(\cdot|s_h,b_h^*)\Big\rangle_{\mathcal{A}}
            -\Big\langle \widehat{Q}_h(s_h,b_h^*,\cdot), \pi_h(\cdot|s_h,b_h^*)\Big\rangle_{\mathcal{A}}\\
            &+\Big\langle \widehat{Q}_h(s_h,b_h^*,\cdot), \pi_h(\cdot|s_h,b_h^*)\Big\rangle_{\mathcal{A}}
            -\Big\langle  \widehat{Q}_h(s_h,b_h^*,\cdot), \widehat{\pi}_h(\cdot|s_h,b_h^*) \Big\rangle_{\mathcal{A}}\\
            =&\Big\langle Q_h^{\pi}(s_h,b_h^*,\cdot)-\widehat{Q}_h(s_h,b_h^*,\cdot), \pi_h(\cdot|s_h,b_h^*)\Big\rangle_{\mathcal{A}}
            +\Big\langle  \widehat{Q}_h(s_h,b_h^*,\cdot), \pi_h(\cdot|s_h,b_h^*)-\widehat{\pi}_h(\cdot|s_h,b_h^*) \Big\rangle_{\mathcal{A}}\\
            =&\Big\langle \mathbb{E}_{s'\sim\mathbb{P}_h(\cdot|s,\cdot)}\Big[V_{h+1}^{\pi}(s_{h+1},b_{h+1}^*)-\widehat{V}_{h+1}(s_{h+1},b_{h+1}^*)\Big]+\iota_h(s_h,b_h^*,\cdot), \pi_h(\cdot|s_h,b_h^*)\Big\rangle_{\mathcal{A}}
            +\Delta_h(s_h,b_h^*)
        .\end{split}
    \end{equation*}
    Therefore, we have
    \begin{equation}\label{eq: static-OCE v-v}
        \begin{split}
            &V^{\pi}_h(s_h,b_h^*)-\widehat{V}_h(s_h,b_h^*)\\
            =&
            \Big\langle \mathbb{E}\Big[V_{h+1}^{\pi}(s_{h+1},b_{h+1}^*)-\widehat{V}_{h+1}(s_{h+1},b_{h+1}^*)\Big]+\iota_h(s_h,b_h^*,\cdot), \pi_h(\cdot|s_h,b_h^*)\Big\rangle_{\mathcal{A}}
            +\Delta_h(s_h,b_h^*)
        .\end{split}
    \end{equation}
    Since $V^{\pi}_{H+1}(s_h,b_h^*)-\widehat{V}_{H+1}(s_h,b_h^*)=u(-b_{H+1}^*)-u(-b_{H+1}^*)=0$, recursively applying \eqref{eq: static-OCE v-v} proves the claim.
\end{proof}

\begin{lemma}\label{lemma: static-OCE term 1}
    Under the definitions of $V^*_1(s_1,b_1^*)$ and $\widehat{V}_1(s_1,b_1^*)$,
    \begin{equation*}
         V^*_1(s_1,b_1^*)-\widehat{V}_1(s_1,b_1^*)\leq \sum_{h=1}^H\mathbb{E}_{\pi^*}\Big[\iota_h(s_h,b_h,a_h)\Big|s_1,b_1^*\Big].
    \end{equation*}
\end{lemma}

\begin{proof}
    Applying Lemma~\ref{lemma: cai 2020 static-OCE} with $\pi_h=\pi_h^*$ for every $h\in[H]$ gives
    \begin{equation*}
        \begin{split}
            &V^*_1(s_1,b_1^*)-\widehat{V}_1(s_1,b_1^*)\\
            =&
            \sum_{h=1}^H\mathbb{E}_{\pi^*}\big[\iota_h(s_h,b_h,a_h)\big|s_1,b_1^*\big]
            +\sum_{h=1}^H\mathbb{E}_{\pi^*}\Big[\Big\langle  \widehat{Q}_h(s_h,b_h^*,\cdot), \pi^*_h(\cdot|s_h,b_h^*)-\widehat{\pi}_h(\cdot|s_h,b_h^*) \Big\rangle_{\mathcal{A}}\Big|s_1,b_1^*\Big]\\
            \leq&
            \sum_{h=1}^H\mathbb{E}_{\pi^*}\big[\iota_h(s_h,b_h,a_h)\big|s_1,b_1^*\big]
        .\end{split}
    \end{equation*}
    The last inequality follows from $\widehat{\pi}_h\in\argmax_{\pi_h}\{ \langle  \widehat{Q}_h(s_h,b_h^*,\cdot),\allowbreak \pi_h(\cdot|s_h,b_h^*)\rangle_{\mathcal{A}}\}$, which implies $\langle  \widehat{Q}_h(s_h,b_h^*,\cdot), \pi^*_h(\cdot|s_h,b_h^*)-\widehat{\pi}_h(\cdot|s_h,b_h^*) \rangle_{\mathcal{A}}\leq 0$.
    This completes the proof.
\end{proof}

\begin{lemma}\label{lemma: static-OCE bound iota}
    Under event $\mathcal{E}'_h$, for all $s\in\mathcal{S}$, $a\in\mathcal{A}$, $b\in[0,H]$, and $h\in[H]$, we have
    \begin{equation*}
        0\leq\iota_h(s,b,a)\leq2\Gamma_h(s,b,a).
    \end{equation*}
\end{lemma}

\begin{proof}
     If $\overline{Q}_h(s,b,a)<u(-b)$, then, by the definition of $\widehat{Q}_h(s,b,a)$ in Algorithm~\ref{alg: static-OCE-PVI}, we obtain
    \begin{equation*}
        \widehat{Q}_h(s,b,a)=\max\Big\{\min\big\{\overline{Q}_h(s,b,a),u(H-h+1-b)\big\},u(-b)\Big\}=u(-b).
    \end{equation*}
    Furthermore, this leads to
    \begin{equation*}
        \iota_h(s,b,a)=\mathbb{U}_h\widehat{V}_{h+1}(s,b,a)-u(-b)\geq0,
    \end{equation*}
    where the last inequality follows from $\widehat{V}_{h+1}(s',b-r)\geq u(-b)$ for all feasible $(s',r,b)$. 
    If $\overline{Q}_h(s,b,a)\geq u(-b)$, we have
    \begin{equation*}
        \widehat{Q}_h(s,b,a)=
        \max\Big\{\min\big\{\overline{Q}_h(s,b,a),u(H-h+1-b)\big\},u(-b)\Big\}
        \leq \overline{Q}_h(s,b,a).
    \end{equation*}
    Then, we have
    \begin{equation*}
        \begin{split}
            \iota_h(s,b,a)
            \geq& \mathbb{U}_h\widehat{V}_{h+1}(s,b,a)
            -\Big(\widehat{\mathbb{U}}_h\widehat{V}_{h+1}(s,b,a)-\Gamma_h(s,b,a)\Big)
            \geq 0,
        \end{split}
    \end{equation*}
    where the second inequality follows from the definition of $\mathcal{E}'_h$. It remains to prove the upper bound. On the event $\mathcal{E}'_h$, the triangle inequality gives
    \begin{equation*}
        \widehat{\mathbb{U}}_h\widehat{V}_{h+1}(s,b,a)-\Gamma_h(s,b,a)\leq\mathbb{U}_h\widehat{V}_{h+1}(s,b,a).
    \end{equation*}
    This leads to 
    \begin{equation*}
        \begin{split}
            \overline{Q}_h(s,b,a)&=\widehat{\mathbb{U}}_h\widehat{V}_{h+1}(s,b,a)-\Gamma_h(s,b,a)
            \leq \mathbb{U}_h\widehat{V}_{h+1}(s,b,a)
            \leq u(H-h+1-b),
        \end{split}
    \end{equation*}
    where the last inequality follows from the definition of $\mathbb{U}_h\widehat{V}_{h+1}(s,b,a)$.
    Therefore, we have
    \begin{equation*}
        \begin{split}
            &\widehat{Q}_h(s,b,a)=\max\Big\{\min\big\{\overline{Q}_h(s, b, a),u(H-h+1-b)\big\},u(-b)\Big\}\\
            =&
            \max\Big\{\overline{Q}_h(s, b, a), u(-b)\Big\}
            \geq \overline{Q}_h(s,b,a)
        .\end{split}
    \end{equation*}
    Applying the definition of $\iota_h(s,b,a)$ in \eqref{eq: static-OCE def iota}, we have
    \begin{equation*}
        \begin{split}
            \iota_h(s,b,a)&=\mathbb{U}_h\widehat{V}_{h+1}(s,b,a)-\widehat{Q}_h(s,b,a)
            \leq \mathbb{U}_h\widehat{V}_{h+1}(s,b,a)-\overline{Q}_h(s,b,a)\\
            &=\mathbb{U}_h\widehat{V}_{h+1}(s,b,a)-\widehat{\mathbb{U}}_h\widehat{V}_{h+1}(s,b,a)+\Gamma_h(s,b,a)
            \leq 2\Gamma_h(s,b,a)
        .\end{split}
    \end{equation*}
    Thus, on the event $\mathcal{E}'_h$, $0\leq\iota_h(s,b,a)\leq 2\Gamma_h(s,b,a)$.
\end{proof}
The definition of static-OCE suboptimality and the event $\mathcal{E}'_h$ yield the following lemma.

\begin{lemma}\label{lemma: static-OCE iota}
Under the static-OCE formulation, with probability at least $1-\delta$,
    \begin{equation*}
        \Big|\mathbb{U}_h\widehat{V}_{h+1}(s,b,a)-\widehat{\mathbb{U}}_h\widehat{V}_{h+1}(s,b,a)\Big|\leq\Gamma_h(s,b,a), \forall h\in [H],
    \end{equation*}
    where $\{\Gamma_h\}_{h=1}^H$ denotes the bonus sequence, 
the suboptimality of Algorithm~\ref{alg: static-OCE tabular OCE-PVI} and Algorithm~\ref{alg: static-OCE-PVI} is bounded by
    \begin{equation*}
        \SubOptS(\widehat{\pi})\leq \sum_{h=1}^H\mathbb{E}_{\pi^*}\Big[\iota_h(s_h,b_h,a_h)\Big|s_1,b_1^*\Big],
    \end{equation*}
where $\mathbb{E}_{\pi^*}$ is with respect to the trajectory generated by $\pi^*_h, \forall h\in[H]$. 
\end{lemma}

Lemma~\ref{lemma: static-OCE iota} relates the suboptimality to the Bellman estimation error $\mathbb{U}_h\widehat{V}_{h+1}(s,b,a)-\widehat{\mathbb{U}}_h\widehat{V}_{h+1}(s,b,a)$ and the bonus $\Gamma_h$. It applies to both the tabular and linear-function-approximation settings.

\begin{proof}
    Let $b_1^*\in\argmax_{b\in[0,H]}\{ b+V_1^*(s_1,b)\}$. Then
\begin{equation*}
    \begin{split}
        \SubOptS(\widehat{\pi})
        &=\sup_{b\in[0, H]}\Big\{ b+V_1^*(s_1,b)\Big\}-\sup_{b\in[0, H]}\Big\{b+V_1^{\widehat{\pi}}(s_1,b)\Big\}\\
        &\leq\Big\{b_1^*+V^*_1(s_1,b_1^*)\Big\}-\Big\{b_1^*+V_1^{\widehat{\pi}}(s_1,b_1^*)\Big\}\\
        &=V^*_1(s_1,b_1^*)-\widehat{V}_1(s_1,b_1^*)+\widehat{V}_1(s_1,b_1^*)-V_1^{\widehat{\pi}}(s_1,b_1^*).
    \end{split}
\end{equation*}
Applying Lemma~\ref{lemma: cai 2020 static-OCE} with $\pi_h=\widehat{\pi}_h$ for every $h\in[H]$ gives
\begin{equation*}
    \begin{split}
        \widehat{V}_1(s_1,b_1^*)-V_1^{\widehat{\pi}}(s_1,b_1^*)=-\sum_{h=1}^H\mathbb{E}_{\widehat{\pi}}\Big[\iota_h(s_h,b_h,a_h)\Big|s_1,b_1^*\Big]
    .\end{split}
\end{equation*}
Applying Lemmas~\ref{lemma: static-OCE term 1} and~\ref{lemma: static-OCE bound iota} on $\mathcal{E}'_h$ at every step yields
\begin{equation*}
    \begin{split}
        \SubOptS(\widehat{\pi})&\leq V_1^*(s_1)-V_1^{\widehat{\pi}}(s_1)\\
        &\leq
        \sum_{h=1}^H\mathbb{E}_{\pi^*}\Big[\iota_h(s_h,b_h,a_h)\Big|s_1,b_1^*\Big]-\sum_{h=1}^H\mathbb{E}_{\widehat{\pi}}\Big[\iota_h(s_h,b_h,a_h)\Big|s_1,b_1^*\Big]\\
        &\leq \sum_{h=1}^H\mathbb{E}_{\pi^*}\Big[\iota_h(s_h,b_h,a_h)\Big|s_1,b_1^*\Big]
    .\end{split}
\end{equation*}
This completes the proof of Lemma \ref{lemma: static-OCE iota}.
\end{proof}

\begin{lemma}\label{lemma: static-OCE xu lemma 4}
    For any $\delta\in(0,1)$, any fixed $(s,b,a)\in\mathcal{S}\times[0,H]\times\mathcal{A}$, and $h\in[H]$, with probability at least $1-\delta$, we have
    \begin{equation*}
        \mathbb{U}_h\widehat{V}_{h+1}(s,b,a)-\widehat{\mathbb{U}}_h\widehat{V}_{h+1}(s,b,a)\leq 2HL_u\sqrt{\frac{2\log\frac{|\mathcal{S}||\mathcal{A}|HK}{\delta}}{\max\{1,N_h(s,a)\}}}.
    \end{equation*}
\end{lemma}

\begin{proof}
    When $N_h(s,a)=0$, $\widehat{\mathbb{P}}_h(\cdot|s,a)=0$, we have 
    \begin{equation*}
        \begin{split}
        \mathbb{U}_h\widehat{V}_{h+1}(s,b,a)-\widehat{\mathbb{U}}_h\widehat{V}_{h+1}(s,b,a)
        =
        \mathbb{U}_h\widehat{V}_{h+1}(s,b,a)
        \leq
        HL_u\sqrt{\frac{2\log\frac{|\mathcal{S}||\mathcal{A}|HK}{\delta}}{\max\{1,N_h(s,a)\}}},
        \end{split}
    \end{equation*}
    where the first inequality uses the uniform upper bound on $\widehat{V}_{h}(s,b)$, and the last uses $\log\frac{|\mathcal{S}||\mathcal{A}|HK}{\delta}>1$.
    When $N_h(s,a)\geq 1$, we have 
    \begin{equation*}
            \widehat{\mathbb{U}}_h\widehat{V}_{h+1}(s,b,a)=\frac{1}{N_h(s,a)}\sum_{k=1}^K\mathbb{I}\Big((s_h^k,a_h^k)=(s,a)\Big)\widehat{V}_{h+1}(s_{h+1}^k, b).
    \end{equation*}
    Setting $Y_i=\mathbb{E}[\mathbb{I}\big((s_h^k,a_h^k)=(s,a)\big)\widehat{V}_{h+1}(s_{h+1}^k, b)]-\mathbb{I}\big((s_h^k,a_h^k)=(s,a)\big)\widehat{V}_{h+1}(s_{h+1}^k, b)$, we have
    \begin{equation*}
        |Y_i|\leq 2HL_u,
    \end{equation*}
    where we use $|\widehat{V}_{h+1}(s,b)|\leq HL_u$.
    For any $i\neq j$, $Y_i$ and $Y_j$ are independent. Hence, Hoeffding's inequality gives, with probability at least $1-\frac{\delta}{|\mathcal{S}||\mathcal{A}|HK}$,
    \begin{equation*}
        \begin{split}
            &\sum_{i=1}^KY_i
            =
            N_h(s,a)\mathbb{U}_h\widehat{V}_{h+1}(s,b,a)-\sum_{k=1}^K\mathbb{I}\Big((s_h^k,a_h^k)=(s,a)\Big)u\Big(\widehat{V}_{h+1}(s_{h+1}^k)-b^*\Big)\\
            =&
            N_h(s,a)\mathbb{U}_h\widehat{V}_{h+1}(s,b,a)-N_h(s,a)\widehat{\mathbb{U}}_h\widehat{V}_{h+1}(s,b,a)
            \leq
            2HL_u\sqrt{2N_h(s,a)\log\frac{|\mathcal{S}||\mathcal{A}|HK}{\delta}}
        .\end{split}
    \end{equation*}
    Therefore, we have
    \begin{equation*}
        \mathbb{U}_h\widehat{V}_{h+1}(s,b,a)-\widehat{\mathbb{U}}_h\widehat{V}_{h+1}(s,b,a)\leq
        2HL_u\sqrt{\frac{2\log\frac{|\mathcal{S}||\mathcal{A}|HK}{\delta}}{\max\{1,N_h(s,a)\}}}.
    \end{equation*}
    This completes the proof of Lemma \ref{lemma: static-OCE xu lemma 4}.
\end{proof}

\subsection{Proof of Theorem \ref{the: static-OCE tabular result}}\label{subsec: proof of static-OCE tabular result}
We first analyze the idealized static-OCE recursion with a continuous budget variable $b\in[0,H]$. The proof below is therefore independent of the budget grid. The implementable grid recursion is obtained afterwards by applying Lemma~\ref{lemma_subopt_static-OCE_discrete}, which contributes at most $(1+L_u)\varsigma$ to the suboptimality gap. For notational convenience, we write the Bellman operator under the joint transition--reward law as $\mathbb U_hf(s,b,a)=\mathbb E_{s',r}[f(s',b-r)]$.
By Lemma~\ref{lemma: static-OCE iota}, it suffices to bound $\mathbb{U}_h\widehat{V}_{h+1}(s,b,a)-\widehat{\mathbb{U}}_h\widehat{V}_{h+1}(s,b,a)=\iota_h(s,b,a)-\Gamma_h(s,b,a)$. The notation $\mathbb{U}_h\widehat{V}_{h+1}(s,b,a)$ and $\widehat{\mathbb{U}}_h\widehat{V}_{h+1}(s,b,a)$ is used throughout for the corresponding population and empirical expectations.
By Lemma~\ref{lemma: static-OCE xu lemma 4}, with probability at least $1-\delta$,
\begin{equation*}
    \mathbb{U}_h\widehat{V}_{h+1}(s,b,a)-\widehat{\mathbb{U}}_h\widehat{V}_{h+1}(s,b,a)\leq 2HL_u\sqrt{\frac{2\log\frac{|\mathcal{S}||\mathcal{A}|HK}{\delta}}{\max\{1,N_h(s,a)\}}}.
\end{equation*}
Combining this result with Lemma~\ref{lemma: static-OCE iota} gives
\begin{equation*}
    \begin{split}
        \SubOptS(\widehat{\pi})
        \leq&
        4\sum_{h=1}^HHL_u\mathbb{E}_{\pi^*}\Big[\sqrt{\frac{2\log\frac{|\mathcal{S}||\mathcal{A}|HK}{\delta}}{\max\{1,N_h(s,a)\}}}\Big|s_1,b_1^*\Big]\\
        \leq&
        4\sum_{h=1}^HHL_u\sum_{s,a}d_h^{\pi^*}(s,a)\sqrt{\frac{2\log\frac{|\mathcal{S}||\mathcal{A}|HK}{\delta}}{\max\{1,N_h(s,a)\}}}\\
        =&
        4\sum_{h=1}^HHL_u\sum_{s,a}\sqrt{d_h^{\pi^*}(s,a)}\sqrt{\frac{2d_h^{\pi^*}(s,a)\log\frac{|\mathcal{S}||\mathcal{A}|HK}{\delta}}{Kd_h^{\mu}(s,a)}},
    \end{split}
\end{equation*}    
where $b_h$ is calculated by $b_{h+1}=b_h-r_h$ starting from $b_1=b_1^*$.
Due to $\frac{d_h^{\pi^*}(s,a)}{d_h^{\mu}(s,a)}\leq C^*$, we have
\begin{equation*}
    \begin{split}
        \SubOptS(\widehat{\pi})
        \leq&
        4\sum_{h=1}^HHL_u\sum_{s,a}\sqrt{d_h^{\pi^*}(s,a)}\sqrt{\frac{2C^*\log\frac{|\mathcal{S}||\mathcal{A}|HK}{\delta}}{K}}\\
        =&
        4\sum_{h=1}^HHL_u\sum_{s,a}\sqrt{d_h^{\pi^*}(s,a)\cdot\mathbb{I}(a=a^*_{h,s})}\sqrt{\frac{2C^*\log\frac{|\mathcal{S}||\mathcal{A}|HK}{\delta}}{K}}\\
        \leq&
        4\sum_{h=1}^HHL_u\sqrt{\sum_{s,a}d_h^{\pi^*}(s,a)\cdot\sum_{s,a}\mathbb{I}(a=a^*_{h,s})}\sqrt{\frac{2C^*\log\frac{|\mathcal{S}||\mathcal{A}|HK}{\delta}}{K}}\\
        \leq&
        4\sum_{h=1}^HHL_u\sqrt{\frac{2C^*S\log\frac{|\mathcal{S}||\mathcal{A}|HK}{\delta}}{K}},
    \end{split}
\end{equation*}    
where $a^*_{h,s}$ is sampled by $a^*_{h,s}\sim\pi^*_h(\cdot|s), \forall h\in[H], s\in\mathcal{S}$.
This proves the continuous-budget claim of Theorem~\ref{the: static-OCE tabular result}. Lemma~\ref{lemma_subopt_static-OCE_discrete} then gives the stated extension to the grid implementation, with the additional $(1+L_u)\varsigma$ term.

\section{Proofs for Section \ref{sec: Information-Theoretic Lower Bounds of Tabular Offline RL with OCE}}\label{sec: lower bounds}
In this section, we present the proof of Theorem~\ref{the: lower bound}. The key step is to construct a hard MDP instance, denoted by $\mathcal{M}^\dag$. Let $C=\lfloor C^*\rfloor$. Under the assumption $|\mathcal{A}|\geq C^*\geq2$, we have $2\leq C\leq|\mathcal{A}|$ and $C\geq C^*/2$. Moreover, $K>\frac14C^*H|\mathcal{S}|$ implies $K>\frac14CH|\mathcal{S}|$. These conditions ensure that the hard instance contains at least two candidate actions under the behavior policy and that the transition gap chosen below satisfies $\tau<p$. We construct $\mathcal{M}^\dag$ with $|\mathcal{S}|+2$ states, $|\mathcal{A}|$ actions, and horizon $H$. Specifically, let $\mathcal{S}=\{s_1,s_2,\dots,s_{|\mathcal{S}|}\}$ denote the set of regular states, and introduce two additional absorbing states: the ``good'' state $s^\sharp$ and the ``bad'' state $s^\diamond$. The resulting augmented state space is $\mathcal{S}^\dag:=\mathcal{S}\cup\{s^\sharp,s^\diamond\}$.

For each $(h,i)$ with $h\in[H]$ and $i\in[|\mathcal{S}|]$, let $\mathcal{M}_{a^*}$ be an MDP with the same structure as $\mathcal{M}^\dag$ and a unique optimal action $a_{h,i}^*\in\mathcal{A}$ at stage $h$ in state $s_i\in\mathcal{S}$. The data are sampled uniformly, so $\mu_h(a_{h,i}^*\mid s_i)=\frac{1}{C}$. The transition dynamics from $s_i$ in $\mathcal{M}_{a^*}$ are
\begin{equation}\label{eq: lower bound transition 1}
        \begin{cases}
            \mathbb{P}_h(s_i|s_i,a)=1-2p, & \forall a\in \mathcal{A}\\
            \mathbb{P}_h(s^\sharp|s_i,a)=\mathbb{P}_h(s^\diamond|s_i,a)=p, & \forall a \neq a_{h,i}^*\\
            \mathbb{P}_h(s^\sharp|s_i,a_{h,i}^*)=p+\tau, & \forall h\in [1,H]\\
            \mathbb{P}_h(s^\diamond|s_i,a_{h,i}^*)=p-\tau, & \forall h\in [1,H],
        \end{cases}
\end{equation}
where $p\in(0,\frac{1}{2})$ and $\tau\in(0,p)$ are the parameters yet to be determined. When $h\leq\overline{H}$, all states are absorbing, so the good state $s^\sharp$ and the bad state $s^\diamond$ remain unreachable. For any $\overline{H}\in[1,H]$ and $a\in \mathcal{A}$, the transitions from the absorbing states and the reward function are defined as follows:

\noindent 
\begin{equation}\label{eq: lower bound transition larger than 1}
        \begin{cases}
            \mathbb{P}_h(s^\sharp|s^\sharp,a)=1, & \forall a\in \mathcal{A}\\
            \mathbb{P}_h(s^\diamond|s^\diamond,a)=1, & \forall a\in \mathcal{A}\\
            \mathbb{P}_h(s_i|s^\sharp,a)=0, & \forall i\in S, a\in \mathcal{A} \\
            \mathbb{P}_h(s_i|s^\diamond,a)=0, & \forall i\in S, a\in \mathcal{A},
        \end{cases}
\qquad \qquad 
        \begin{cases}
            r_h(s_i,a_{h,i})=0, &\forall s_i\\
            r_h(s^\sharp,a_{h,i})=0, &\forall h\in[1, \overline{H}]\\
            r_h(s^\sharp,a_{h,i})=1, &\forall h\in (\overline{H}, H]\\
            r_h(s^\diamond,a_{h,i})=0, &\forall h\in [1, H].
        \end{cases}
\end{equation}
Figure~\ref{fig: lower bound transition} illustrates these transition dynamics for each state $s_i$.

\begin{figure}[H]
\centering
\includegraphics[width=0.6\textwidth]{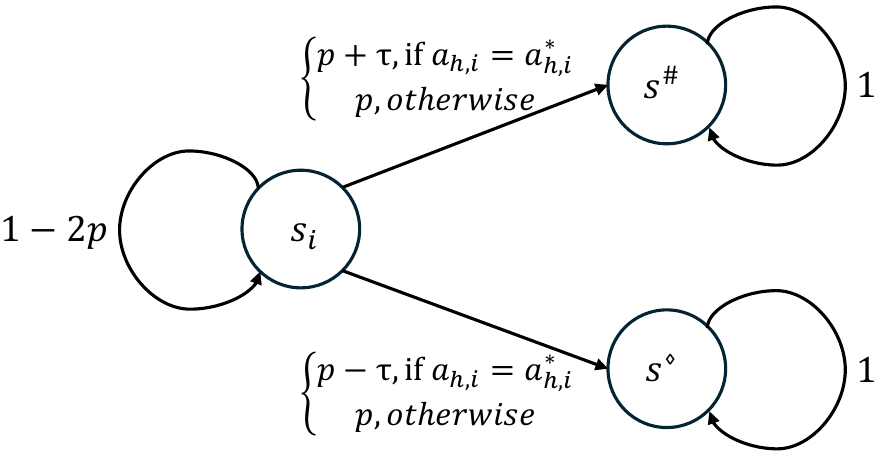}
\caption{Transition Dynamics of $\mathcal{M}_{a^*}$.}
\label{fig: lower bound transition}
\end{figure}

To incorporate as many hard instances as possible, we require the MDP $\mathcal{M}^\dag$ to be sufficiently general. Our construction of $\mathcal{M}^\dag$ is concise and embeds several constructions from the literature. For example, with a single regular state (i.e., $|\mathcal{S}|=1$ and $|\mathcal{S}^\dag|=3$), we recover the construction of \citet{Jin2020}. Prepending a ``tree'' structure in front of the absorbing states instead of the regular states yields that of \citet{Xu2023}. Setting the horizon to $2H$ and $\overline{H}=H$ recovers the hard instance of \citet{Xie2021}. Thus, the hard-instance MDP $\mathcal{M}^\dag$ is flexible enough to unify these cases. Moreover, both dynamic-OCE-PVI and static-OCE-PVI share the same information-theoretic lower bound on the hard instance $\mathcal{M}_{a^*}$. Intuitively, Proposition~\ref{sec: properties of OCE} implies that, in the absorbing-state setup, dynamic-OCE-PVI evaluates the ``true'' OCE only once at a deterministic time step. A formal argument is given in the proof of Lemma~\ref{lemma: same lower bound}. Finally, we define the ``null'' MDP $\mathcal{M}_0$, which has the same structure as $\mathcal{M}_{a^*}$ and transition dynamics $\mathbb{P}_h(s^\sharp|s_i,a)=\mathbb{P}_h(s^\diamond|s_i,a)=p$ for all $h$ and $a \in \mathcal{A}$.

\subsection{Lemmas for Theorem \ref{the: lower bound}}\label{subsec: Lemmas for static-OCE tabular lower bound}
Under the constructed MDP $\mathcal{M}^\dag$, we establish Lemma~\ref{lemma: same lower bound}.
\begin{lemma}\label{lemma: same lower bound}
    Under the constructed hard MDP $\mathcal{M}^\dag=(\mathcal{S}^\dag,\mathcal{A},H,\mathbb{P},r)$, where $\mathcal{S}^\dag=\{s_1,s_2,\dots,\allowbreak s_{|\mathcal{S}|},s^\sharp,s^\diamond\}$, $\mathcal{A}=\{a_1,a_2,\dots,a_{|\mathcal{A}|}\}$, $\mathbb{P}$ is defined in \eqref{eq: lower bound transition 1} and \eqref{eq: lower bound transition larger than 1}, and $r_h(s,a)$ is defined in \eqref{eq: lower bound transition larger than 1}, the suboptimality gaps of Algorithm~\ref{alg:recursiv OCE-PVI} and Algorithm~\ref{alg: static-OCE-PVI} have the same form.
\end{lemma}

\begin{proof}
    We consider the dynamic-OCE and static-OCE settings under the hard MDP of Lemma~\ref{lemma: same lower bound}.

    In the dynamic-OCE setting, $V_{H+1}^\pi(s)=0$ for every constant policy $\pi$ and state $s$. Suppose that the process transitions to $s^\sharp$ or $s^\diamond$ at stage $\overline{h}\in[H]$, and let $h^*=\max\{\overline{h},\overline{H}\}$. By Property~\ref{prop:utility},
    \begin{equation*}
        V_{h^*+1}^\pi(s)=
            \begin{cases}
                 H-h^*, &\text{ w.p. } p+\tau\\
                 0, &\text{ w.p. } 1-p-\tau.
            \end{cases}
    \end{equation*}
    Consequently,
    \begin{equation*}
        \begin{split}
            V_1^\pi(s_i)
            &=
            V_{h^*}^\pi(s)
            =
            \OCE^u_{s \sim \mathbb{P}_{h^*}(\cdot|s_{h^*},a_{h^*})}(V_{h^*+1}^\pi(s))
            \\
            &=
            \sup_{b \in [0,H-h^*]}\big\{b + pu(H-h^*-b)+(1-p)u(-b)\big\}
        .\end{split}
    \end{equation*}
    Therefore, we obtain
    \begin{equation*}
        \begin{split}
            &\SubOptD(\mathcal{M}_{a^*}, Alg(\mathcal{D}))\\
            =&
            \sup_{b \in [0,H-h^*]}\big\{b + (p+\tau)u(H-h^*-b)+(1-p-\tau)u(-b)\big\}\\
            &-\sup_{b \in [0,H-h^*]}\big\{b + pu(H-h^*-b)+(1-p)u(-b)\big\}\\
            \geq&
            \sup_{b \in [0,H-h^*]}\big\{b + (p+\tau)u(H-h^*-b)+(1-p-\tau)u(-b)\big\} \\
            &-\sup_{b \in [0,H-h^*]}\big\{b + pu(H-\overline{H}-b)+(1-p)u(-b)\big\}\\
            \geq&
            \sum_{h=1}^{H}\sum_{i=1}^{|\mathcal{S}|}d_h^{\widehat{\pi}}(s_i)\tau \big[u(H-\overline{H}-b_1^*)-u(-b_1^*)\big]
            \cdot\mathbb{I}\big\{\widehat{a}_{h,i}\neq a^*_{h,i}\big\},
        \end{split}
    \end{equation*}
    where $\widehat{a}_{h, i} \sim \widehat{\pi}_h(\cdot | s_i)$ denotes the action sampled from the stochastic policy obtained by the algorithm.
    For the static-OCE setting, if $s_{h^*+1}=s^\sharp$, we have
    \begin{equation*}
        \begin{split}
            &V^{\pi}_{H+1}(s^\sharp,b)=u(-b), \text{ for any } s,b\\
            &V^{\pi}_{H}(s^\sharp,b_H)=u(1-b)\\
            &\dots\\
            &V^{\pi}_{h^*+1}(s^\sharp,b_{h^*+1})=u(H-h^*-b)
        .\end{split}
    \end{equation*}
    If $s_{h^*+1}\neq s^\sharp$, then
    \begin{equation*}
        V^{\pi}_{h^*+1}(s^\diamond,b_{h^*+1})=u(-b).
    \end{equation*}
    Consequently,
    \begin{equation*}
        \begin{split}
            V_1^\pi(s_i,b_1) 
            = V_{h^*}^\pi(s,b_{h^*})
            = \mathbb{E}[V_{h^*+1}^\pi(s,b_{h^*+1})]
            = pu(H-h^*-b)+(1-p)u(-b)
        .\end{split}
    \end{equation*}
    Therefore,
    \begin{equation*}
        \begin{split}
            \OCE_u^{\pi}=\sup_{b \in [0,H]}\big\{b + V^{\pi}_1(s_i,b)\big\}
            =\sup_{b \in [0,H]}\big\{b + pu(H-h^*-b)+(1-p)u(-b)\big\}
        .\end{split}
    \end{equation*} 
    Thus, on the constructed hard MDP, $V_1^\pi(s_i)$ in the dynamic-OCE setting and $\OCE_u^{\pi}$ in the static-OCE setting have the same form and induce the same form of suboptimality. Likewise, the two settings share the following form of $V_1^*(s_i)$:
    \begin{equation*}
            V_1^*(s_i) 
            =\sup_{b \in [0,H-h^*]}\big\{b + (p+\tau)u(H-h^*-b)+(1-p-\tau)u(-b)\big\}.
    \end{equation*}
    Therefore, with $\widehat{a}_{h,i}\sim\widehat{\pi}_h(\cdot\mid s_i)$,
    \begin{equation*}
        \begin{split}
            &\SubOptS(\mathcal{M}_{a^*}, Alg(\mathcal{D}))\\
            =&
            \sup_{b \in [0,H-h^*]}\big\{b + (p+\tau)u(H-h^*-b)+(1-p-\tau)u(-b)\big\}\\
            &-\sup_{b \in [0,H-h^*]}\big\{b + pu(H-h^*-b)+(1-p)u(-b)\big\}\\
            \geq&
            \sup_{b \in [0,H-h^*]}\big\{b + (p+\tau)u(H-h^*-b)+(1-p-\tau)u(-b)\big\}\\
            &-\sup_{b \in [0,H-h^*]}\big\{b + pu(H-\overline{H}-b)+(1-p)u(-b)\big\}\\
            \geq&
            \sum_{h=1}^{H}\sum_{i=1}^{|\mathcal{S}|}d_h^{\widehat{\pi}}(s_i)\tau \big[u(H-\overline{H}-b_1^*)-u(-b_1^*)\big]
            \cdot\mathbb{I}\big\{\widehat{a}_{h,i}\neq a^*_{h,i}\big\},
        \end{split}
    \end{equation*}
    where the first inequality holds due to the fact that $u(\cdot)$ is non-decreasing and $h^*\geq\overline{H}$, and last inequality holds when setting $b_1^*=\argmax_{b \in [0,H-\overline{H}]}\big\{b + pu(H-\overline{H}-b)+(1-p)u(-b)\big\}$. 
    Therefore, we define
    \begin{equation*}
        \SubOpt(\mathcal{M}_{a^*}, Alg(\mathcal{D}))=\SubOptD(\mathcal{M}_{a^*}, Alg(\mathcal{D}))=\SubOptS(\mathcal{M}_{a^*}, Alg(\mathcal{D}))
    \end{equation*}
    to be the shared form of suboptimality of Algorithm \ref{alg: recursiv tabular OCE-PVI} and Algorithm \ref{alg: static-OCE tabular OCE-PVI} under the MDP $\mathcal{M}_{a^*}$.
    This concludes the proof.
\end{proof}

\begin{lemma}\label{lemma: lower bound C*}
    For any $[a_{i,h}]\in \{1,2,\dots,K\}^{H|\mathcal{S}|}$, in MDP $\mathcal{M}_{a^*}$, we have
    \begin{equation*}
        \sup_{h,s,a}\frac{d_h^{\pi^*}(s,a)}{d_h^{\mu}(s,a)}<C\leq C^*,
    \end{equation*}
    where $\pi^*$ is the optimal policy for the MDP $\mathcal{M}_{a^*}$. 
\end{lemma}

\begin{proof}
    By the definitions of $d_h^{\pi^*}(s_i)$ and $d_h^{\mu}(s_i)$,
    \begin{equation*}
        d_h^{\pi^*}(s_i)=d_h^{\mu}(s_i)=\frac{1}{|\mathcal{S}|}(1-2p)^{h-1}.
    \end{equation*}
    Moreover, $\pi_h^*(a_{h,i}^*\mid s_i)=1$, and uniform data collection gives $\mu_h(a_{h,i}^*\mid s_i)=\frac{1}{C}$. Hence,
    \begin{equation*}
        \frac{d_h^{\pi^*}(s_i,a_{h,i}^*)}{d_h^{\mu}(s_i,a_{h,i}^*)}=\frac{d_h^{\pi^*}(s_i)\pi_h^*(a_{h,i}^*|s_i)}{d_h^{\mu}(s_i)\mu_h(a_{h,i}^*|s_i)}=C.
    \end{equation*}
    Then we consider the good state $s_g$ and the bad state $s_b$. There is
    \begin{equation*}
        \begin{split}
            d_h^{\pi^*}(s_g)
            =\sum_{S}^{i}\sum_{j}^{h-1}\frac{1}{S}(1-2p)^{j-1}(p+\tau)
            =\sum_{j}^{h-1}(1-2p)^{j-1}(p+\tau)
        .\end{split}
    \end{equation*}
    The underlying policy $\mu$ takes the action $a_{h,i}^*$ with probability $\frac{1}{C}$, then we have
    \begin{equation*}
        \begin{split}
            d_h^{\mu}(s_g)
            =  
            \sum_{j}^{h-1}(1-2p)^{j-1}\Big[\frac{1}{C}(p+\tau)+(1-\frac{1}{C})p\Big]
            =
            \sum_{j}^{h-1}(1-2p)^{j-1}\Big[p+\frac{\tau}{C}\Big]
        .\end{split}
    \end{equation*}
    Consequently,
    \begin{equation*}
        \begin{split}
            \frac{d_h^{\pi^*}(s_g)}{d_h^{\mu}(s_g)}=\frac{p+\tau}{p+\frac{\tau}{C}}
            =C\frac{p+\tau}{Cp+\tau}< C,
        \end{split}
    \end{equation*}
    where the last inequality holds since $\frac{p+\tau}{Cp+\tau}\leq 1$.
    Therefore for all the states including $s_i$, $s_g$ and $s_b$, we have
    \begin{equation*}
        \sup_{h,s,a}\frac{d_h^{\pi^*}(s,a)}{d_h^{\mu}(s,a)}<C\leq C^*.
    \end{equation*}
    This completes the proof of Lemma \ref{lemma: lower bound C*}.
\end{proof}

\subsection{Proof of Theorem \ref{the: lower bound}}\label{subsec: proof of static-OCE tabular lower bound}
For the hard MDP $\mathcal{M}$ in Theorem~\ref{the: lower bound}, the suboptimality of Algorithm~\ref{alg:recursiv OCE-PVI} and Algorithm~\ref{alg: static-OCE-PVI} satisfies
\begin{equation*}
    \begin{split}
        &\SubOpt \big(\mathcal{M}_{a^*}, Alg(\mathcal{D})\big)\\
        \geq&
        \sum_{h=1}^{H}\sum_{i=1}^{|\mathcal{S}|}d_h^{\widehat{\pi}}(s_i)\tau \big[u(H-\overline{H}-b_1^*)-u(-b_1^*)\big]\cdot\mathbb{I}\big\{\widehat{a}_{h,i}\neq a^*_{h,i}\big\}\\
        =&   
        \sum_{h=1}^{H}\sum_{i=1}^{|\mathcal{S}|}\frac{1}{|\mathcal{S}|}\big(1-2p\big)^{h-1}\tau \big[u(H-\overline{H}-b_1^*)-u(-b_1^*)\big]\cdot\mathbb{I}\big\{\widehat{a}_{h,i}\neq a^*_{h,i}\big\},
    \end{split}
\end{equation*}
where $\widehat{a}_{h,i}\sim\widehat{\pi}_h(\cdot\mid s_i)$, and the final equality follows from the definition of $d_h^{\widehat{\pi}}(s_i)$. Consequently,
\begin{equation*}
    \begin{split}
        &\max_{\mathcal{M}}\SubOpt\big(\mathcal{M}, Alg(\mathcal{D})\big)
        \geq
        \mathbb{E}_{\mathcal{M}}\SubOpt\big(\mathcal{M}, Alg(\mathcal{D})\big)\\
        \geq &
        \frac{1}{C}\sum_{a_{h,i}^*=1}^{C}\mathbb{E}_{\mathcal{M}_{a^*}}\big[\SubOpt\big(\mathcal{M}_{a^*}, Alg(\mathcal{D})\big)\big]\\
        =&
        \frac{1}{C}\sum_{a_{h,i}^*=1}^{C}\mathbb{E}_{\mathcal{M}_{a^*}}\bigg[\sum_{h=1}^{H}\sum_{i=1}^{|\mathcal{S}|}\frac{1}{|\mathcal{S}|}\big(1-2p\big)^{h-1}\tau \big[u(H-\overline{H}-b_1^*)-u(-b_1^*)\big]
        \cdot\mathbb{I}\big\{\widehat{a}_{h,i}\neq a^*_{h,i}\big\}\bigg]\\
        =&
        \big(1-2p\big)^{H-1}\tau \big[u(H-\overline{H}-b_1^*)-u(-b_1^*)\big]
        \bigg(H-\frac{1}{CK|\mathcal{S}|}\sum_{h=1}^{H}\sum_{i=1}^{|\mathcal{S}|}\sum_{a_{h,i}^*=1}^{C}\mathbb{E}_{\mathcal{M}_{a^*}}\big[N_h(s_i,a^*_{h,i})\big]\bigg)
    .\end{split}
\end{equation*}
It remains to bound $\sum_{h=1}^{H}\sum_{i=1}^{|\mathcal{S}|}\sum_{a_{h,i}^*=1}^{C}\frac{1}{K}\mathbb{E}_{\mathcal{M}_{a^*}}\big[N_h(s_i,a^*_{h,i})\big]$, where $N_h(s_i,a^*_{h,i})$ is the number of times action $a^*_{h,i}$ is selected at stage $h$ in state $s_i$. Comparing $\mathbb{E}_{\mathcal{M}_{a^*}}\big[N_h(s_i,a^*_{h,i})\big]$ with $\mathbb{E}_{\mathcal{M}_0}\big[N_h(s_i,a^*_{h,i})\big]$ yields
\begin{equation*}
    \begin{split}
        &\frac{1}{K}\mathbb{E}_{\mathcal{M}_{a^*}}\big[N_h(s_i,a^*_{h,i})\big]-\frac{1}{K}\mathbb{E}_{\mathcal{M}_0}\big[N_h(s_i,a^*_{h,i})\big]\\
        \leq &
        TV\big(\mathbb{P}_{\mathcal{M}_{a^*}}, \mathbb{P}_{\mathcal{M}_0}\big)
        \leq
        \sqrt{\frac{1}{2}KL\big(\mathbb{P}_{\mathcal{M}_0}, \mathbb{P}_{\mathcal{M}_{a^*}}\big)}
        =
        \sqrt{\frac{1}{2}p\log\frac{p^2}{p^2-\tau^2}\mathbb{E}_{\mathcal{M}_0}\big[N_h(s_i,a^*_{h,i})\big]},
    \end{split}
\end{equation*}
where the second inequality follows from Pinsker's inequality. Consequently,
\begin{equation*}
    \begin{split}
        &\frac{1}{K}\sum_{h=1}^{H}\sum_{i=1}^{|\mathcal{S}|}\sum_{a_{h,i}^*=1}^{C}\mathbb{E}_{\mathcal{M}_{a^*}}\big[N_h(s_i,a^*_{h,i})\big]\\
        \leq& 
        \sum_{h=1}^{H}\sum_{i=1}^{|\mathcal{S}|}\sum_{a_{h,i}^*=1}^{C}\Bigg\{\frac{1}{K}\mathbb{E}_{\mathcal{M}_0}\big[N_h(s_i,a^*_{h,i})\big]+\sqrt{\frac{1}{2}p\log\frac{p^2}{p^2-\tau^2}\mathbb{E}_{\mathcal{M}_0}\big[N_h(s_i,a^*_{h,i})\big]}\Bigg\}\\
        =& 
        H+\sum_{h=1}^{H}\sum_{i=1}^{|\mathcal{S}|}\sum_{a_{h,i}^*=1}^{C}\sqrt{\frac{1}{2}p\log\frac{p^2}{p^2-\tau^2}\mathbb{E}_{\mathcal{M}_0}\big[N_h(s_i,a^*_{h,i})\big]}
        \leq
        H+H\sqrt{\frac{1}{2}p\log\frac{p^2}{p^2-\tau^2}CK|\mathcal{S}|},
    \end{split}
\end{equation*}
where the first equation holds since $\sum_{h=1}^{H}\sum_{i=1}^{|\mathcal{S}|}\sum_{a_{h,i}^*=1}^{C}\mathbb{E}_{\mathcal{M}_0}\big[N_h(s_i,a^*_{h,i})\big]\leq HK$ and the last inequality holds because of Cauchy-Schwarz inequality. 
Therefore, we have
\begin{equation*}
    \begin{split}
        &\max_{\mathcal{M}}\SubOpt\big(\mathcal{M}, Alg(\mathcal{D})\big)\\
        \geq&
        H\big(1-2p\big)^{H-1}\tau \big[u(H-\overline{H}-b_1^*)-u(-b_1^*)\big]
        \cdot\Bigg(1-\frac{1}{C|\mathcal{S}|}
        -\frac{1}{C|\mathcal{S}|}\sqrt{\frac{1}{2}p\log\frac{p^2}{p^2-\tau^2}CK|\mathcal{S}|}\Bigg)\\
        \geq&
        H\big(1-2p\big)^{H-1}\tau \big[u(H-\overline{H}-b_1^*)-u(-b_1^*)\big]\cdot\Bigg(\frac{1}{2}
        -\frac{1}{C|\mathcal{S}|}\sqrt{\frac{\tau^2}{2p}CK|\mathcal{S}|}\Bigg),
    \end{split}
\end{equation*}
where the last inequality holds since $\frac{1}{C|\mathcal{S}|}\leq\frac{1}{2}$ and $\log\frac{p^2}{p^2-\tau^2}\leq\frac{\tau^2}{p^2}$.
Since $\rho H\in\mathbb{N}$, set $p=\frac{1}{2H}$ and $\overline{H}=(1-\rho)H$, so that $H-\overline{H}=\rho H$. Then we have
\begin{equation*}
    \begin{split}
        &\max_{\mathcal{M}}\SubOpt\big(\mathcal{M}, Alg(\mathcal{D})\big)\\
        \geq&
        H\Big(1-\frac{1}{H}\Big)^{H-1}\tau \big[u\big(\rho H-b_1^*\big)-u(-b_1^*)\big]\Big(\frac{1}{2}
        -\frac{1}{C|\mathcal{S}|}\sqrt{\tau^2CHK|\mathcal{S}|}\Big)\\
        \geq&
        \frac{1}{3}H\tau \Big[u\big(\rho H-b_1^*\big)-u(-b_1^*)\Big]\cdot\Big(\frac{1}{2}
        -\sqrt{\frac{\tau^2HK}{C|\mathcal{S}|}}\Big),
    \end{split}
\end{equation*}
where the last inequality holds since $(1-\frac{1}{H})^{H-1}\geq e^{-1}\geq\frac{1}{3}$.
Let $\tau=\sqrt{\frac{C|\mathcal{S}|}{16HK}}<p=\frac{1}{2H}$, we have
\begin{equation*}
    \begin{split}
        \max_{\mathcal{M}}\SubOpt(\mathcal{M}, Alg(\mathcal{D}))
        &\geq
        \frac{1}{48}\sqrt{\frac{CH|\mathcal{S}|}{K}} \Big[u\big(\rho H-b_1^*\big)-u(-b_1^*)\Big]\\
        &\geq
        \frac{1}{48\sqrt{2}}\big[u(\rho H-b_1^*)-u(-b_1^*)\big]\sqrt{\frac{C^*H|\mathcal{S}|}{K}},
    \end{split}
\end{equation*}
where the last inequality follows from $C=\lfloor C^*\rfloor\geq C^*/2$. Here $b_1^*$ is any maximizer in $\argmax_{b\in[0,\rho H]}\{b+\frac{1}{2H}u(\rho H-b)+(1-\frac{1}{2H})u(-b)\}$. The argument requires neither differentiability nor an interior maximizer of $u$.
This concludes the proof of Theorem \ref{the: lower bound}.

\section{Proofs for Section \ref{sec: OCE-RL with Linear Function Approximation}}\label{sec: proof of OCE RL with Linear Function Approximation}

\subsection{Lemmas for Theorem \ref{the: subopt dynamic-OCE}}\label{subsec: Lemmas for subopt dynamic-OCE}
\begin{lemma}\label{lemma: dynamic-OCE bound theta and w}
    Under the dynamic-OCE RL formulation,
    \begin{equation*}
        \begin{split}
            \Big\|\theta_h+w_h(b)\Big\|\leq[1+(H-h)L_u]\sqrt{d}
        .\end{split}
    \end{equation*}
\end{lemma}

\begin{proof}
    By the definition of $w_h(b)$ under the dynamic-OCE RL formulation,
    \begin{equation*}
        \begin{split}
            \Big\|\theta_h+w_h(b)\Big\|
            &\leq
            \Big\|\theta_h\Big\|+\Big\|\int_{\mathcal{S}}u(V(s')-b)\mu_h(s')ds'\Big\| \\
            &\leq
            \sqrt{d}+\int_{\mathcal{S}}\Big\|u(V(s')-b)\mu_h(s')\Big\|ds'
            \leq
            [1+(H-h)L_u]\sqrt{d}
        .\end{split}
    \end{equation*}
    The third inequality follows from $|V(s)-b|\leq H-h$ and $|u(V(s)-b)|\leq (H-h)L_u$. This completes the proof.
\end{proof}

\begin{definition}\label{def: dynamic-OCE function class}
    Define the following parametric class $\mathcal{V}$ of functions from $\mathcal{S}\times[0,H]$ to $\mathbb{R}$:
    \begin{equation*}
        \begin{split}
            V(\cdot)
            =&
            \max_a\Big\{\max\big\{\min\big\{\phi(\cdot,a)^\top \theta+\sup_{b\in[0,H-h]}\big\{b+\phi(\cdot,a)^\top w(b)\big\}\\
            &-
            \beta\sqrt{\phi(\cdot,a)^\top \Lambda^{-1}\phi(\cdot,a)},H-h+1\big\},0\big\}\Big\},
        \end{split}
    \end{equation*}
    where $b\in[0,H-h]$ is a parameter, $\|\theta\|\leq T$, $\|w(b)\|\leq L$, $\beta \in [0,B]$ and $\Lambda\succeq \lambda I$. 
\end{definition}

\begin{lemma}\label{lemma: dynamic-OCE bound hat}
    Under the dynamic-OCE RL setting, we have
    \begin{equation*}
        \begin{split}
            \Big\|\widehat{w}_h(b)\Big\|\leq (H-h)L_u\sqrt{\frac{dK}{\lambda}}, \quad
            \Big\|\widehat{\theta}_h\Big\|\leq\sqrt{\frac{dK}{\lambda}}
        .\end{split}
    \end{equation*}
\end{lemma}

\begin{proof}
    By the definition of $\widehat{w}_h(b)$,
    \begin{equation*}
        \begin{split}
            &\Big\|\widehat{w}_h(b)\Big\|\\
            =&
            \Big\|\Lambda_h^{-1}\big\{\sum_{k=1}^K\phi(s_h^k, a_h^k)u\Big(\widehat{V}_{h+1}(s_{h+1}^k)-b\Big)\big\}\Big\|
            \leq\sum_{k=1}^K\Big\|\Lambda_h^{-1}\phi(s_h^k, a_h^k)u\Big(\widehat{V}_{h+1}(s_{h+1}^k)-b\Big)\Big\|\\
            \leq& 
            (H-h)L_u\sum_{k=1}^K\Big\|\Lambda_h^{-1}\phi(s_h^k, a_h^k)\Big\|
            =(H-h)L_u\sum_{k=1}^K\sqrt{\phi(s_h^k, a_h^k)^\top \Lambda_h^{-\frac{1}{2}}\Lambda_h^{-1}\Lambda_h^{-\frac{1}{2}}\phi(s_h^k, a_h^k)}
        .\end{split}
    \end{equation*}
    The Cauchy--Schwarz inequality and the trajectory properties yield
    \begin{equation*}
        \begin{split}
            &\Big\|\widehat{w}_h(b)\Big\| \\
            \leq &
            (H-h)L_u\sqrt{\frac{K}{\lambda}}\sqrt{\tr\Big(\Lambda_h^{-1}\sum_{k=1}^K\phi(s_h^k, a_h^k)^\top \phi(s_h^k, a_h^k)\Big)}
            = (H-h)L_u\sqrt{\frac{K}{\lambda}}\sqrt{\tr\Big(\Lambda_h^{-1}(\Lambda_h-\lambda I)\Big)}\\
            \leq & (H-h)L_u\sqrt{\frac{K}{\lambda}}\sqrt{\tr\Big(\Lambda_h^{-1}\Lambda_h\Big)}
            = (H-h)L_u\sqrt{\frac{dK}{\lambda}}
        .\end{split}
    \end{equation*}
    The same argument, using $|r_h|\leq1$, gives $\big\|\widehat{\theta}_h\big\|\leq\sqrt{\frac{dK}{\lambda}}$. This completes the proof.
\end{proof}

\begin{lemma}\label{lemma: dynamic-OCE Jinying B2}
      Under the dynamic-OCE formulation, let $f^h:\mathcal{S}\to[0,H-h+1]$ be fixed at stage $h\in[H]$, and suppose that $\mathcal{D}$ is generated by a behavior policy $\mu$. Then, for any $\Delta\in(0,1)$ and $(s_h^k,a_h^k)$ in $\mathcal{D}$,

      \begin{equation*}
        \mathbb{P}\bigg(\Big\|\sum_{K=1}^K\phi(s_h^k,a_h^k)\epsilon_h^k(V_h)\Big\|_{\Lambda_h^{-1}}^2>[1+(H-h)L_u]^2\Big(2\log\frac{1}{\Delta}+d\log\big(1+\frac{K}{\lambda}\big)\Big)\bigg)\leq\Delta.
      \end{equation*}
\end{lemma}

\begin{proof}
    For any fixed $h\in[H]$ and $k\in\{0,1,\dots, K\}$, we have the $\sigma$-algebra
    \begin{equation*}
        \mathcal{F}_h^k=\sigma
        \bigg(
            \big\{(s_h^j,a_h^j)\big\}_{j=1}^{\min\{k+1,K\}}
            \cup 
            \big\{(r_h^j,s_{h+1}^j)\big\}_{j=1}^k
        \bigg).
    \end{equation*}
    For any $k\in[K],$
    \(
    \phi(s_h^k,a_h^k)\in\mathcal{F}_h^k
    \), since $(s_h^k,a_h^k)$ is measurable with respect to $\mathcal{F}_h^{k-1}$.
    Since $(r_h^j,s_{h+1}^j)$ is measurable with respect to $\mathcal{F}_h^k$, for a fixed $f^h:\mathcal{S}\to[0,H-h+1]$ and $k\in[K],$
    \begin{align*}
        \epsilon_h^k(f^h)
        =
        r_h^k+u(f^h(s_{h+1}^k)-b_h)-\mathbb{B}_{h}f^h(s_h^k,a_h^k)
        \in \mathcal{F}_h^k.
    \end{align*}
    Therefore, $\{\epsilon_h^k(f^h)\}_{k=1}^K$ is a stochastic process adapted to the filtration $\{\mathcal{F}_h^k\}_{k=0}^K$. Assumption~\ref{assumption: Markov property of dataset} gives
    \begin{align*}
        \mathbb{E}\Big[\epsilon_h^k(f^h)| \mathcal{F}_h^{k-1}\Big]
        =
        \mathbb{E}
        \Big[
            r_h^k+u(f^h(s_{h+1}^k)-b_h)| 
            \{(s_h^j,a_h^j)\}_{j=1}^{k}, (r_h^j,s_{h+1}^j)_{j=1}^{k}
        \Big]
        -\mathbb{B}_{h}f^h(s_h^k,a_h^k)
        =0.
    \end{align*}
    By the definition of $\epsilon_h^k(f^h)$, $|\epsilon_h^k(f^h)|\leq 1+(H-h)L_u$. Thus, for fixed $h$ and every $k\in[K]$, $\epsilon_h^k(f^h)$ is conditionally zero-mean and $[1+(H-h)L_u]$-sub-Gaussian given $\mathcal{F}_h^{k-1}$. Applying Lemma~\ref{lemma: Jinying B2} with $M_0=\lambda I$ and $M_k=\lambda I+\frac{1}{K}\sum_{j=1}^K\phi(s_h^j,a_h^j)\phi(s_h^j,a_h^j)^\top$ gives, for all $\Delta\in(0,1)$ and $(s_h^k,a_h^k)$ in $\mathcal{D}$,
    \begin{equation*}
        \mathbb{P}\bigg(\Big\|\sum_{k=1}^K\phi(s_h^k,a_h^k)\epsilon_h^k(f^h)\Big\|_{\Lambda_h^{-1}}^2>2[1+(H-h)L_u]^2\log\Big(\frac{\det(\Lambda_h)^{\frac{1}{2}}}{\Delta\cdot\det(\lambda I)^{\frac{1}{2}}}\Big)\bigg)\leq\Delta,
    \end{equation*}
    where the equality uses $M_K=\Lambda_h$.
    Since $\|\Lambda_h\|_2\leq\lambda+K$, we have $\det(\Lambda_h)\leq(\lambda+K)^d$. Therefore,
    \begin{align*}
        &\mathbb{P}\bigg(\Big\|\sum_{K=1}^K\phi(s_h^k,a_h^k)\epsilon_h^k(f^h)\Big\|_{\Lambda_h^{-1}}^2>[1+(H-h)L_u]^2\Big(2\log\frac{1}{\Delta}+d\log\big(1+\frac{K}{\lambda}\big)\Big)\bigg)\\
        \leq&
        \mathbb{P}\bigg(\Big\|\sum_{k=1}^K\phi(s_h^k,a_h^k)\epsilon_h^k(V_h)\Big\|_{\Lambda_h^{-1}}^2>2[1+(H-h)L_u]^2\log\Big(\frac{\det(\Lambda_h)^{\frac{1}{2}}}{\Delta\cdot\det(\lambda I)^{\frac{1}{2}}}\Big)\bigg)
        \leq
        \Delta.
    \end{align*}
    This completes the proof.
\end{proof}

\begin{lemma}\label{lemma: dynamic-OCE e covering number}
    By Definition~\ref{def: dynamic-OCE function class}, for all $h\in[H]$ and $\varepsilon>0$,
    \begin{equation*}
        \log\mathcal{N}_h(\varepsilon)\leq d\log\Big(1+\frac{4T}{\varepsilon}\Big)+d\log\Big(1+\frac{4L}{\varepsilon}\Big)+d^2\log\Big(1+\frac{8\sqrt{d}B^2}{\lambda\varepsilon^2}\Big).
    \end{equation*}    
\end{lemma}

\begin{proof}
    For the function class $\mathcal{V}$, set $A=\beta^2\Lambda^{-1}$. By definition, $\|\theta\|\leq T$, $\|w(b)\|\leq L$, and $\|A\|\leq \frac{B^2}{\lambda}$. For any $V_1,V_2\in\mathcal{V}$,
    \begin{equation*}
        \begin{split}
            \dist(V_1, V_2) 
            \leq&
            \sup_{s,a}\Bigg|\Big[\phi(\cdot,a)^\top \theta_1+\sup_{b\in[0,H-h]}\Big\{b+\phi(\cdot,a)^\top w_1(b)\Big\}-\sqrt{\phi(\cdot,a)^\top A_1\phi(\cdot,a)}\Big]\\
            &-\Big[\phi(\cdot,a)^\top \theta_2+\sup_{b\in[0,H-h]}\Big\{b+\phi(\cdot,a)^\top w_2(b)\Big\}-\sqrt{\phi(\cdot,a)^\top A_2\phi(\cdot,a)}\Big]\Bigg|\\
            \leq&
            \sup_{s,a}\Bigg|\Big[\phi(\cdot,a)^\top \theta_1+\Big\{b^\dag+\phi(\cdot,a)^\top w_1(b)\Big\}-\sqrt{\phi(\cdot,a)^\top A_1\phi(\cdot,a)}\Big]\\
            &-\Big[\phi(\cdot,a)^\top \theta_2+\Big\{b^\dag+\phi(\cdot,a)^\top w_2(b)\Big\}-\sqrt{\phi(\cdot,a)^\top A_2\phi(\cdot,a)}\Big]\Bigg|\\
            \leq&
            \sup_{\phi:\|\phi\|\leq1, \theta:\|\theta\|\leq T, w:\|w\|\leq L}
            \Bigg|
            \Big[(\theta_1+w_1)\phi-\sqrt{\phi^\top A_1\phi}\Big]-
            \Big[(\theta_2+w_2)\phi-\sqrt{\phi^\top A_2\phi}\Big]
            \Bigg|,   
        \end{split}
    \end{equation*}
    where the second inequality holds by setting $b^\dag=\argmax_{b\in[0,H-h]}\{b+\phi(\cdot,a)^\top w_1(b)\}$. Since $|\sqrt{x}-\sqrt{y}|\leq\sqrt{|x-y|}$, for $x>0, y>0$, we have
    \begin{equation*}\label{eq: dynamic-OCE cover num main}
        \begin{split}
            &\dist(V_1, V_2) 
            \leq
            \sup_{\phi:\|\phi\|\leq1, \theta:\|\theta\|\leq T, w:\|w\|\leq L}
            \big|(\theta_1-\theta_2+w_1-w_2)\phi\big|-
            \Big|\sqrt{\phi^\top (A_1-A_2)\phi}\Big|\\
            =&
            \Big\|\theta_1-\theta_2\Big\|+\Big\|w_1-w_2\Big\|+\sqrt{\Big\|A_1-A_2\Big\|_2}
            \leq
            \Big\|\theta_1-\theta_2\Big\|+\Big\|w_1-w_2\Big\|+\sqrt{\Big\|A_1-A_2\Big\|_{F}}            
        .\end{split}
    \end{equation*}
    Let $\mathcal{C}_\theta$ be an $\frac{\varepsilon}{2}$-cover of $\{\theta\in \mathbb{R}^d \quad\vert \quad \|w\|\leq T\}$ with respect to the 2-norm $\mathcal{C}_w$ be an $\frac{\varepsilon}{2}$-cover of $\{w\in \mathbb{R}^d\quad\vert \quad \|w\|\leq L\}$ with respect to the 2-norm, and $\mathcal{C}_A$ be an $\frac{\varepsilon^2}{4}$-cover of $\{A\in \mathbb{R}^{d\times d}\quad\vert \quad \|A\|_F\leq \sqrt{d}B^2\lambda^{-1}\}$ with respect to the Frobenius norm. By Lemma \ref{lemma: Covering Number of Euclidean Ball}, we have
    \begin{equation*}
        \begin{split}
            |\mathcal{C}_\theta|\leq\Big(1+\frac{4T}{\varepsilon}\Big)^d,\quad
            |\mathcal{C}_w|\leq\Big(1+\frac{4L}{\varepsilon}\Big)^d,\quad
            |\mathcal{C}_A|\leq\Big(1+\frac{8\sqrt{d}B^2}{\lambda\varepsilon^2}\Big)^{d^2}
        .\end{split}
    \end{equation*}
    By \eqref{eq: dynamic-OCE cover num main}, for any $V_1\in\mathcal{V}$, there exist $\theta_2\in\mathcal{C}_\theta$, $w_2\in\mathcal{C}_w$, and $A_2\in\mathcal{C}_A$ such that the function $V_2$ parameterized by $(\theta_2,w_2,A_2)$ satisfies $\dist(V_1,V_2)\leq\varepsilon$. Therefore, $\mathcal{N}(\varepsilon)\leq|\mathcal{C}_\theta
    |\cdot|\mathcal{C}_w
    |\cdot|\mathcal{C}_A|$. Then, we can conclude that
    \begin{equation*}
        \begin{split}
            \log\mathcal{N}_h(\varepsilon)
            &\leq
            \log|\mathcal{C}_\theta|+\log|\mathcal{C}_w|+\log|\mathcal{C}_A|\\
            &\leq
            d\log\Big(1+\frac{4T}{\varepsilon}\Big)+d\log\Big(1+\frac{4L}{\varepsilon}\Big)+d^2\log\Big(1+\frac{8\sqrt{d}B^2}{\lambda\varepsilon^2}\Big)
        .\end{split}
    \end{equation*}
    This completes the proof.
\end{proof}


\subsection{Proof of Theorem \ref{the: subopt dynamic-OCE}}\label{subsec: proof of subopt dynamic-OCE}
By Lemma~\ref{lemma: dynamic-OCE iota}, it suffices to bound $\mathbb{B}_h\widehat{V}_{h+1}(s,a)-\widehat{\mathbb{B}}_h\widehat{V}_{h+1}(s,a)$, which can be written as
\begin{align*}
        &\mathbb{B}_h\widehat{V}_{h+1}(s,a)-\widehat{\mathbb{B}}_h\widehat{V}_{h+1}(s,a)\\
        =&
        r_h(s,a)-\widehat{r}_h(s,a)+\OCE_{s'\sim\mathbb{P}_h(\cdot|s,a)}^u\Big\{\widehat{V}_{h+1}(s_{h+1})\Big\}-\OCE_{s'\sim\widehat{\mathbb{P}}_h(\cdot|s,a)}^u\Big\{\widehat{V}_{h+1}(s_{h+1})\Big\}\\
        =&
        r_h(s,a)-\widehat{r}_h(s,a)
        +\max_{b\in[0,H-h]}\Big\{b+\mathbb{E}_{s'\sim\mathbb{P}_h(\cdot|s,a)}\Big[u\Big(\widehat{V}_{h+1}(s_{h+1})-b\Big)\Big]\Big\}\\
        &-\max_{b\in[0,H-h]}\Big\{b+\mathbb{E}_{s'\sim\widehat{\mathbb{P}}_h(\cdot|s,a)}\Big[u\Big(\widehat{V}_{h+1}(s_{h+1})-b\Big)\Big]\Big\}\\
        =&
        \phi(s,a)^\top \theta_h-\phi(s,a)^\top \widehat{\theta}_h
        +\max_{b\in[0,H-h]}\Big\{b+\phi(s,a)^\top w_h(b)\Big\}
        -\max_{b\in[0,H-h]}\Big\{b+\phi(s,a)^\top \widehat{w}_h(b)\Big\}.
\end{align*}
Let $b_h\in\argmax_{b\in[0,H-h]}\{b+\phi(s,a)^\top \widehat{w}_h(b)\}$. Then
\begin{align}\label{eq_dynamic_OCE_aim}
\begin{aligned}
    &\mathbb{B}_h\widehat{V}_{h+1}(s,a)-\widehat{\mathbb{B}}_h\widehat{V}_{h+1}(s,a)\\
    \leq&
    \phi(s,a)^\top \theta_h
    -\phi(s,a)^\top \widehat{\theta}_h
    +\phi(s,a)^\top w_h(b_h)
    -\phi(s,a)^\top \widehat{w}_h(b_h)\\
    =&
    \phi(s,a)^\top \big(\theta_h+w_h(b_h)\big)
    -\phi(s,a)^\top \Bigg[\Lambda_h^{-1}\sum_{k=1}^K\phi(s_h^k,a_h^k)\Bigg(r_h^k+u\Big(\widehat{V}_{h+1}(s_{h+1}^k)-b_h\Big)\Bigg)\Bigg]\\
    =&\text{(I)} + \text{(II)}, 
    \end{aligned}
\end{align}
where we denote
\begin{align*}
    \text{(I)}
    :=& 
    \phi(s,a)^\top \big(\theta_h+w_h(b_h)\big)\\
    &-\phi(s,a)^\top \Bigg[\Lambda_h^{-1}\sum_{k=1}^K\phi(s_h^k,a_h^k)\Bigg(
        r_h^k(s_h^k,a_h^k)+\mathbb{E}_{s'\sim\mathbb{P}_h(\cdot\mid s,a)}\Big[u\Big(\widehat{V}_{h+1}(s_{h+1}^k)-b_h\Big)\Big]
    \Bigg)\Bigg], \\ 
    \text{(II)}
    :=& 
    -\phi(s,a)^\top \Bigg[\Lambda_h^{-1}\sum_{k=1}^K\phi(s_h^k,a_h^k)\Bigg(
        \Big\{r_h^k(s_h^k,a_h^k)+u\Big(\widehat{V}_{h+1}(s_{h+1}^k)-  b_h\Big)\Big\}\\
        &-\Big\{r_h^k(s_h^k,a_h^k)+\mathbb{E}_{s'\sim\mathbb{P}_h(\cdot\mid s,a)}\Big[u\Big(\widehat{V}_{h+1}(s_{h+1}^k)-b_h\Big)\Big]\Big\}
    \Bigg)\Bigg]. 
\end{align*}
For term \text{(I)}, we have
\begin{equation*}
    \begin{split}
        \big|\text{(I)}\big|
        =&
        \Big|\phi(s,a)^\top \big(\theta_h+w_h(b_h)\big)
        -\phi(s,a)^\top \Lambda_h^{-1}\sum_{k=1}^K\phi(s_h^k,a_h^k)\phi(s_h^k,a_h^k)^\top \big(\theta_h+w_h(b_h)\big)\Big|\\
        =&
         \Big|\phi(s,a)^\top \big(\theta_h+w_h(b_h)\big)
        -\phi(s,a)^\top \Lambda_h^{-1}\big(\Lambda_h-\lambda I\big)\big(\theta_h+w_h(b_h)\big)\Big|\\
        =&\lambda\Big|\phi(s,a)^\top \Lambda_h^{-1}(\theta_h+w_h(b_h))\Big|
    .\end{split}
\end{equation*}
The Cauchy--Schwarz inequality further gives
\begin{equation}\label{eq: dynamic-OCE term 1}
    \begin{split}
        |\text{(I)}|
        \leq&
        \lambda\cdot\Big\|\Lambda_h^{-1}\Big\|_2^{\frac{1}{2}}\cdot\Big\|\theta_h+w_h(b_h)\Big\|\cdot\Big\|\phi(s,a)\Big\|_{\Lambda_h^{-1}}\\
        \leq&
        \big(1+(H-h)L_u\big)\sqrt{d\lambda}\sqrt{\phi(s,a)^\top \Lambda_h^{-1}\phi(s,a)},
    \end{split}
\end{equation}
where the last inequality holds due to Lemma \ref{lemma: dynamic-OCE bound theta and w}. 
It remains to bound term \text{(II)}. Define
\begin{equation*}
    \epsilon_h^k(f):=\Big\{r_h^k+u\big(f(s_{h+1}^k)-b_h\big)\Big\}-\Big\{r_h(s_h^k,a_h^k)+\mathbb{E}_{s'\sim\mathbb{P}_h(\cdot\mid s_h^k,a_h^k)}\big[u\big(f(s')-b_h\big)\big]\Big\},
\end{equation*}
where $f:\mathcal{S}\to[0,f_{\max}]$ is an arbitrary function.
Then
\begin{equation*}
    \begin{split}
        |\text{(II)}|
        =&
        \Big|\phi(s,a)\Lambda_h^{-1}\Big(\sum_{k=1}^K\phi(s_h^k,a_h^k)\epsilon_h^k\big(\widehat{V}_{h+1}\big)\Big)\Big|\\
        \leq&
        \Big\|\sum_{k=1}^K\phi(s_h^k,a_h^k)\epsilon_h^k(\widehat{V}_{h+1})\Big\|_{\Lambda_h^{-1}}\sqrt{\phi(s,a)^\top \Lambda_h^{-1}\phi(s,a)}
    .\end{split}
\end{equation*}
It remains to bound $\big\|\sum_{k=1}^K\phi(s_h^k,a_h^k)\epsilon_h^k(\widehat{V}_{h+1})\big\|_{\Lambda_h^{-1}}$. By Definition~\ref{def: dynamic-OCE function class}, $\widehat{V}_{h+1}\in\mathcal{V}$. Let $\mathcal{N}_{h+1}(\varepsilon)$ be an $\varepsilon$-cover of $\mathcal{V}$. Then there exists $V'_{h+1}\in\mathcal{N}_{h+1}(\varepsilon)$ such that
\begin{equation*}
        \sup_{s\in\mathcal{S}}\big|\widehat{V}_{h+1}(s)-V'_{h+1}(s)\big|\leq\varepsilon.
\end{equation*}

Therefore, we have
\begin{equation*}
    \begin{split}
        \big|u(\widehat{V}_{h+1}(s)-b_h)-u(V'_{h+1}(s)-b_h)\big|
        \leq
        L_u\big|\widehat{V}_{h+1}(s)-V'_{h+1}(s)\big|
        \leq L_u\varepsilon
    .\end{split}
\end{equation*}
The first inequality follows from Assumption~\ref{assump: Lipschitz utility}, because both utility arguments belong to $[-H,H]$. Then, we obtain
\begin{equation*}
    \begin{split}
        &\big|\mathbb{E}_{s'\sim\mathbb{P}_h(\cdot|s,a)}\big[u(\widehat{V}_{h+1}(s')-b_h)\big]-\mathbb{E}_{s'\sim\mathbb{P}_h(\cdot|s,a)}\big[u(V'_{h+1}(s')-b_h)\big]\big|\\
        \leq&
        \big|\mathbb{E}_{s'\sim\mathbb{P}_h(\cdot|s,a)}\big[u(\widehat{V}_{h+1}(s')-b_h)-u(V'_{h+1}(s')-b_h)\big]\big|
        \leq L_u\varepsilon
    .\end{split}
\end{equation*}
Therefore, by the triangle inequality, we have
\begin{equation*}
    \begin{split}
        2L_u\varepsilon
        \geq
        &\big|\big(u(\widehat{V}_{h+1}(s)-b_h)-\mathbb{E}_{s'\sim\mathbb{P}_h(\cdot|s,a)}\big[u(\widehat{V}_{h+1}(s)-b_h)\big]\big)
        \\
        &-\big(u(V'_{h+1}(s)-b_h)-\mathbb{E}_{s'\sim\mathbb{P}_h(\cdot|s,a)}\big[u(V'_{h+1}(s)-b_h)\big]\big)\big|
    .\end{split}
\end{equation*}
This can further guarantee that
\begin{equation*}
    \big|\epsilon_h^k(\widehat{V})-\epsilon_h^k(V')\big|\leq 2L_u\varepsilon.
\end{equation*}
Then based on the fact that $\|a+b\|_{\Lambda_h^{-1}}^2\leq2\|a\|_{\Lambda_h^{-1}}^2+2\|b\|_{\Lambda_h^{-1}}^2$, we can then get
\begin{align}
        &\Big\|\sum_{k=1}^K\phi(s_h^k, a_h^k)\epsilon_h^k(\widehat{V}_{h+1})\Big\|_{\Lambda_h^{-1}}^2 \nonumber\\
        \leq&
        2\Big\|\sum_{k=1}^K\phi(s_h^k, a_h^k)\epsilon_h^k(V'_{h+1})\Big\|_{\Lambda_h^{-1}}^2
        +2\Big\|\sum_{k=1}^K\phi(s_h^k, a_h^k)\big[\epsilon_h^k(\widehat{V}_{h+1})-\epsilon_h^k(V'_{h+1})\big]\Big\|_{\Lambda_h^{-1}}^2\nonumber\\
        \leq&2\sup_{V\in\mathcal{N}_{h+1}(\varepsilon)}\Big\|\sum_{k=1}^K\phi(s_h^k, a_h^k)\epsilon_h^k(V)\Big\|_{\Lambda_h^{-1}}^2+\frac{8L_u^2\varepsilon^2K^2}{\lambda}. \label{equation: dynamic-OCE++}
\end{align}
Here we can get an upper bound without the influence of the dataset $\mathcal{D}$. Combining Lemma \ref{lemma: dynamic-OCE Jinying B2} and the union bound, for $(s_h^k,a_h^k)$ in dataset $\mathcal{D}$, we have that with probability at least $1-\Delta|\mathcal{N}_{h+1}(\varepsilon)|$, 
\begin{equation*}
    \begin{split}
        \sup_{V\in\mathcal{N}_{h+1}(\varepsilon)}\Big\|\sum_{K=1}^K\phi(s_h^k,a_h^k)\epsilon_h^k(V)\Big\|_{\Lambda_h^{-1}}^2
        \leq [1+(H-h)L_u]^2\Big(2\log\frac{1}{\Delta}+d\log\Big(1+\frac{K}{\lambda}\Big)\Big).
        \end{split}
\end{equation*}
Letting $\Delta=\frac{\delta}{H\big|\mathcal{N}_{h+1}(\varepsilon)\big|}$, for $\delta\in(0,1)$, then at any step $h$, with probability $1-\frac{\delta}{H}$, we have
\begin{equation*}
    \sup_{V\in\mathcal{N}_{h+1}(\varepsilon)}\Big\|\sum_{K=1}^K\phi(s_h^k,a_h^k)\epsilon_h^k(V)\Big\|_{\Lambda_h^{-1}}^2\leq [1+(H-h)L_u]^2\Bigg(2\log\frac{H\big|\mathcal{N}_{h+1}(\varepsilon)\big|}{\delta}+d\log\Big(1+\frac{K}{\lambda}\Big)\Bigg).
\end{equation*}
Combining this bound with \eqref{equation: dynamic-OCE++}, with probability at least $1-\delta$, for every $h\in[H]$,
\begin{align*}
&\Big\|\sum_{K=1}^K\phi(s_h^k,a_h^k)\epsilon_h^k(\widehat{V}_{h+1})\Big\|_{\Lambda_h^{-1}}^2\\   
        \leq&
        2[1+(H-h)L_u]^2\Big(2\log\frac{H|\mathcal{N}_{h+1}(\varepsilon)|}{\delta}+d\log(1+\frac{K}{\lambda})\Big)+\frac{8L_u^2\varepsilon^2K^2}{\lambda}, \forall h\in[H]
    .
\end{align*}
We next bound $\log(\mathcal{N}_{h+1}(\varepsilon))$. By Lemma \ref{lemma: dynamic-OCE bound hat}, set $T=\sqrt{dK/\lambda}$ and $L=(H-h)L_u\sqrt{dK/\lambda}$. Let $\varepsilon=d/(L_uK)$, $B=2\beta$, $\beta=cd[1+(H-h)L_u]\sqrt{\zeta}$, and $\zeta=\log(2dHKL_u(1+HL_u)\delta^{-1})$, where $c>0$ is a constant. Assumption~\ref{assump: Lipschitz utility} gives the required bound on $L$ directly.
By Lemma \ref{lemma: dynamic-OCE e covering number} and the choices above, $\log\mathcal{N}_h(\varepsilon)$ is bounded by a constant multiple of $d^2\zeta$. In particular, the definition of $\zeta$ controls the logarithms generated by $T$, $L$, and $B$ uniformly over $h\in[H]$.
Consequently, substituting the covering-number bound into the preceding concentration inequality (and using $\lambda=1$) gives
    \begin{equation*}
        \Big\|\sum_{K=1}^K\phi(s_h^k,a_h^k)\epsilon_h^k\big(\widehat{V}_{h+1}\big)\Big\|_{\Lambda_h^{-1}}^2
        \leq d^2[1+(H-h)L_u]^2\zeta\big(12\log(64c^2)+46\big).
    \end{equation*}
    By setting $12\log(64c^2)+46\leq\frac{c^2}{4}$, the following inequality holds,
    \begin{equation}\label{eq: dynamic-OCE parameter of term 2}
        \begin{split}
        \Big\|\sum_{K=1}^K\phi(s_h^k,a_h^k)\epsilon_h^k\big(\widehat{V}_{h+1}\big)\Big\|_{\Lambda_h^{-1}}\leq\frac{1}{2}cd[1+(H-h)L_u]\sqrt{\zeta}=\frac{\beta}{2}
        .\end{split}
    \end{equation}
    Therefore, based on \eqref{eq_dynamic_OCE_aim}, \eqref{eq: dynamic-OCE term 1}, \eqref{eq: dynamic-OCE parameter of term 2}, we have
    \begin{equation*}
        \begin{split}
            &\Big|\mathbb{B}_h\widehat{V}_{h+1}(s,a)-\widehat{\mathbb{B}}_h\widehat{V}_{h+1}(s,a)\Big|\\
            \leq& \Big(\big[1+(H-h)L_u\big]\sqrt{d}+\frac{1}{2}cd[1+(H-h)L_u]\sqrt{\zeta}\Big)\sqrt{\phi(s,a)^\top \Lambda_h^{-1}\phi(s,a)}\\
            \leq&\beta\sqrt{\phi(s,a)^\top \Lambda_h^{-1}\phi(s,a)}
        .\end{split}
    \end{equation*}
This completes the proof of Theorem \ref{the: subopt dynamic-OCE}.

\subsection{Lemmas for Theorem \ref{the: static-OCE oce result}}\label{subsec: Lemmas for static-OCE oce result}
The main-text static linear algorithm is stated for deterministic rewards. This subsection gives a separate stochastic-reward extension, which requires a joint linear model and an empirical estimator for the pair $(s',r)$. Under this extension, let $r_h\sim\mathcal R_h(\cdot\mid s,a)$ and assume
\begin{align*}
    \mathbb{P}_h(\cdot|s,a)=\big\langle\mu_h(\cdot),\varphi(s,a)\big\rangle,\quad
    \mathcal{R}_h(\cdot|s,a)=\big\langle\nu_h(\cdot),\psi(s,a)\big\rangle.
\end{align*}

Define a matrix $\Phi(s,a)\in\mathbb{R}^{d\times d}$, a vector $\xi_h(s',r)\in\mathbb{R}^{d^2\times 1}$, and a vector $\phi(s,a)\in\mathbb{R}^{d^2\times 1}$ by
\begin{equation*}
    \begin{split}
        \Phi(s,a)=\psi(s,a)\varphi(s,a)^\top,\quad
        \xi_h(s',r)_{i\times d+j}=\big(\nu_h(r)\mu_h(s')^\top \big)_{i,j},\quad
        \phi(s, a)_{i\times d+j}=\Phi(s,a)_{i,j}
    .\end{split}
\end{equation*}
Then
\begin{equation*}
    \begin{split}
        &\mathbb{U}_hV(s,b,a)
        =
        \mathbb{E}_{s'\sim\mathbb{P}_h(\cdot|s,a), r\sim\mathcal{R}_h(\cdot|s,a)}\big[V(s',b-r)\big]\\
        =&
        \int_r\mathcal{R}_h(r\mid s,a)\int_{s'}\mathbb{P}_h(s'\mid s,a)V(s',b-r)\,ds'\,dr\\
        =&
        \nu_h(r)^\top \psi(s,a)\varphi(s,a)^\top \mu_h(s')V(s',b-r)
        =
        \nu_h(r)^\top  \Phi(s,a)\mu_h(s')V(s',b-r)\\
        =&
        \phi(s,a)^\top \xi_h(s',r)V(s',b-r)
        =
        \phi(s,a)^\top \widehat{w}_{h}(b),
    \end{split}
\end{equation*}
where the last equality uses $\widehat{w}_{h}(b)=\xi_h(s',r)V(s',b-r)$. This extends the deterministic-reward quantity $\mathbb{E}_{s'\sim\mathbb{P}_h(\cdot|s,a),,r=r_h(s,a)}\big[V(s',b-r)\big]$ to the stochastic-reward quantity $\mathbb{E}_{s'\sim\mathbb{P}_h(\cdot|s,a),,r\sim\mathcal{R}_h(\cdot|s,a)}\big[V(s',b-r)\big]$. In the latter setting, $\phi(s,a)^\top \widehat{w}_{h}(b)$ still estimates the transition term.

\begin{lemma}\label{lemma: static-OCE bound w}
    By the definitions of $\widehat{w}_h(b)$ and $\widehat{V}_{h+1}(s',b-r)$,
    \begin{equation*}
            \Big\|w_h(b)\Big\|\leq HL_u\sqrt{d},\qquad
            \Big\|\widehat{w}_h(b)\Big\|\leq HL_u\sqrt{\frac{dK}{\lambda}}.
    \end{equation*}
\end{lemma}

\begin{proof}
    The $w_h(b)$ is defined as follows,
    \begin{equation*}
        w_h(b)=\int_r\int_{s'}\xi_h(s',r)\widehat{V}_{h+1}(s',b-r)ds'dr.
    \end{equation*}
    Since $|\widehat{V}_{h+1}(s',b-r)|\leq HL_u$ for all feasible arguments, the definition of the linear model gives $\|w_h(b)\|\leq HL_u\sqrt d$. For $\widehat w_h(b)$, the Cauchy--Schwarz inequality gives
    \begin{equation*}
        \begin{split}
            \Big\|\widehat{w}_h(b)\Big\|
            \leq & HL_u\sqrt{\frac{K}{\lambda}}\sqrt{\sum_{k=1}^K\phi(s_h^k, a_h^k)^\top \Lambda_h^{-1}\phi(s_h^k, a_h^k)}\\
            =& 
            HL_u\sqrt{\frac{K}{\lambda}}\sqrt{\tr\big(\Lambda_h^{-1}\sum_{k=1}^K\phi(s_h^k, a_h^k)^\top \phi(s_h^k, a_h^k)\big)} \\
            = & 
            HL_u\sqrt{\frac{K}{\lambda}}\sqrt{\tr\big(\Lambda_h^{-1}(\Lambda_h-\lambda I)\big)}\\
            \leq& 
            HL_u\sqrt{\frac{K}{\lambda}}\sqrt{\tr\big(\Lambda_h^{-1}\Lambda_h\big)}
            = 
            HL_u\sqrt{\frac{dK}{\lambda}}
        .\end{split}
    \end{equation*}
    Therefore, the proof is complete.
\end{proof}

\begin{lemma}\label{lemma: static-OCE Jinying B2}
      For a fixed function $V_h:\mathcal{S}\times[0,H]\to[-HL_u,HL_u]$ at step $h\in[H]$, under the assumption that $\mathcal{D}$ is obtained by an underlying policy $\mu$, for any $\Delta\in(0,1)$, for $(s_h^k,a_h^k)$ in dataset $\mathcal{D}$, we have

      \begin{equation*}
        \mathbb{P}\bigg(\Big\|\sum_{K=1}^K\phi(s_h^k,a_h^k)\epsilon_h^k(V_h)\Big\|_{\Lambda_h^{-1}}^2>4HL_u^2\Big(2\log\frac{1}{\Delta}+d\log\big(1+\frac{K}{\lambda}\big)\Big)\bigg)\leq\Delta.
      \end{equation*}
\end{lemma}

\begin{proof}
    For any fixed $h\in[H]$ and $k\in\{0,1,\dots, K\}$, we have the $\sigma$-algebra
    \begin{equation*}
        \mathcal{F}_h^k=\sigma\bigg(\big\{(s_h^j,a_h^j)\big\}_{j=1}^{\min\{k+1,K\}}\bigg).
    \end{equation*}
    Then for any $k\in[K]$, we have 
    \(
    \phi(s_h^k,a_h^k)\in\mathcal{F}_h^k
    \), since $(s_h^k,a_h^k)$ is measurable with respect to $\mathcal{F}_h^{k-1}$.
    Then for a fixed function $V_h:\mathcal{S}\times[0,H]\to[-HL_u,HL_u]$ at step $h$, and $k\in[K]$, we have
    \begin{align*}
        \epsilon_h^k(V_h)
        =
        V_h(s_{h+1}^k,b_h-r_h^k)-\mathbb{E}_{s_{h+1}^k\sim\mathbb{P}_h(\cdot|s_h^k,a_h^k),r_h^k\sim \mathcal{R}_h(\cdot|s_h^k,a_h^k)}\big[V_h(s_{h+1}^k,b_h-r_h^k)\big]
        \in \mathcal{F}_h^k.
    \end{align*}
    Therefore, $\{\epsilon_h^k(V_h)\}_{k=1}^K$ is a stochastic process with respect to the filtration $\{\mathcal{F}_h^k\}_{k=0}^K$. Then with Assumption \ref{assumption: Markov property of dataset}, we have
    \begin{align*}
        \mathbb{E}\Big[\epsilon_h^k(V_h)| \mathcal{F}_h^{k-1}\Big]
        =
        \mathbb{E}\Big[V_h(s_{h+1}^k,b_h-r_h^k)| \{(s_h^j,a_h^j)\}_{j=1}^{k}\Big]
        -\mathbb{E}_{s',r}\big[V_h(s_{h+1}^k,b_h-r_h^k)\big]
        =0.
    \end{align*}
    By the definition of $\epsilon_h^k(V_h)$, $|\epsilon_h^k(V_h)|\leq2HL_u$. Thus, for fixed $h$ and every $k\in[K]$, $\epsilon_h^k(V_h)$ is conditionally zero-mean and $2HL_u$-sub-Gaussian given $\mathcal{F}_h^{k-1}$. Applying Lemma~\ref{lemma: Jinying B2} with $M_0=\lambda I$ and $M_k=\lambda I+\frac{1}{K}\sum_{j=1}^K\phi(s_h^j,a_h^j)\phi(s_h^j,a_h^j)^\top$ gives, for all $\Delta\in(0,1)$ and $(s_h^k,a_h^k)$ in $\mathcal{D}$,
    \begin{equation*}
        \mathbb{P}\bigg(\Big\|\sum_{k=1}^K\phi(s_h^k,a_h^k)\epsilon_h^k(V_h)\Big\|_{\Lambda_h^{-1}}^2>4HL_u^2\log\Big(\frac{\det(\Lambda_h)^{\frac{1}{2}}}{\Delta\cdot\det(\lambda I)^{\frac{1}{2}}}\Big)\bigg)\leq\Delta,
    \end{equation*}
    where the equation holds based on the fact that $M_K=\Lambda_h$.
    By applying the definition of $\Lambda_h$, we have $\|\Lambda_h\|_2\leq\lambda+K$ which implies $\det(\Lambda_h)\leq(\lambda+K)^{d}$. Therefore, we can get
    \begin{align*}
        &\mathbb{P}\bigg(\Big\|\sum_{K=1}^K\phi(s_h^k,a_h^k)\epsilon_h^k(V_h)\Big\|_{\Lambda_h^{-1}}^2>2HL_u^2\Big(2\log\frac{1}{\Delta}+d\log\big(1+\frac{K}{\lambda}\big)\Big)\bigg)\\
        \leq&
        \mathbb{P}\bigg(\Big\|\sum_{k=1}^K\phi(s_h^k,a_h^k)\epsilon_h^k(V_h)\Big\|_{\Lambda_h^{-1}}^2>4HL_u^2\log\Big(\frac{\det(\Lambda_h)^{\frac{1}{2}}}{\Delta\cdot\det(\lambda I)^{\frac{1}{2}}}\Big)\bigg)
        \leq
        \Delta.
    \end{align*}
    This completes the proof.
\end{proof}

\begin{definition}\label{def: static-OCE function class}
    Let $\mathcal{V}_{h+1}$ denote the class of functions from $\mathcal{S}\times\mathcal{N}_\varsigma^b$ to $\mathbb{R}$ of the following form:
    \begin{equation*}
        V(\cdot,b)=\max_{a\in\mathcal{A}}\Big\{\max\Big\{\min\big\{\phi(\cdot,a)^\top w(b)-\beta\sqrt{\phi(\cdot,a)^\top \Lambda^{-1}\phi(\cdot,a)},u(H-h-b)\big\},u(-b)\Big\}\Big\},
    \end{equation*}
    where $b\in\mathcal{N}_\varsigma^b$, $\|w(b)\|\leq L$, $\beta\in[0,B]$, and $\Lambda\succeq\lambda I$.
\end{definition}

\begin{lemma}\label{lemma: static-OCE e covering number}
    By Definition~\ref{def: static-OCE function class}, for all $h\in[H]$ and $\varepsilon>0$,
    \begin{equation*}
        \log\mathcal{N}_h(\varepsilon)\leq d\log\Big(1+\frac{4L}{\varepsilon}\Big)+d^2\log\Big(1+\frac{8\sqrt{d}B^2}{\lambda\varepsilon^2}\Big).
    \end{equation*}
\end{lemma}

\begin{proof}
    For the function class $\mathcal{V}$, set $A=\beta^2\Lambda^{-1}$. By definition, $\|w(b)\|\leq L$ and $\|A\|\leq \frac{B^2}{\lambda}$. For any $V_1,V_2\in\mathcal{V}$,
    \begin{equation*}\label{eq: cover num main}
        \begin{split}
            &\dist(V_1, V_2)\\
            \leq&
            \sup_{s,a,b}\Big|\Big[w_1(b)\phi(s,a)-\sqrt{\phi(s,a)^\top A_1\phi(s,a)}\Big]-\Big[w_2(b)\phi(s,a)-\sqrt{\phi(s,a)^\top A_2\phi(s,a)}\Big]\Big|\\
            \leq&
            \sup_{\phi:\|\phi\|\leq1, w:\|w\|\leq L}
            \Big|
            \Big[w_1\phi-\sqrt{\phi^\top A_1\phi}\Big]-
            \Big[w_2\phi-\sqrt{\phi^\top A_2\phi}\Big]
            \Big|\\
            \leq&
            \sup_{\phi:\|\phi\|\leq1, w:\|w\|\leq L}
            \big|(w_1-w_2)\phi\big|-
            \Big|\sqrt{\phi^\top (A_1-A_2)\phi}\Big|\\
            =&\Big\|w_1-w_2\Big\|+\sqrt{\Big\|A_1-A_2\Big\|_2} \leq\Big\|w_1-w_2\Big\|+\sqrt{\Big\|A_1-A_2\Big\|_{F}},            
        \end{split}
    \end{equation*}
    where the third inequality holds due to $|\sqrt{x}-\sqrt{y}|\leq\sqrt{|x-y|}$, for $x>0, y>0$. Let $\mathcal{C}_w$ be an $\frac{\varepsilon}{2}$-cover of $\{w\in \mathbb{R}^d\big| \quad \|w\|\leq L\}$ with respect to the 2-norm, and $\mathcal{C}_A$ be an $\frac{\varepsilon^2}{4}$-cover of $\{A\in \mathbb{R}^{d\times d}\big| \quad \|A\|_F\leq \sqrt{d}B^2\lambda^{-1}\}$ with respect to the Frobenius norm. By Lemma \ref{lemma: Covering Number of Euclidean Ball}, we have
    \begin{equation*}
        \begin{split}
            |\mathcal{C}_w|\leq\Big(1+\frac{4L}{\varepsilon}\Big)^d,\quad
            |\mathcal{C}_A|\leq\Big(1+\frac{8\sqrt{d}B^2}{\lambda\varepsilon^2}\Big)^{d^2}.
        \end{split}
    \end{equation*}
    By \eqref{eq: cover num main}, for any $V_1\in\mathcal{V}$, there exist $w_2\in\mathcal{C}_w$ and $A_2\in\mathcal{C}_A$ such that the function $V_2$ parameterized by $(w_2,A_2)$ satisfies $\dist(V_1,V_2)\leq\varepsilon$. Therefore, $\mathcal{N}(\varepsilon)\leq|\mathcal{C}_w
    |\cdot|\mathcal{C}_A|$, which yields
    \begin{equation*}
        \begin{split}
            \log\mathcal{N}_h(\varepsilon)
            \leq\log|\mathcal{C}_w|+\log|\mathcal{C}_A|
            \leq d\log\Big(1+\frac{4L}{\varepsilon}\Big)+d^2\log\Big(1+\frac{8\sqrt{d}B^2}{\lambda\varepsilon^2}\Big)
        .\end{split}
    \end{equation*}
    This completes the proof.
\end{proof}

\subsection{Proof of Theorem \ref{the: static-OCE oce result}}\label{subsec: pf of static-OCE oce result}
We now prove the separate stochastic-reward extension under the joint transition--reward model introduced above. This proof concerns the estimator based on the joint empirical law of $(s',r)$, rather than the practical deterministic-reward implementation in the main text. We first establish the bound for the associated continuous-budget recursion. Lemma~\ref{lemma_subopt_static-OCE_discrete} subsequently transfers it to the grid implementation. When $\mathcal{R}_h(r_h\mid s_h,a_h)=1$ for every $h$, the argument reduces to the deterministic-reward case of Theorem~\ref{the: static-OCE oce result}.
By Lemma~\ref{lemma: static-OCE iota}, it suffices to bound $\mathbb{U}_h\widehat{V}_{h+1}(s,b,a)-\widehat{\mathbb{U}}_h\widehat{V}_{h+1}(s,b,a)$.
In this section, we use the notation $\mathbb{E}_{s',r}$ to denote the expectation over the transition and reward distribution $\mathbb{E}_{s'\sim\mathbb{P}_h(\cdot|s,a), r\sim\mathcal{R}_h(\cdot|s,a)}$.
We have
    \begin{equation*}
        \begin{split}
            &\mathbb{U}_h\widehat{V}_{h+1}(s,b,a)-\widehat{\mathbb{U}}_h\widehat{V}_{h+1}(s,b,a)\\
            =&\phi(s,a)^\top w_h(b)-\phi(s,a)^\top \Lambda_h^{-1}\Big(\sum_{k=1}^K\phi(s_h^k, a_h^k)\widehat{V}_{h+1}(s_{h+1}^k, b-r_h^k)\Big)\\
            =&\phi(s,a)^\top w_h(b)
            -\phi(s,a)^\top \Lambda_h^{-1}\Big(\sum_{k=1}^K\phi(s_h^k, a_h^k)\mathbb{E}_{s',r}[\widehat{V}_{h+1}(s_{h+1}^k, b-r_h^k)]\Big)\\
            &-\phi(s,a)^\top \Lambda_h^{-1}\Big[\sum_{k=1}^K\phi(s_h^k, a_h^k)\Big(\widehat{V}_{h+1}(s_{h+1}^k, b-r_h^k)
            -\mathbb{E}_{s',r}\big[\widehat{V}_{h+1}(s_{h+1}^k, b-r_h^k)\big]\Big)\Big]
        .\end{split}
    \end{equation*}
    The triangle inequality gives
    \begin{equation}\label{eq: static-OCE aim}
        \begin{split}
            &\mathbb{U}_h\widehat{V}_{h+1}(s,b,a)-\widehat{\mathbb{U}}_h\widehat{V}_{h+1}(s,b,a)\\
            \leq& \Big|\phi(s,a)^\top w_h(b)
            -\phi(s,a)^\top \Lambda_h^{-1}\Big[\sum_{k=1}^K\phi(s_h^k, a_h^k)\mathbb{E}_{s',r}\big[\widehat{V}_{h+1}(s_{h+1}^k, b-r_h^k)\big]\Big]\Big|\\
            &+\Big|\phi(s,a)^\top \Lambda_h^{-1}\Big[\sum_{k=1}^K\phi(s_h^k, a_h^k)\Big(\widehat{V}_{h+1}(s_{h+1}^k, b-r_h^k)
            -\mathbb{E}_{s',r}\big[\widehat{V}_{h+1}(s_{h+1}^k, b-r_h^k)\big]\Big)\Big]\Big|
        .\end{split}
    \end{equation}   
    For the first term, $|\phi(s,a)^\top w_h(b)
            -\phi(s,a)^\top \Lambda_h^{-1}[\sum_{k=1}^K\phi(s_h^k, a_h^k)\mathbb{E}_{s',r}[\widehat{V}_{h+1}(s_{h+1}^k, b-r_h^k)]]|$, we have
    \begin{equation*}
        \begin{split}
            &\Big|\phi(s,a)^\top w_h(b)
            -\phi(s,a)^\top \Lambda_h^{-1}\Big[\sum_{k=1}^K\phi(s_h^k, a_h^k)\mathbb{E}_{s',r}\big[\widehat{V}_{h+1}(s_{h+1}^k, b-r_h^k)\big]\Big]\Big|\\
            =&\Big|\phi(s,a)^\top w_h(b)
            -\phi(s,a)^\top \Lambda_h^{-1}\Big[\sum_{k=1}^K\phi(s_h^k, a_h^k)\phi(s_h^k, a_h^k)^\top w_h(b)\Big]\Big|\\
            =&\Big|\phi(s,a)^\top w_h(b)
            -\phi(s,a)^\top \Lambda_h^{-1}\big(\Lambda_h-\lambda I\big)w_h(b)\Big|
            =\lambda\Big|\phi(s,a)^\top \Lambda_h^{-1}w_h(b)\Big|
        .\end{split}
    \end{equation*}
    By the Cauchy--Schwarz inequality,
    \begin{equation*}
        \begin{split}
            &\Big|\phi(s,a)^\top w_h(b)
            -\phi(s,a)^\top \Lambda_h^{-1}\Big[\sum_{k=1}^K\phi(s_h^k, a_h^k)\mathbb{E}_{s',r}\big[\widehat{V}_{h+1}(s_{h+1}^k, b-r_h^k)\big]\Big]\Big|\\
            \leq& \lambda\Big\|\phi(s,a)^\top \Big\|_{\Lambda_h^{-1}}\big\|w_h(b)\big\|_{\Lambda_h^{-1}}
            \leq\lambda\Big\|\Lambda_h^{-1}\Big\|_2^{\frac{1}{2}}\big\|w_h(b)\big\|\sqrt{\phi(s,a)^\top \Lambda_h^{-1}\phi(s,a)}\\
            \leq&\lambda\cdot\lambda^{-\frac{1}{2}}HL_u\sqrt{d}\sqrt{\phi(s,a)^\top \Lambda_h^{-1}\phi(s,a)}
            =HL_u\sqrt{d\lambda}\sqrt{\phi(s,a)^\top \Lambda_h^{-1}\phi(s,a)},
        \end{split}
    \end{equation*}
    where the last inequality is based on Lemma \ref{lemma: static-OCE bound w}. For any function $V: \mathcal{S}\times[0,H]\to[-HL_u,HL_u]$, we set
    \begin{equation*}
        \epsilon_h^k(V):=V(s_{h+1}^k,b-r_h^k)-\mathbb{E}_{s',r}\big[V(s',b-r)\mid s_h^k,a_h^k\big].
    \end{equation*}
    Thus, for the second term in \eqref{eq: static-OCE aim}, by the Cauchy-Schwarz inequality, we have
    \begin{align}
            &\Big|\phi(s,a)^\top \Lambda_h^{-1}\Big[\sum_{k=1}^K\phi(s_h^k, a_h^k)\Big(\widehat{V}_{h+1}(s_{h+1}^k, b-r_h^k)
            -\mathbb{E}_{s',r}\big[\widehat{V}_{h+1}(s',b-r)\mid s_h^k,a_h^k\big]\Big)\Big]\Big| \label{eq: static-OCE term 2} \\
            =&\Big|\phi(s,a)^\top \Lambda_h^{-1}\Big[\sum_{k=1}^K\phi(s_h^k, a_h^k)\epsilon_h^k(\widehat{V}_{h+1})\Big]\Big|
            \leq\Big\|\sum_{k=1}^K\phi(s_h^k, a_h^k)\epsilon_h^k(\widehat{V}_{h+1})\Big\|_{\Lambda_h^{-1}}\sqrt{\phi(s,a)^\top \Lambda_h^{-1}\phi(s,a)}.\nonumber
    \end{align}
    It remains to control $\big\|\sum_{k=1}^K\phi(s_h^k,a_h^k)\epsilon_h^k(\widehat{V}_{h+1})\big\|_{\Lambda_h^{-1}}$. By Definition~\ref{def: static-OCE function class}, $\widehat{V}_{h+1}\in\mathcal{V}$. Let $\mathcal{N}_{h+1}(\varepsilon)$ be an $\varepsilon$-cover of $\mathcal{V}$ with respect to the supremum norm over $\mathcal{S}\times[0,H]$. Then there exists $V'_{h+1}\in\mathcal{N}_{h+1}(\varepsilon)$ such that
    \begin{equation*}
        \sup_{(s,b)\in\mathcal{S}\times[0,H]}\big|\widehat{V}_{h+1}(s,b)-V'_{h+1}(s,b)\big|\leq\varepsilon.
    \end{equation*}
    Hence,
    \begin{equation*}
        \begin{split}
            &\Big|\mathbb{E}_{s',r}\big[\widehat{V}_{h+1}(s,b)\big|s_h,a_h\big]-\mathbb{E}_{s',r}\big[V'_{h+1}(s,b)\big|s_h,a_h\big]\Big|\\
            =&
            \Big|\mathbb{E}_{s',r}\big[\widehat{V}_{h+1}(s,b)-V'_{h+1}(s,b)\big|s_h,a_h\big]\Big|
            \leq\varepsilon
        .\end{split}
    \end{equation*}
    Then, by the triangle inequality, we have
    \begin{equation*}
        \Big|\Big(\widehat{V}_{h+1}(s',b)-\mathbb{E}_{s',r}\big[\widehat{V}_{h+1}(s',b)\big]\Big)
        -\Big(V'_{h+1}(s',b)-\mathbb{E}_{s',r}\big[V'_{h+1}(s',b)\big]\Big)\Big|\leq 2\varepsilon.
    \end{equation*}
    Thus, we get
    \begin{equation*}
        \Big|\epsilon_h^k\big(\widehat{V}\big)-\epsilon_h^k\big(V'\big)\Big|\leq 2\varepsilon.
    \end{equation*}
    Due to $\|a+b\|_\Lambda^2\leq2\|a\|_\Lambda^2+2\|b\|_\Lambda^2$, we have
    \begin{equation}\label{equation: static-OCE++}
        \begin{split}
            &\Big\|\sum_{k=1}^K\phi(s_h^k, a_h^k)\epsilon_h^k(\widehat{V}_{h+1})\Big\|_{\Lambda_h^{-1}}^2\\
            \leq& 2\Big\|\sum_{k=1}^K\phi(s_h^k, a_h^k)\epsilon_h^k\big(V'_{h+1}\big)\Big\|_{\Lambda_h^{-1}}^2
            +2\Big\|\sum_{k=1}^K\phi(s_h^k, a_h^k)\big[\epsilon_h^k\big(\widehat{V}_{h+1}\big)-\epsilon_h^k\big(V'_{h+1}\big)\big]\Big\|_{\Lambda_h^{-1}}^2\\
            \leq& 2\sup_{V\in\mathcal{N}_{h+1}(\varepsilon)}\Big\|\sum_{k=1}^K\phi(s_h^k, a_h^k)\epsilon_h^k(V)\Big\|_{\Lambda_h^{-1}}^2+\frac{8\varepsilon^2K^2}{\lambda}
        .\end{split}
    \end{equation}    
    Here we obtain an upper bound independent of the dataset $\mathcal{D}$. Applying Lemma \ref{lemma: static-OCE Jinying B2} and the union bound gives, with probability at least $1-\Delta\big|\mathcal{N}_{h+1}(\varepsilon)\big|$,
    \begin{equation*}
        \begin{split}
            \sup_{V\in\mathcal{N}_{h+1}(\varepsilon)}\Big\|\sum_{K=1}^K\phi(s_h^k,a_h^k)\epsilon_h^k(V)\Big\|_{\Lambda_h^{-1}}^2
            \leq 4HL_u^2\Big(2\log\frac{1}{\Delta}+d\log\big(1+\frac{K}{\lambda}\big)\Big)
        .\end{split}
    \end{equation*}
    Set $\Delta=\frac{\delta}{H|\mathcal{N}_{h+1}(\varepsilon)|}$, where $\delta\in(0,1)$. For any $h$, with probability $1-\frac{\delta}{H}$, there is
    \begin{align*}
        \sup_{V\in\mathcal{N}_{h+1}(\varepsilon)}\Big\|\sum_{K=1}^K\phi(s_h^k,a_h^k)\epsilon_h^k(V)\Big\|_{\Lambda_h^{-1}}^2
        \leq 4HL_u^2\Big(2\log\frac{H|\mathcal{N}_{h+1}(\varepsilon)|}{\delta}+d\log\big(1+\frac{K}{\lambda}\big)\Big).
    \end{align*}
    Then, with \eqref{equation: static-OCE++}, with probability at least $1-\delta$, for all $h\in [H]$, the following inequality holds 
    \begin{equation*}
        \begin{split}
           \Big\|\sum_{K=1}^K\phi(s_h^k,a_h^k)\epsilon_h^k(\widehat{V}_{h+1})\Big\|_{\Lambda_h^{-1}}^2
           \leq 8HL_u^2\Big(2\log\frac{H|\mathcal{N}_{h+1}(\varepsilon)|}{\delta}+d\log\big(1+\frac{K}{\lambda}\big)\Big)+\frac{8\varepsilon^2K^2}{\lambda}.
           \end{split}
    \end{equation*}
    Setting $\varepsilon=\frac{dH}{K}$ and $\lambda=1$, $L=HL_u\sqrt{\frac{dK}{\lambda}}$, by Lemma \ref{lemma: static-OCE e covering number}, we have
    \begin{equation*}
        \begin{split}
            \log\mathcal{N}_h(\varepsilon)
            \leq& d\log\Big(1+\frac{4L}{\varepsilon}\Big)+d^2\log\Big(1+\frac{8\sqrt{d}B^2}{\varepsilon^2}\Big)\\
            \leq& 
            d\log\Big(1+4HL_u d^{-\frac{1}{2}}K^\frac{3}{2}H^{-1}\Big)
            +d^2\log\Big(1+8B^2d^{-\frac{3}{2}}K^2H^{-2}\Big)
        .\end{split}
    \end{equation*}
    Then we set $B=2\beta$, $\beta=2cd\,HL_u\sqrt{\zeta}$, and $\zeta=\log(2dHK\delta^{-1})$, where $c>0$ is a constant. Therefore, we have
    \begin{equation}\label{eq: static-OCE logN}
        \begin{split}
            \log\mathcal{N}_h(\varepsilon)
            \leq 
            d\log\big(1+4d^{-\frac{1}{2}}K^\frac{3}{2}\big)
            +d^2\log\big(1+32c^2d^{\frac{1}{2}}K^2\zeta\big)
            \leq
            2d^2\log\big(64c^2d^{\frac{1}{2}}K^2\zeta\big)
        .\end{split}
    \end{equation}
    Then with the fact that $\log\zeta\leq\zeta$, $\log(1+K)\leq\log(2K)\leq\zeta$, and \eqref{eq: static-OCE logN}, we have
    \begin{equation*}
        \begin{split}
            &\Big\|\sum_{K=1}^K\phi(s_h^k,a_h^k)\epsilon_h^k(\widehat{V}_{h+1})\Big\|_{\Lambda_h^{-1}}^2\\
            \leq& 
            8HL_u^2\Big(
            2\log(H\delta^{-1})+4d^2\log(64c^2d^{\frac{1}{2}}K^2\zeta)+d\log(1+K)+4d^2
            \Big)\\
            \leq& 
            8HL_u^2\Big(
            2\log(H\delta^{-1})+4d^2\log(64c^2)+4d^2\zeta+2d^2\log(dK^4)+d\zeta+4d^2
            \Big)\\
            \leq& 
            8HL_u^2\Big(
            2d^2\log(dHK^4\delta^{-1})+4d^2\log(64c^2)+9d^2\zeta
            \Big)
            \leq
            4d^2HL_u^2\zeta\Big(
            8\log(64c^2)+34
            \Big)
        .\end{split}
    \end{equation*}
    By setting $8\log(64c^2)+34\leq\frac{c^2}{4}$, it holds that
    \begin{equation}\label{eq: static-OCE parameter of term 2}
        \begin{split}
            \Big\|\sum_{K=1}^K\phi(s_h^k,a_h^k)\epsilon_h^k\big(\widehat{V}_{h+1}\big)\Big\|_{\Lambda_h^{-1}}\leq cd\,HL_u\sqrt{\zeta}=\frac{\beta}{2}
        .\end{split}
    \end{equation}
    Therefore, based on \eqref{eq: static-OCE aim}, \eqref{eq: static-OCE term 2}, \eqref{eq: static-OCE parameter of term 2}, we have
    \begin{equation*}
        \begin{split}
            &\Big|\mathbb{U}_h\widehat{V}_{h+1}(s,b,a)-\widehat{\mathbb{U}}_h\widehat{V}_{h+1}(s,b,a)\Big|\\
            \leq& \Big(HL_u\sqrt{d}+cd\,HL_u\sqrt{\zeta}\Big)\sqrt{\phi(s,a)^\top \Lambda_h^{-1}\phi(s,a)}
            \leq
            \beta\sqrt{\phi(s,a)^\top \Lambda_h^{-1}\phi(s,a)}
        .\end{split}
    \end{equation*}
    By Lemma~\ref{lemma: static-OCE iota}, we have
    \begin{align*}
        \SubOptS(\widehat{\pi})
        \leq
        \sum_{h=1}^H\Big\{4cd\,HL_u\sqrt{\log \frac{2dHK}{\delta}}
        \cdot\mathbb{E}_{\pi^*}\Big[\sqrt{\phi(s_h,a_h)^\top \Lambda_h^{-1}\phi(s_h,a_h)}\Big|s_1,b_1^*\Big]\Big\}.
    \end{align*}
    Here the bound follows from Lemma~\ref{lemma: static-OCE iota}, the high-probability Bellman-error bound, and the choice $\beta=2cd\,HL_u\sqrt{\log(2dHK/\delta)}$.
This establishes the continuous-budget statistical term in Theorem~\ref{the: static-OCE oce result}. Applying Lemma~\ref{lemma_subopt_static-OCE_discrete} gives the discrete-budget statement with the additional $(1+L_u)\varsigma$ term.

\section{Other Supporting Lemmas}\label{sec: other important lemmas}
\begin{lemma} [\citet{Abbasiyadkori2011}]\label{lemma: Jinying B2}
Let $\{\mathcal{F}_t\}_{t=1}^\infty$ be a filtration and $\{\epsilon_t\}_{t=1}^\infty$ be an $\mathbb{R}$-valued stochastic process such that $\epsilon_t$ is $\mathcal{F}_t$-measurable for all $t \ge 1$. Moreover, suppose that conditioning on $\mathcal{F}_{t-1}$, $\epsilon_t$ is a zero-mean and $\sigma$-sub-Gaussian random variable for all $t \ge 1$, that is,
\(
\mathbb{E}[\epsilon_t | \mathcal{F}_{t-1}] = 0, 
\mathbb{E}[\exp(\lambda \epsilon_t) | \mathcal{F}_{t-1}] \le \exp(\lambda^2 \sigma^2 / 2), 
\forall \lambda \in \mathbb{R}.
\)
Meanwhile, let $\{\phi_t\}_{t=1}^\infty$ be an $\mathbb{R}^d$-valued stochastic process such that $\phi_t$ is $\mathcal{F}_{t-1}$-measurable for all $t \ge 1$. Also, let $M_0 \in \mathbb{R}^{d \times d}$ be a deterministic positive-definite matrix and
\(
M_t = M_0 + \sum_{s=1}^t \phi_s \phi_s^\top
\)
for all $t \ge 1$. For all $\Delta > 0$, it holds that
\[
\left\| \sum_{s=1}^t \phi_s \epsilon_s \right\|_{M_t^{-1}}^2 
\le 2 \sigma^2 \cdot \log \left( \frac{\det(M_t)^{1/2} \cdot \det(M_0)^{-1/2}}{\Delta} \right)
\]
for all $t \ge 1$ with probability at least $1 - \Delta$.
\end{lemma}

\begin{lemma} [Covering Number of Euclidean Ball \citep{Jin2019}] \label{lemma: Covering Number of Euclidean Ball}
    For any $\varepsilon\geq0$, the $\varepsilon-\text{covering}$ number of the Euclidean ball in $\mathbb{R^d}$ with radius $R\geq0$ can be upper bounded by $(1+\frac{2R}{\varepsilon})^d$.
\end{lemma}

\bibliography{rl_gsref}

\end{document}